\documentclass{article}

\usepackage[nonatbib, final]{neurips_2024}
\usepackage{watml}
\usepackage[utf8]{inputenc} 
\usepackage[T1]{fontenc}    
\usepackage{url}            
\usepackage{booktabs}       
\usepackage{amsfonts}       
\usepackage{nicefrac}       
\usepackage{microtype}      
\usepackage{xcolor}         
\usepackage[colorlinks=true,linkcolor=red,citecolor=red,urlcolor=red]{hyperref} 

\usepackage{mathtools}
\usepackage{xfrac}
\mathtoolsset{showonlyrefs}
\usepackage{dsfont}
\usepackage{bm}
\usepackage[linesnumbered,ruled,vlined]{algorithm2e}
\usepackage{amsthm} 
\usepackage{stmaryrd}
\DeclarePairedDelimiter{\idt}{\llbracket}{\rrbracket}
\usepackage{graphicx}
\graphicspath{ {figs/} }
\usepackage{stackengine}

\title{One Sample Fits All: Approximating All Probabilistic Values Simultaneously and Efficiently}

\author{%
	Weida Li\\
	School of Computing\\
	National University of Singapore\\
	\texttt{vidaslee@gmail.com} \\
	\And 
	Yaoliang Yu \\
	School of Computer Science\\
	University of Waterloo \\
	Vector Institute \\
	\texttt{yaoliang.yu@uwaterloo.ca} 
}

\begin{document}

	\maketitle

	\begin{abstract}

		The concept of probabilistic values, such as Beta Shapley values and weighted Banzhaf values, has gained recent attention in applications like feature attribution and data valuation. However, exact computation of these values is often exponentially expensive, necessitating approximation techniques. Prior research has shown that the choice of probabilistic values significantly impacts downstream performance, with no universally superior option. Consequently, one may have to approximate multiple candidates and select the best-performing one. Although there have been many efforts to develop efficient estimators, none  are intended to approximate all probabilistic values both simultaneously and efficiently. In this work, we embark on the first exploration of achieving this goal. Adhering to the principle of maximum sample reuse, we propose a one-sample-fits-all framework parameterized by a sampling vector to approximate  intermediate terms that can be converted to any probabilistic value without amplifying scalars. Leveraging the concept of $ (\epsilon, \delta) $-approximation, we theoretically identify a key formula that effectively determines the convergence rate of our framework. By optimizing the sampling vector using this formula, we obtain i) a one-for-all estimator that achieves the currently best time complexity for all probabilistic values on average, and ii) a faster generic estimator with the sampling vector optimally tuned for each probabilistic value. Particularly, our one-for-all estimator achieves the fastest convergence rate on Beta Shapley values, including the well-known Shapley value, both theoretically and empirically. Finally, we establish a connection between probabilistic values and the least square regression used in (regularized) datamodels, showing that our one-for-all estimator can solve a family of datamodels simultaneously. Our code is available at \url{https://github.com/watml/one-for-all}.
	\end{abstract}

	\section{Introduction}
	The problem of attribution is central in many aspects of machine learning \citep{rozemberczki2022shapley}. Examples include data valuation \citep{ghorbani2019data}, feature attribution \citep{lundberg2017unified}, multi-agent reinforcement learning \citep{wang2022shaq}, data attribution \citep{ilyas2022datamodels}, and the list goes on. One popular methodology is to leverage the concept of probabilistic values, which is uniquely characterized by the axioms of linearity, null, monotonicity and symmetry in cooperative game theory \citep{weber1988probabilistic}. Recent studies demonstrate that downstream performance employing this concept relies on the choice of probabilistic values, and the best one varies \citep{kwon2022weightedshap,li2023robust}. Therefore, practitioners may resort to approximating multiple candidates of probabilistic values and then select the best-performing one \citep{kwon2022weightedshap}.

	In general, probabilistic values can only be approximated as they require exponentially many utility evaluations to compute exactly. 
	Precisely, there has been a line of work devoted to developing efficient estimators for the Shapley value \citep[\eg,][]{jia2019towards,covert2021improving,zhang2023efficient,kolpaczki2024approximating}, while \citet{wang2023data,li2023robust} propose efficient estimators specific to weighted Banzhaf values. Although the research on generic estimators designed to approximate any probabilistic value has recently made progress \citep{lin2022measuring,li2024faster}, none of them can approximate all probabilistic values \emph{simultaneously and efficiently}. All in all, there is a strong demand for an efficient one-for-all estimator, the possibilities of which will be explored in this work.
	
	To sum up, we propose a \textbf{O}ne-sample-\textbf{F}its-\textbf{A}ll (OFA) framework parameterized by a sampling vector to approximate  intermediate terms that can be converted to any probabilistic value. Particularly, our framework i) adheres to the principle of maximum sample reuse and ii) does not include amplifying scalars in the conversion. 
	These two properties are considered indispensable as we observe that i) the empirical fastest estimators designed for the Shapley value or weighted Banzhaf values all follow the principle of maximum sample reuse and ii) amplifying scalars could deteriorate the convergence rates of estimators.
	Then, 
	using the concept of $ (\epsilon, \delta) $-approximation, \ie, $ P(\| \hat{\boldsymbol\phi} - \boldsymbol\phi \|_{2} \geq \epsilon) \leq \delta $ where $ \boldsymbol\phi $ refers to some probabilistic value and $ \hat{\boldsymbol\phi} $ is its estimate, we theoretically identify a formula from our framework that effectively determines the corresponding convergence rate, through which the sampling vector can be optimized. Specifically, we deduce i) an efficient one-for-all estimator (OFA-A) while optimizing the formula for all probabilistic values on \textbf{A}verage and ii) a faster generic estimator (OFA-S) while the optimization is done for each \textbf{S}pecific probabilistic value. The results of our convergence analysis are summarized as follows:
	\begin{itemize}[leftmargin=*]
		\item Our OFA-A achieves the convergence rate $ O(n\log n) $ for all probabilistic values on average. Notably, $ O(n\log n) $ is the currently-known best time complexity for \emph{some} probabilistic values. 

		\item For Beta Shapley values parameterized by $ \alpha, \beta \geq 1 $ \citep{kwon2022beta}, our OFA-A estimator requires $ O(n\log n) $ utility evaluations to achieve an $ (\epsilon, \delta) $-approximation \emph{simultaneously}. Note that $ \alpha=\beta=1 $ corresponds to the commonly-used Shapley value \citep{shapley1953value}.
		For the Shapley value, the previous best convergence rate is $ O(n(\log n)^{2}) $, achieved by the group testing estimator \citep[Theorem 6]{wang2023note}; however, we note that in our experiments the previous best-performing estimator is the complement estimator \citep{zhang2023efficient}, whose convergence rate is unknown.
		For Beta Shapley values with $ (\alpha = 1, \beta>1) $ or $ (\alpha> 1, \beta = 1) $, the previous best estimator requires $ O(n(\log n)^{3}) $ utility evaluations instead \citep[Proposition 4 and Remark 3]{li2024faster}.
		
		\item For weighted Banzhaf values parameterized by $ 0<w<1 $, the time complexity of our OFA-A is $ O(n^{\frac{3}{2}}\log n) $, not rivaling the previous best convergence rate $ O(n\log n) $ achieved by the estimator exclusive to weighted Banzhaf values \citep[Proposition 2]{li2023robust}.
		
		\item Nevertheless, our OFA-S achieves the convergence rate of $ O(n\log n) $ for both Beta Shapley values (with $ \alpha,\beta\geq 1 $) and weighted Banzhaf values.
	\end{itemize}
	In our experiments, the empirical convergence rates align well with the theoretical ones derived using the concept of $ (\epsilon, \delta) $-approximation. Additionally, we establish a connection between probabilistic values and the least square regressions employed in datamodels \citep{ilyas2022datamodels}, demonstrating that our OFA-A estimator can solve a family of datamodels simultaneously if it is the distances between feature coordinates that matter. This condition is met while using datamodels to detect similar training examples to a given target. Furthermore, we also identify a group of regularized datamodels that our OFA-A estimator can solve simultaneously without this condition.
	
	\section{Preliminaries}
	Let $ n $ be the number of players and $ [n] := \{ 1,2,\dots,n \} $ be the set of all players. In data valuation (feature attribution, respectively), $ n $ refers to the number of training data (features, respectively). For simplicity, we write $ S\backslash i $ and $ S\cup i $ instead of $ S\backslash \{ i \} $ and $ S \cup \{ i \} $, respectively. Meanwhile, (lowercase) $ s $ denotes the cardinality of the set (uppercase) $ S $.
	Then, each probabilistic value can be written as
	\begin{equation}
		\begin{gathered}
			\phi_{i} = \phi_{i}(U) = \sum_{S \subseteq [n]\backslash i} p_{s+1}[U(S\cup i) - U(S)]
		\end{gathered}
	\end{equation}
	where $ U:2^{[n]} \to \mathbb{R} $ is a utility function and $ \mathbf{p} \in \mathbb{R}^{n} $ is a non-negative vector such that $ \sum_{s=1}^{n}\binom{n-1}{s-1}p_{s} = 1 $. Take data valuation as an example, $ U(S) $ may measure the performance of models trained on $ S \subseteq [n] $, with which $ \phi_{i}(U) $ can be interpreted as the contribution of the $ i $-th data point to the performance of models trained on $ [n] $.
	
	If there exists a (Borel) probability measure $ \mu $ on the closed interval $[0,1]$ such that $ p_{s} = \int_{0}^1 w^{s-1}(1-w)^{n-s}\mathrm{d}\mu(w) $, then the resulting probabilistic value is referred to as a semi-value \parencite{dubey1981value}. 
	If $ \mu $ represents a Dirac delta distribution $ \delta_{a} $, the corresponding probabilistic value is referred to as the weighted Banzhaf value parameterized by $ a $, or WB-$ a $.
	For Beta Shapley values, denoted by Beta$ (\alpha, \beta) $, $ \mu(A) = \int_{A} w^{\beta-1}(1-w)^{\alpha-1} \mathrm{d}w $. In practice, the considered range of $ \alpha $ or $ \beta $ is $ [1,\infty) $ \citep{kwon2022beta,kwon2022weightedshap}. Particularly, Beta$ (1,1) $, whose $ \mu $ is the uniform distribution (over $[0,1]$), corresponds to the Shapley value.
	
	We will use the standard notion of $(\epsilon,\delta)$-approximation to analyze a (randomized) estimate $ \hat{\boldsymbol\phi} $ of some probabilistic value $ \boldsymbol\phi $.
	
	\begin{definition}
		We say a (randomized) estimate $ \hat{\boldsymbol\phi} $ achieves an $ (\epsilon, \delta) $-approximation of $\phiv$ if $ P(\| \hat{\boldsymbol\phi} - \boldsymbol\phi \|_{2} \geq \epsilon) \leq \delta $.
	\end{definition}
	
	For instance, \citet[Theorem 4.9]{wang2023data} proved that their proposed estimator requires $ O(\frac{n}{\epsilon^{2}}\log\frac{n}{\delta}) $ utility evaluations to achieve an $ (\epsilon, \delta) $-approximation for WB-$ 0.5 $, provided that $ \| U \|_{\infty} \leq 1 $. When $\epsilon$ and $\delta$ are considered fixed constants, we  then simply say the estimator converges at $ O(n\log n) $ rate. 

	\section{Motivations} 
	\paragraph{One-For-All Estimators} In this paper, an estimator is referred to as one-for-all if it is capable of sampling subsets \textbf{O}nce to approximate \textbf{A}ll probabilistic values.  
	
	Though existing estimators are not designed to approximate all probabilistic values simultaneously, some of them can be easily modified for this end by using the weighted sampling technique. Take the sampling lift (SL) estimator \citep{moehle2021portfolio} as an example, its approximation is based on
	\begin{equation} 
		\begin{gathered}
			\phi_{i} = \mathbb{E}_{S\subseteq [n]\backslash i}[U(S\cup i) - U(S)] \text{ where } P(S) = p_{s+1} .
		\end{gathered}
	\end{equation}
	If we fix the probability of sampling $ S $ to be the one, denoted by $ \mathbf{q} \in \mathbb{R}^{n} $, for the Shapley value, there is
	\begin{equation} \label{eq:sl-shapley}
		\begin{gathered}
			\phi_{i} = \mathbb{E}_{S\subseteq [n]\backslash i}^{\text{Shap}}\left[\frac{p_{s+1}}{q_{s+1}}(U(S\cup i) - U(S))\right] ,
		\end{gathered}
	\end{equation}
	which is the weighted sampling lift (WSL) estimator employed by \citet{kwon2022beta}.
	Therefore, we can store the accumulated results $ \{ U(S\cup i) - U(S) \} $ separately for each subset size of $ S $ so that they can be reweighted to be any probabilistic value.

	\paragraph{The Effect of Amplifying Factors}
	However, the scalar $ \frac{p_{s+1}}{q_{s+1}} $ potentially amplifies the theoretical convergence rate. To demonstrate, we take the WSL estimator as an example. In this case, $ \hat{\phi}_{i} = \frac{1}{T}\sum_{t=1}^{T}X_{t} $ where $ \{ X_{t} \}_{t=1}^{T} $ are i.i.d. random variable such that $ P(X_{t} = \frac{p_{s+1}}{q_{s+1}}(U(S\cup i) - U(S))) = q_{s+1} $ and thus $ \mathbb{E}[X_{t}] = \phi_{i} $. 
	Assume that $ \| U \|_{\infty}\leq 1 $, by the Heoffding's inequality, $ P(|\hat{\phi}_{i} - \phi_{i}|\geq \epsilon) \leq 2\exp\left( -\frac{T\epsilon^{2}}{8C^{2}} \right) $ where $ C = \max_{1\leq k\leq n} \frac{p_{k}}{q_{k}} $. By solving $ 2\exp\left( -\frac{T\epsilon^{2}}{8C^{2}} \right) \leq \delta $, we eventually obtain $ T\geq \frac{8C^{2}}{\epsilon^{2}}\log\frac{2}{\delta} $ and therefore the convergence rate of $ \hat{\phi}_{i} $ is $ O(\frac{C^{2}}{\epsilon^{2}}\log\frac{2}{\delta}) $. Consequently, if $ C \rightarrow \infty $ as $ n \rightarrow \infty $, this theoretical convergence rate deteriorates asymptotically. 
	For the Banzhaf value, $ p_{k} = \frac{1}{2^{n-1}} $; since $ q_{k} = \frac{(k-1)!(n-k)!}{n!} $, if $ k = \frac{n+1}{2} $, there is $ \frac{p_{k}}{q_{k}} \in \Theta(n^{\frac{1}{2}}) $ by the Stirling's approximation $ m! \simeq \sqrt{m}\left( \frac{m}{e} \right)^{m} $. Therefore, $ C^{2} $ introduces a factor of $ n $ into the theoretical convergence rate, though the derived formula may not be tight. If we switch the roles of $ \mathbf{p} $ and $ \mathbf{q} $, the amplifying scalar could be as worst as $ \Theta(\frac{2^{n}}{n^{k}}) $ for small $ k $. 
	


	Meanwhile, we also notice that \citet{kwon2022weightedshap} resort to a one-for-all estimator based on
	\begin{equation} \label{eq:weightedSHAP}
		\begin{gathered}
			\phi_{i} = \sum_{s=1}^{n}m_{s}\cdot\mathbb{E}_{\substack{R\subseteq[n]\backslash i\\ r=s-1}}[U(R\cup i)-U(R)] 
		\end{gathered}
	\end{equation}
	where $ m_{s} = \binom{n-1}{s-1}p_{s} $ and each expectation is taken over the corresponding uniform distribution. We refer to this estimator as weightedSHAP in this work. As can be verified, Eq.~\eqref{eq:weightedSHAP} does not contain any amplifying scalars, i.e., $ m_{s} \leq 1 $. 
	
	\paragraph{The Principle of Maximum Sample Reuse}
	However, estimators designed according to Eqs.~\eqref{eq:sl-shapley} and~\eqref{eq:weightedSHAP} are not expected to be efficient as it does not obey the principle of maximum sample reuse. Precisely, an estimator adheres to the principle of maximum sample reuse if each sampled subset is used to update all estimates $ \{ \hat{\phi}_{i} \}_{i\in [n]} $.
	As analyzed by \citet[Section 4.2]{zhang2023efficient}, estimators based on sampled marginal contributions $ \{ U(S\cup i) - U(S) \} $ are impossible to meet the principle of maximum sample reuse. 
	By contrast, we observe that the SHAP-IQ estimator proposed by \citet{fumagalli2024shap} can also be adopted for this end, which employs the formula
	\begin{equation} \label{eq:shapiq}
		\begin{gathered}
			\phi_{i} = p_{n} (U([n])-U(\emptyset)) + 2H\mathbb{E}_{\emptyset\subsetneq S \subsetneq [n]}[((n-s)m_{s}\mathds{1}_{i\in S} - sm_{s+1}\mathds{1}_{i\not\in S})(U(S)-U(\emptyset))]
		\end{gathered}
	\end{equation}
	where $ m_{s} = \binom{n-1}{s-1}p_{s} $, $ H = \sum_{j=1}^{n-1}\frac{1}{j} $, and $ P(S) \propto \binom{n-2}{s-1}^{-1} $. In particular, SHAP-IQ is equal to the unbiased KernelSHAP estimator \citep{covert2021improving} for the Shapley value; see \citep[Theorem 4.5]{fumagalli2024shap}. Although SHAP-IQ follows the principle of maximum sample reuse, it is apparent that Eq.~\eqref{eq:shapiq} contains amplifying scalars even for the Shapley value. Meanwhile,	
	there is another line of research in quest of efficient estimators for the Shapley value by reducing the variance via the stratified sampling technique \citep{maleki2013bounding,burgess2021approximating,castro2017improving,wu2023variance}. However, such a technique also does not verify the principle of maximum sample reuse.
	
	\begin{figure}[t]
		\centering
		\begin{tabular}{ccc}
			\multicolumn{3}{c}{\includegraphics[width=0.85\linewidth]{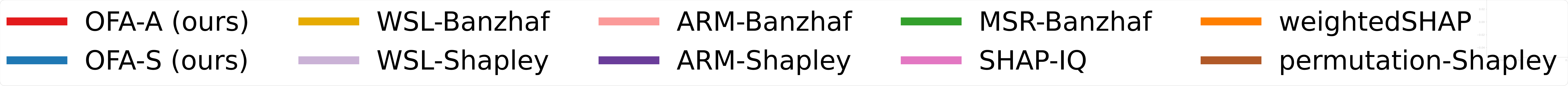}} \\
			\includegraphics[width=0.3\linewidth]{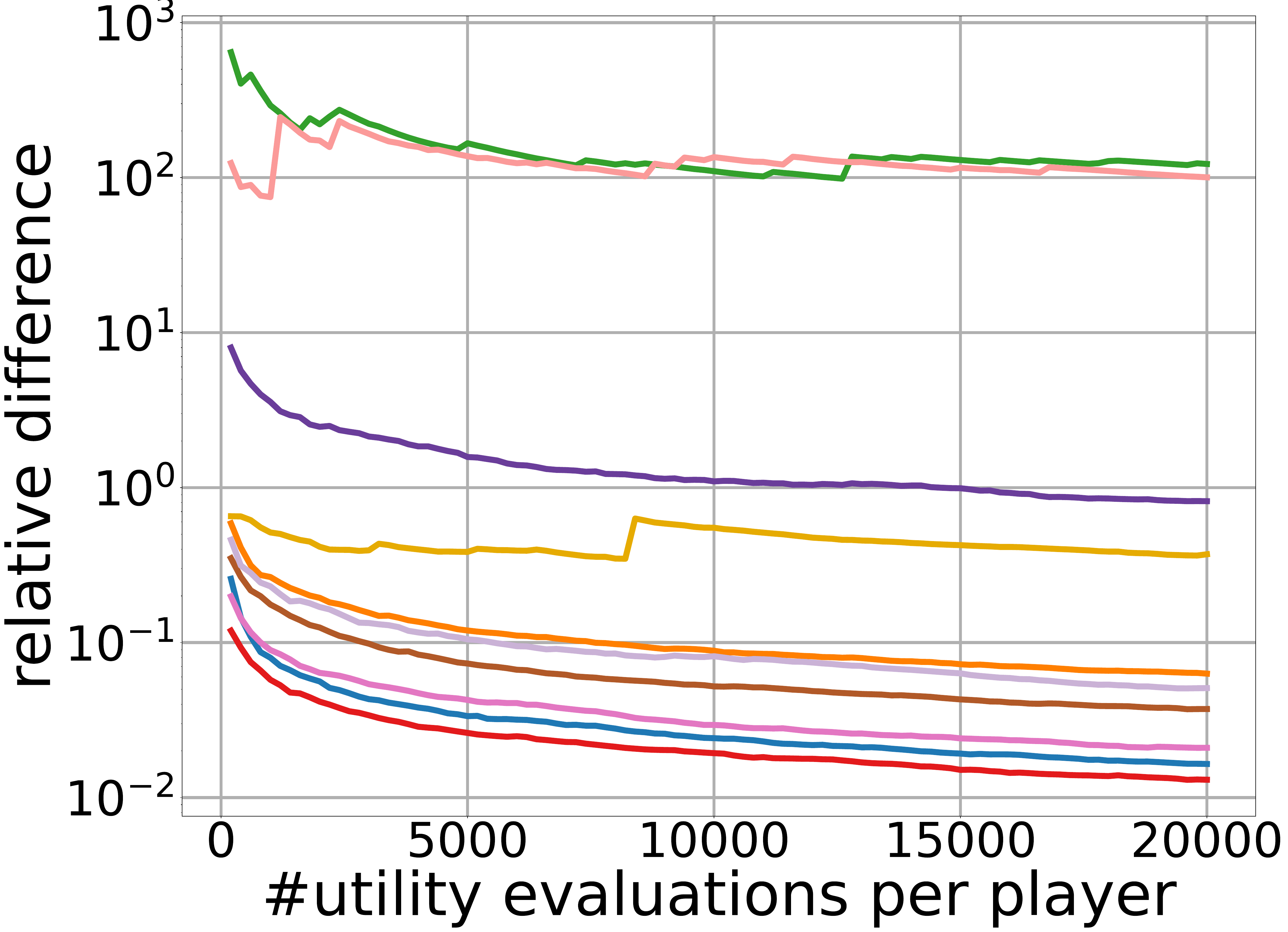} & \includegraphics[width=0.3\linewidth]{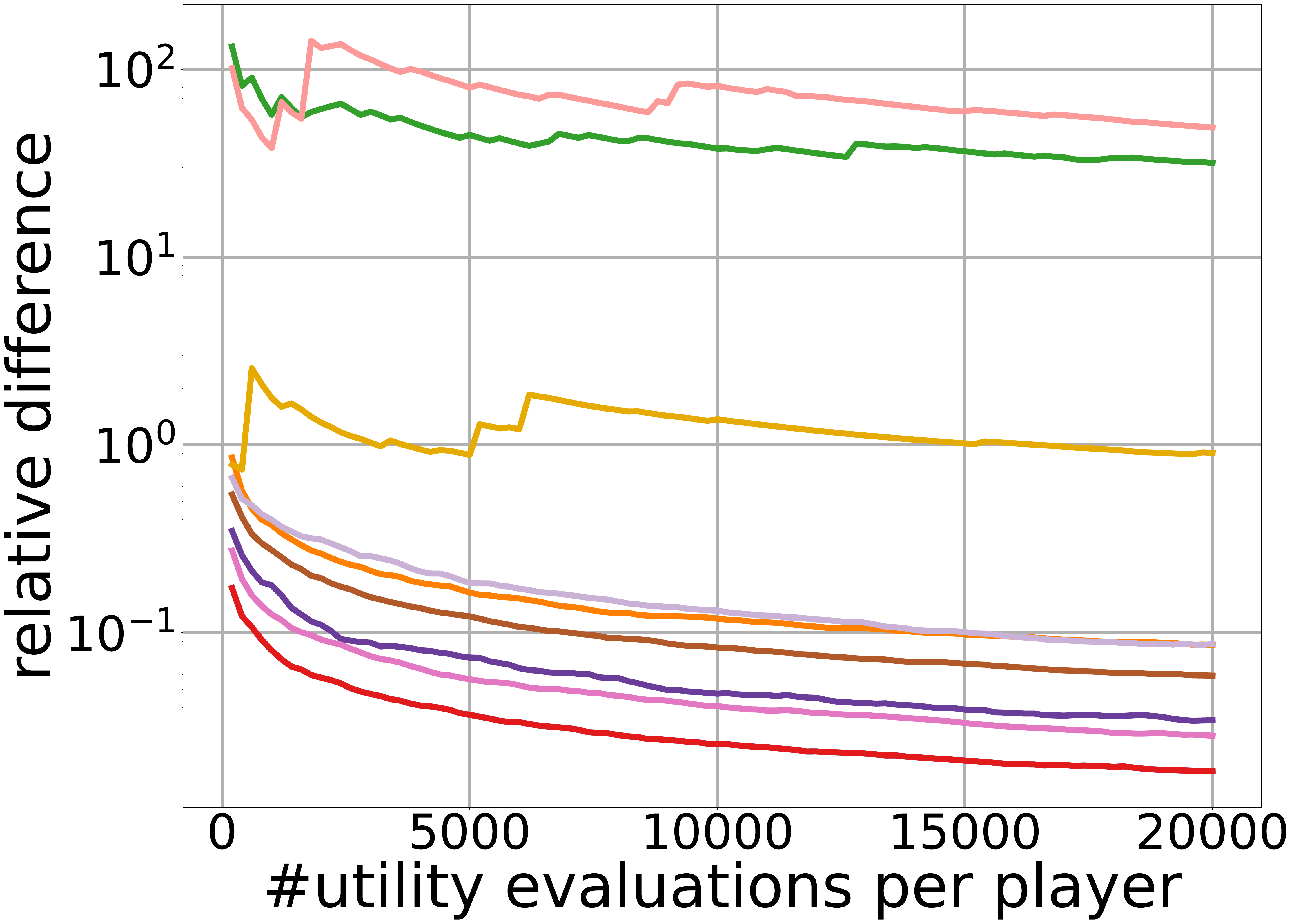} &
			\includegraphics[width=0.3\linewidth]{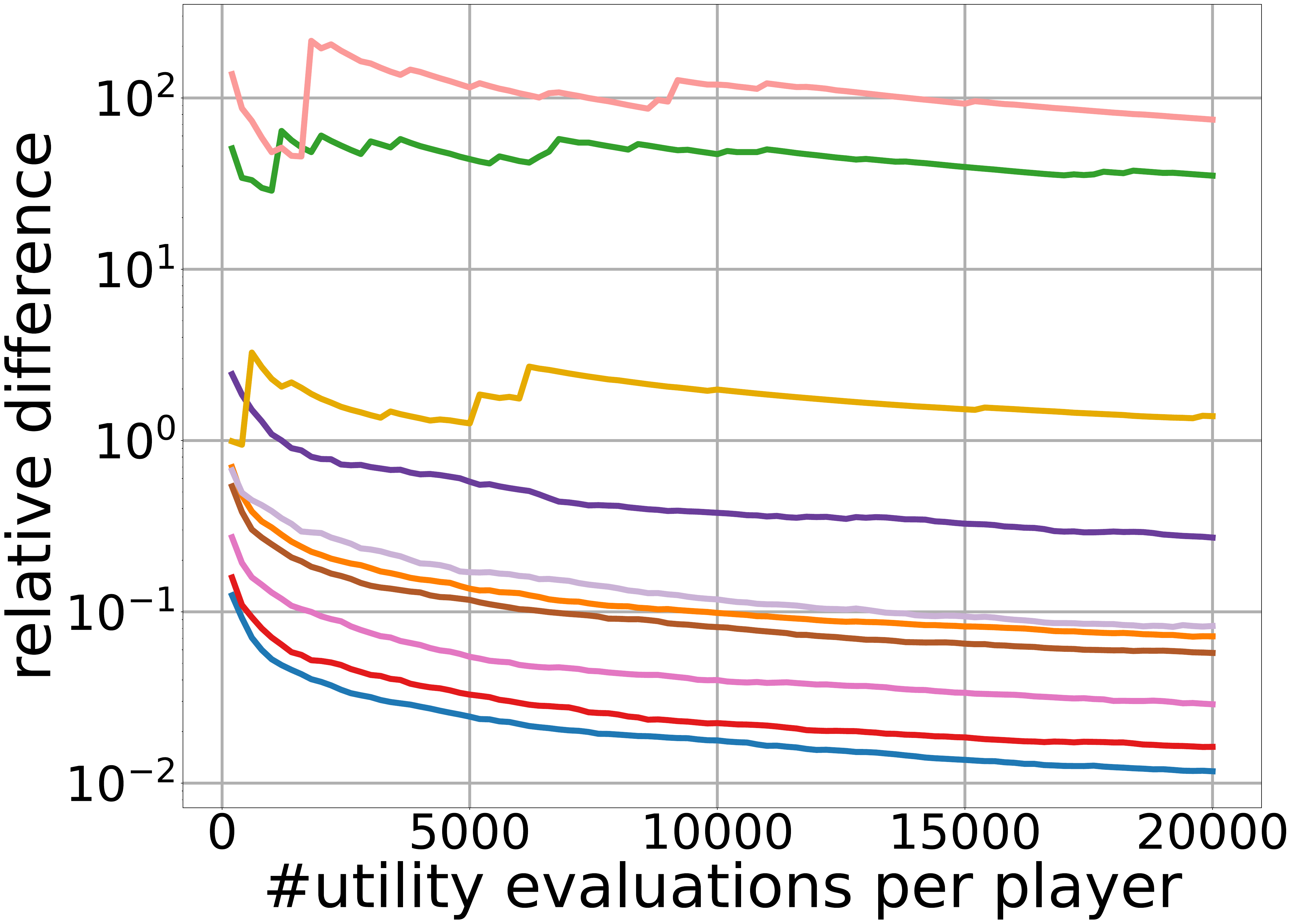} \\
			Beta$ (4,1) $ & Beta$ (1,1) $ & Beta$ (1, 4) $\\
			\includegraphics[width=0.3\linewidth]{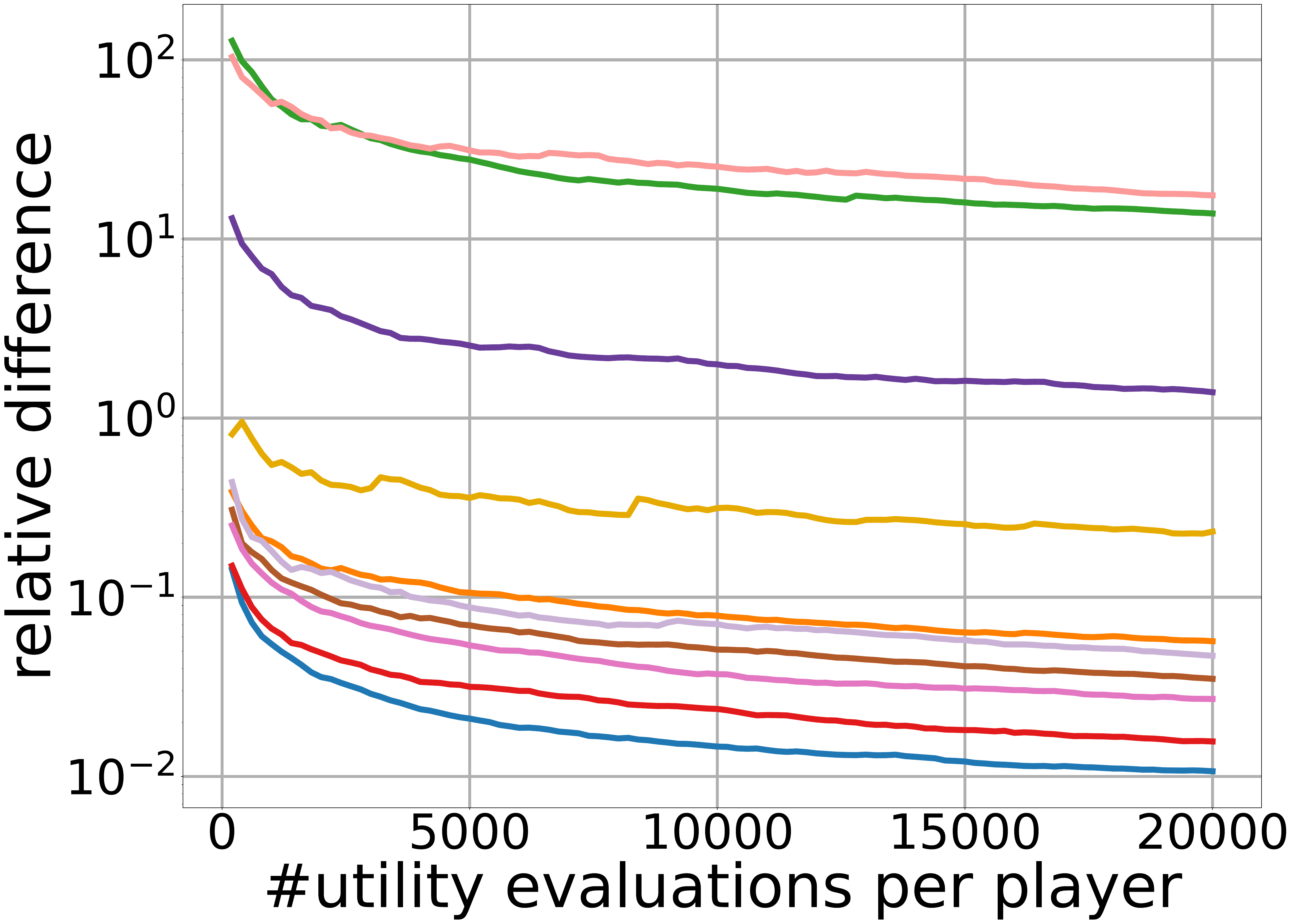} & \includegraphics[width=0.3\linewidth]{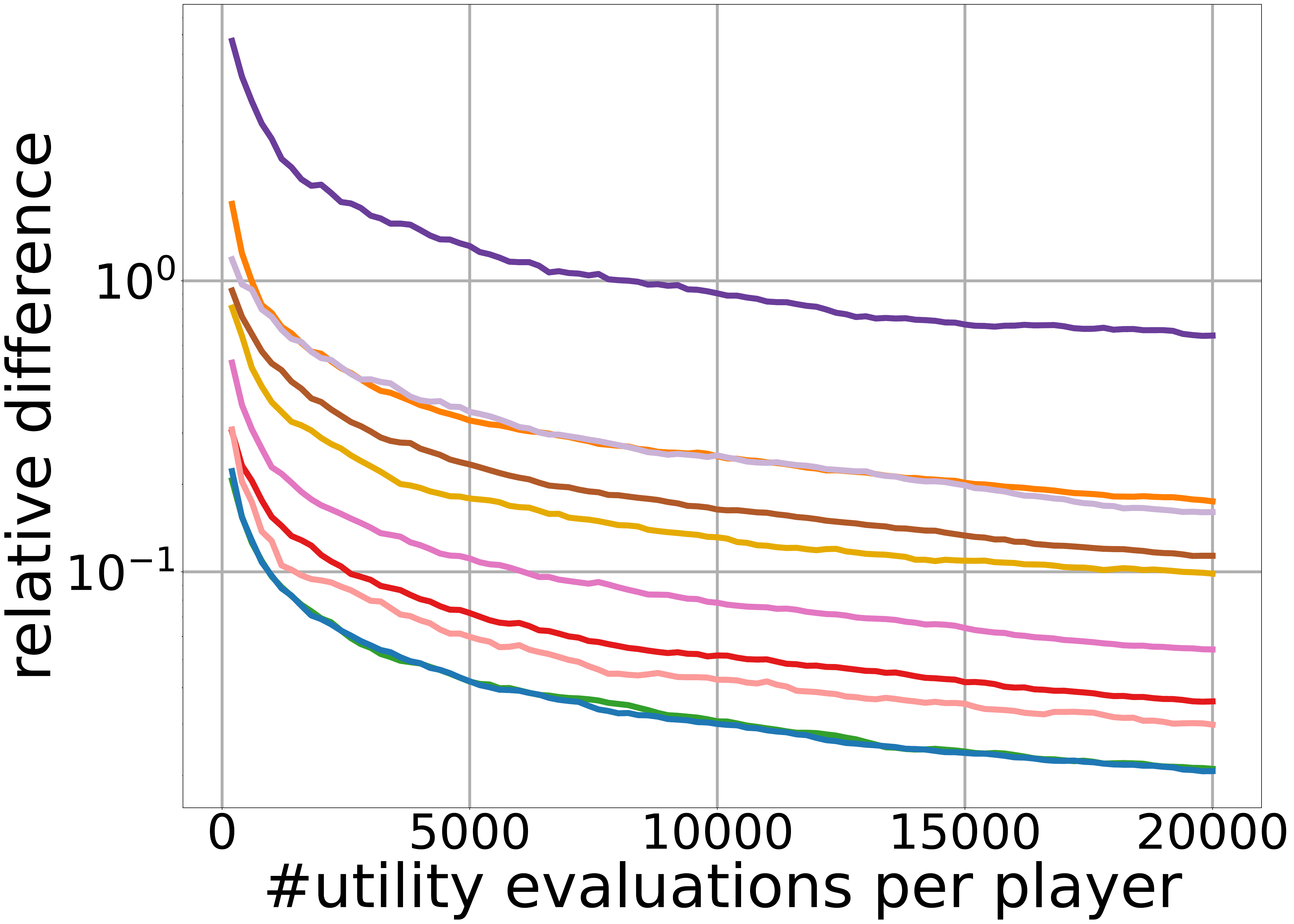} &
			\includegraphics[width=0.3\linewidth]{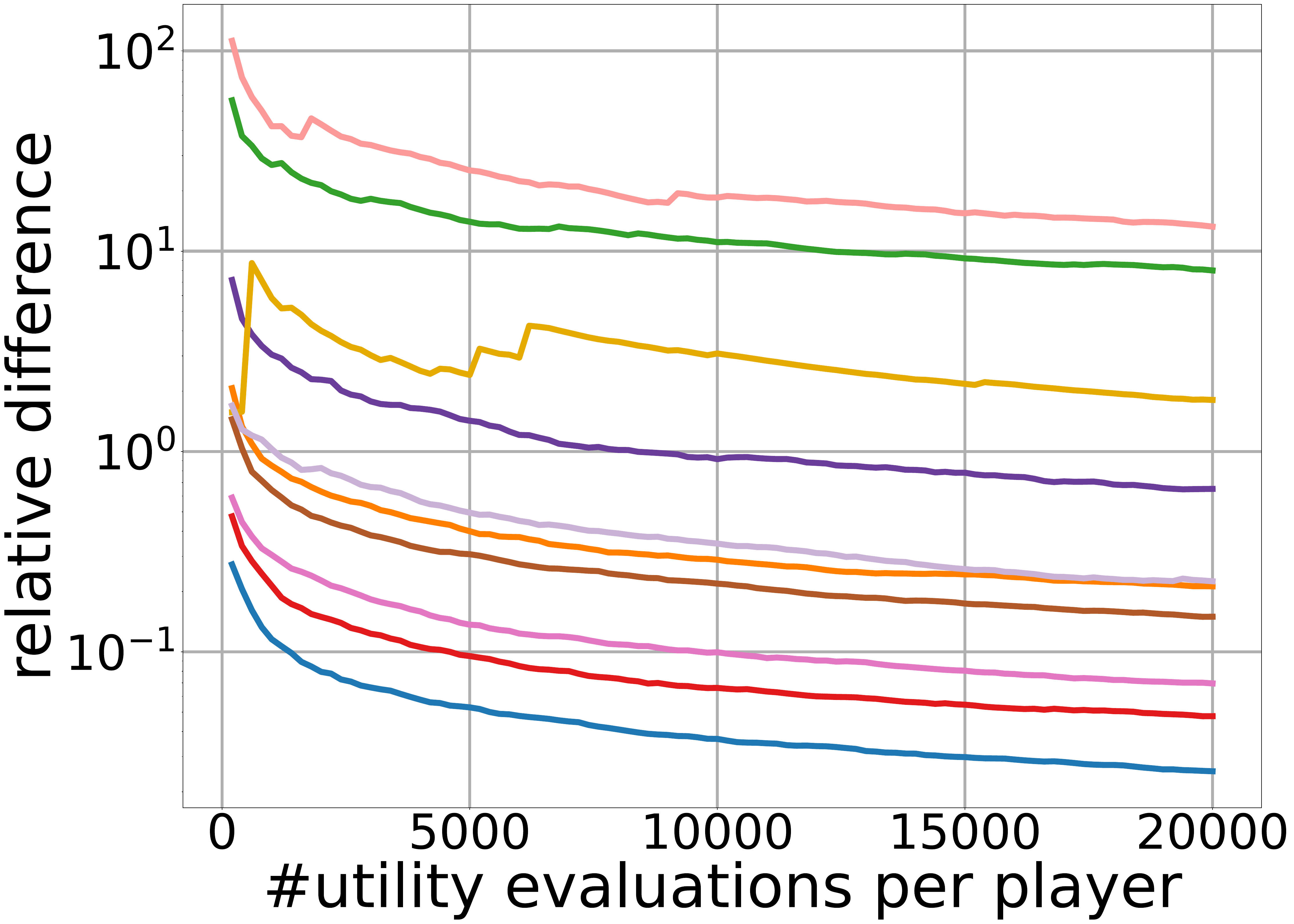} \\
			WB-$ 0.2 $ & WB-$ 0.5 $ & WB-$ 0.8 $
		\end{tabular}
		\vspace{-.5em}
		\caption{Comparison of ten one-for-all estimators. Beta$ (\alpha, \beta) $ denotes Beta Shapley values, whereas WB-$ a $ refers to weighted Banzhaf values. Our OFA-S estimator is equal to the OFA-A estimator for the Shapley value. The suffix ``Shapley'' indicates that there is no reweighting for the Shapley value, while ``Banzhaf'' stands for the Banzhaf value. The permutation estimator is originally proposed for the Shapley value. The utility function $ U $ is the cross-entropy loss of LeNet trained on $ 24 $ data from FMNIST. All the results are averaged using $ 30 $ random seeds.}
		\vspace{-1em}
		\label{fig:motivation}
	\end{figure}
	\paragraph{Empirical Evidence}
	For convenience, we formally define the two aforementioned desirable properties for estimators to possess as \textbf{P1:} The underlying formula contains no amplifying scalars and \textbf{P2:} Each sampled subset is used to update all the estimates $ \{ \hat{\phi}_{i} \}_{i=1}^{n} $.
	In Figure~\ref{fig:motivation}, we provide some experiment results while setting $ n=24 $ to support our aforementioned informal analysis. Precisely, we implement six one-for-all estimators by combining the weighted sampling technique and the previous estimators. 		Some of our observations are:
	\begin{itemize}[leftmargin=*]
		
		
		\item On WB-0.5, weightedSHAP, which satisfies \textbf{P1} but not \textbf{P2}, is empirically not comparable to MSR-Banzhaf that possesses both \textbf{P1} and \textbf{P2}. This observation supports the necessity of \textbf{P2}.

		\item On WB-0.5, SHAP-IQ sticks to \textbf{P2} but not \textbf{P1}. It is clear that SHAP-IQ also performs significantly worse than MSR-Banzhaf, which highlights the role of \textbf{P1}. 
		
		
		
		\item The sudden rises of relative differences stem from the existence of significantly large amplifying scalars. For WSL-Banzhaf on Beta$ (1,1) $, the amplifying scalar for $ U(i) - U(\emptyset) $ is as large as $ \frac{2^{24}}{24} $!
	\end{itemize}	
	In Table~\ref{tab:sum}, we summarize the previous estimators in terms of \textbf{P1} and \textbf{P2}, and defer the technical details to Appendix~\ref{app:overview}. Notably, the complement estimator is empirically the best for the Shapley value, while it is MSR for weighted Banzhaf values; both of them follow \textbf{P1} and \textbf{P2}.
	
	\begin{table}[t]
		\centering
		\caption{A scope of ``all'' indicates that the estimator is able to approximate any probabilistic value, whereas ``weighted Banzhaf'' suggests that the estimator can only approximate weighted Banzhaf values.
		\textbf{P1} refers to the property that the underlying formula does not contain any amplifying scalars \textit{for all probabilistic values in its scope}, while \textbf{P2} means whether each sampled subset is used to update all the estimates $ \{ \hat{\phi}_{i} \}_{i=1}^{n} $.} 
		\label{tab:sum}
		\resizebox{\columnwidth}{!}{
			\begin{tabular}{cccccccc}
				\toprule
				& \stackanchor{WSL}{\citep{kwon2022beta}} & \stackanchor{SL}{\citep{moehle2021portfolio}} & \stackanchor{GELS}{\citep{li2024faster}} & \stackanchor{ARM}{\citep{kolpaczki2024approximating}} & \stackanchor{MSR}{\citep{wang2023data}} & \stackanchor{SHAP-IQ}{\citep{fumagalli2024shap}} & \stackanchor{weightedSHAP}{\citep{kwon2022weightedshap}} \\ \midrule
				scope & all & Shapley & all & all & weighted Banzhaf & all & all\\
				\textbf{P1} & \xmark & \cmark & \xmark & \cmark & \cmark & \xmark & \cmark \\
				\textbf{P2} & \xmark & \xmark & \xmark & \xmark & \cmark & \cmark & \xmark \\\midrule		
				& \stackanchor{permutation}{\citep{castro2009polynomial}} & \stackanchor{kernelSHAP}{\citep{lundberg2017unified}} & \stackanchor{unbiased kernelSHAP}{\citep{covert2021improving}} & \stackanchor{group testing}{\citep{wang2023note}} & \stackanchor{complement}{\citep{zhang2023efficient}} & \stackanchor{AME}{\citep{lin2022measuring}} & OFA (ours) \\
				scope & Shapley & Shapley & Shapley & Shapley & Shapley & partial & all\\
				\textbf{P1} & \cmark & \xmark  & \xmark & \xmark & \cmark & \xmark & \cmark \\
				\textbf{P2} & \xmark & \cmark & \cmark & \xmark & \cmark & \cmark & \cmark \\\bottomrule
			\end{tabular}
		}
	\end{table}
	
	\section{Main Results} \label{sec:main}
	The framework we propose is built upon
	\begin{equation}\label{eq:ofa}
		\begin{gathered}
			\phi_{i} = \sum_{s=1}^{n} m_{s}\cdot\left( \underset{\substack{i \in R\\r=s}}{\mathbb{E}}[U(R)] - \underset{\substack{i\not\in R\\r=s-1}}{\mathbb{E}}[U(R)] \right) 
		\end{gathered}
	\end{equation}
	where $ m_{s} = \binom{n-1}{s-1} p_s$ and each expectation is taken over the corresponding uniform distribution.
	For simplicity, we write $ \phi_{i,s}^{+} = \mathbb{E}_{i\in R, r=s}[U(R)] $ and $ \phi_{i,s-1}^{-} = \mathbb{E}_{i\not\in R, r=s-1}[U(R)] $.
	Clearly, there is no amplifying scalars in Eq.~\eqref{eq:ofa}. Meanwhile, the structure of Eq.~\eqref{eq:ofa} suits the principle of maximum sample reuse. 
	Since $ \{ \phi_{i, k}^{+} \}_{k=1,n-1,n} $ and $ \{ \phi_{i,k}^{-} \}_{k=0,1,n-1} $ can be calculated exactly using $ 2n+2 $ utility evaluations of $ U $, our focus is to efficiently approximate $ \{ \phi_{i,s}^{+}, \phi_{i,s}^{-} \}_{2\leq s\leq n-2} $. 
	The resulting framework is demonstrated in Algorithm~\ref{alg:ofa}; $ q_{j} $ refers to the probability of drawing a subset of $ [n] $ with size $ j+1 $.
	
	\begin{figure}[t]
		\begin{algorithm}[H]
			\DontPrintSemicolon
			\caption{The \textbf{O}ne-Sample-\textbf{F}its-\textbf{A}ll (OFA) Framework}
			\label{alg:ofa}
			\KwIn{A utility function $ U: 2^{[n]} \to \mathbb{R} $, a positive probability vector $ \mathbf{q} \in \mathbb{R}^{n-3} $, and a total number $ T $ of samples}
			
			\KwOut{Estimates to $ \phi^{+}_{i,k^{+}} $ and $ \phi^{-}_{i,k^{-}} $ with $ i, k^{+} \in [n] $ and $ 0\leq k^{-} \leq n-1 $} 
			
			$ \hat{\phi}_{i,k^{+}}^{+} \gets \phi_{i,k^{+}}^{+} $ and $ \hat{\phi}_{i,k^{-}}^{-} \gets \phi_{i,k^{-}}^{-} $ for $ i \in [n] $, $ k^{+} \in \{ 1,n-1,n \} $ and $ k^{-} \in \{ 0,1,n-1 \} $
			
			$ \hat{\phi}_{i,k}^{+} \gets 0 $, $ T^{+}_{i,k} \gets 0 $, $ \hat{\phi}_{i,k}^{-} \gets 0 $ and $ T^{-}_{i,k} \gets 0 $ with $ i\in [n] $ and $ 2\leq k \leq n-2 $

			\For{$ t = 1,2,\dots, T $}{
				Sample $ s_{t} $ from $ \{ 2,3,\dots,n-2 \} $ according to $ \mathbf{q} $
				
				Sample $ S_{t} $ uniformly from $ \{ R \subseteq [n] \mid r=s_{t} \} $
				
				$ v \gets U(S_{t}) $
				
				\For{$ i=1,2,\dots,n $}{
					\uIf{$ i \in S_{t} $}{
						$ \hat{\phi}_{i,s_{t}}^{+} \gets \frac{T_{i,s_{t}}^{+}}{T_{i,s_{t}}^{+} + 1}\hat{\phi}_{i,s_{t}}^{+} + \frac{1}{T_{i,s_{t}}^{+} + 1} v $ and $ T_{i,s_{t}}^{+} \gets T_{i,s_{t}}^{+} + 1 $	
					}
					\uElse{
						$ \hat{\phi}_{i,s_{t}}^{-} \gets \frac{T_{i,s_{t}}^{-}}{T_{i,s_{t}}^{-} + 1}\hat{\phi}_{i,s_{t}}^{-} + \frac{1}{T_{i,s_{t}}^{-} + 1} v $ and $ T_{i,s_{t}}^{-} \gets T_{i,s_{t}}^{-} + 1 $	
					}
				}
			}
			
			\textbf{Aggregation Phase:} $ \hat{\phi}_{i} = \sum_{s=1}^{n} m_{s}(\hat{\phi}_{i,s}^{+} - \hat{\phi}_{i,s-1}^{-}) $
		\end{algorithm}
		\vspace{-.5em}
	\end{figure}
	
	To facilitate the choice of the sampling vector $ \mathbf{q} \in \mathbb{R}^{n-3} $ appearing in Algorithm~\ref{alg:ofa}, our first step is to theoretically ascertain a key formula that effectively determines the convergence rate of Algorithm~\ref{alg:ofa}.
	
	\begin{restatable}{theorem}{convergenceFormula} \label{thm:convergence}
		Assume i) $ \| U \|_{\infty} \leq u $ and ii) $ 0<\epsilon\leq \sqrt{2D(\mathbf{m}, \mathbf{q})\gamma(\mathbf{q})^{2}u^{2}} $. For $ \hat{\boldsymbol\phi} $ in Algorithm~\ref{alg:ofa}, it requires $ \frac{4nu^{2}D(\mathbf{m}, \mathbf{q})}{\epsilon^{2}}\log\frac{8n^{2}}{\delta} $ evaluations of $ U $ to achieve $ P(\| \hat{\boldsymbol\phi} - \boldsymbol\phi \|_{2} \geq \epsilon) \leq \delta $ where
		\begin{equation}
			\begin{gathered}
				D(\mathbf{m}, \mathbf{q}) = \sum_{s=2}^{n-2}\frac{n}{q_{s-1}}\left( \frac{m_{s}^{2}}{s} + \frac{m_{s+1}^{2}}{n-s} \right) \text{ and } \gamma(\mathbf{q}) = \min_{2\leq s\leq n-2} \min\left( \frac{q_{s-1}\cdot s}{n},\ \frac{q_{s-1}\cdot (n-s)}{n} \right).
			\end{gathered}
		\end{equation}
	\end{restatable}
	We remark that $D(\mv, \qv)$ is jointly convex in $\mv$ and $\qv$.	
	The second assumption in Theorem~\ref{thm:convergence} can be removed if we pre-allocate the number of sampled subsets for each $ \phi_{i,s}^{+} $ or $ \phi_{i,s}^{-} $ and draw subsets in a predetermined order; see the proof in Appendix~\ref{app:theorem 1} for details. Precisely, let $ T_{i,s}^{+} $ be the number of subsets for estimating $ \phi_{i,s}^{+} $, and define $ T_{i,s}^{-} $ similarly; then the pre-allocated numbers are $ T_{i,s}^{+} \approx \frac{s\cdot q_{s-1}}{n} T $ and $ T_{i,s}^{-} \approx \frac{(n-s)q_{s-1}}{n} T $, which are the expected values of $ T_{i,s}^{+} $ and $ T_{i,s}^{+} $ while using Algorithm~\ref{alg:ofa}; $ T $ refers to the total number of sampled subsets. By Theorem~\ref{thm:convergence}, the convergence rate of Algorithm~\ref{alg:ofa} is $ O(D(\mathbf{m}, \mathbf{q})\cdot n\log n) $, and thus achieving the currently best convergence rate $ O(n\log n) $ requires $ D(\mathbf{m}, \mathbf{q}) \in O(1) $.
	
	\subsection{A One-For-All Estimator}
	To obtain our one-for-all estimator, our goal is to find a $ \mathbf{q}^{\text{OFA-A}} $ such that $ D(\mathbf{m}, \mathbf{q}^{\text{OFA-A}}) \in O(1) $ for as many $ \mathbf{m} $ as possible.  To this end, we define $ \mathbf{q}^{\text{OFA-A}} $ to be the uniquely optimal solution to
	\begin{equation}
		\begin{gathered}
			\argmin_{\mathbf{q} \in \mathbb{R}^{n-3}} \overline{D}(\mathbf{q}) = \int_{\mathbf{m}\in\Delta} D(\mathbf{m}, \mathbf{q}) \mathrm{d}\nu(\mathbf{m})
		\end{gathered}
	\end{equation} 
	where $ \Delta = \{ \mathbf{m} \in \mathbb{R}^{n} \mid m_{s} \geq 0 \text{ and } \sum_{s=1}^{n} m_{s} = 1 \} $ and $ \nu $ is the uniform distribution on $ \Delta $. In our work, our OFA-A estimator refers to the use of $ \mathbf{q}^{\text{OFA-A}} $ in Algorithm~\ref{alg:ofa}.
	
	\begin{restatable}{proposition}{ofaEstimator} \label{prop:ofa-a}
		$ \mathbf{q}^{\text{OFA-A}}_{s-1} \propto \frac{1}{\sqrt{(s)(n-s)}}  $ and $ \overline{D}(\mathbf{q}^{\text{OFA-A}}) \in O(1) $. In other words, our OFA-A estimator achieves the convergence rate of $ O(n\log n) $ simultaneously for all probabilistic values on average.
	\end{restatable}
	
	Our next proposition  provides a condition on $ \mu $ for semi-values such that our OFA-A estimator achieves the convergence rate of $ O(n\log n) $. In other words, we explicitly identify a subfamily of semi-values for which our OFA-A estimator achieves the currently best time complexity simultaneously.
	
	\begin{restatable}{proposition}{ofaConvergence}
		Our OFA-A estimator achieves the convergence rate of $ O(n\log n) $ simultaneously for all semi-values whose probability density functions exist and are bounded. Particularly, Beta Shapley values with $ \alpha,\beta \geq 1 $ all satisfy this condition.
	\end{restatable}
	
	
	To our knowledge, the previous theoretically-fastest estimator for the Shapley value is demonstrated by \citet[Theorem 6]{wang2023note} as $ O(n(\log n)^{2}) $. By contrast, our OFA-A estimator achieves the convergence rate of $ O(n\log n) $. Meanwhile, it also surpasses the previous best time complexity for Beta Shapley values with $ (\alpha=1, \beta>1) $ or $ (\alpha>1, \beta=1) $, which is $ O(n(\log n)^{3}) $ \citep[Proposition 4 and Remark 3]{li2024faster}. Remarkably, our OFA-A estimator enjoys this fastest convergence rate \emph{simultaneously for a broad subfamily of probabilistic values} .
	
	\begin{restatable}{proposition}{ofaWeighted} \label{prop:ofaWeighted}
		If $ p_{s} = a^{s-1}(1-a)^{n-s} $ with $ 0<a<1 $, which corresponds to the weighted Banzhaf value parameterized by $ w $, then $ D(\mathbf{m}, \mathbf{q}^{\text{OFA-A}}) \in O(n^{\frac{1}{2}}) $. In other words, the OFA estimator achieves the convergence rate of $ O(n^{\frac{3}{2}}\log n) $ simultaneously for all WB-$ a $ with $ 0<a<1 $.
	\end{restatable}
	
	
	The previous best convergence rate for weighted Banzhaf values is $ O(n\log n) $ \citep[Proposition 2]{li2023robust}, ours is slower by a factor of $ n^{\frac{1}{2}} $. Nevertheless, we will demonstrate that our generic estimator, which is expected to be faster than our OFA-A estimator, achieves the best convergence rate for all weighted Banzhaf values.
	
	\subsection{A Faster Generic Estimator}
	Our faster generic estimator (OFA-S) is obtained via optimizing $ \mathbf{q} $ for each \textbf{S}pecific $ \mathbf{m} $. Precisely, for each $ \mathbf{m} $, we have
	\begin{equation} \label{eq:optimal q}
		\begin{gathered}
			\mathbf{q}_{s-1}^{\text{OFA-S}} \propto \sqrt{\frac{m_{s}^{2}}{s} + \frac{m_{s+1}^{2}}{n-s}} \ \text{ where }\ \mathbf{q}^{\text{OFA-S}} =  \argmin_{\mathbf{q} \in \mathbb{R}^{n-3}} D(\mathbf{m}, \mathbf{q})\  \text{ s.t. }\ \sum_{j=1}^{n-3}q_{j} = 1,
		\end{gathered}
	\end{equation}
	which can be obtained using the Cauchy-Schwarz inequality. For the Shapley value, $ \mathbf{q}^{\text{OFA-S}} = \mathbf{q}^{\text{OFA-A}} $. Our next proposition specifies a sufficient condition for semi-values such that $ D(\mathbf{m},\mathbf{q}^{\text{OFA-S}}) \in O(1) $.
	
	\begin{restatable}{proposition}{genericConvergence}
		For semi-values,
		$ D(\mathbf{m}, \mathbf{q}^{\text{OFA-S}}) \in O(1) $ if i) $ \mu $ has a bounded probability density function or ii) $ \int_{(0,1)}\frac{1}{w(1-w)}\mathrm{d}\mu(w) < \infty$. Particularly, this condition covers all weighted Banzhaf values and Beta Shapley values with $ \alpha, \beta \geq 1 $.
	\end{restatable}
	
	All in all, we demonstrate that by sticking to the principle of maximum sample reuse and avoiding any amplifying scalars, 	
	we are able to establish a generic estimator that achieves the currently best convergence rate for any previously-studied semi-value.
	
	\subsection{A Connection between Probabilistic Values and Datamodels}
	A datamodel, proposed by \citet{ilyas2022datamodels}, is to learn an easy-to-interpret surrogate to represent a model output distribution related to a specific test example. In this circumstance, the set of players $ [n] $ is identified with all the available training data. Precisely, the feature coordinates $ \boldsymbol\theta^{*} \in \mathbb{R}^{n} $ imputed to every data point in $ [n] $ is defined to be the uniquely optimal solution (together with a bias $ b^{*} \in \mathbb{R} $) to the optimization problem
	\begin{equation} \label{op:datamodels}
		\begin{gathered}
			\argmin_{\boldsymbol\theta\in\mathbb{R}^{n}, b \in \mathbb{R}} \sum_{S \subseteq [n]} \eta_{s+1}\left( U(S) - b - \sum_{i\in S}\theta_{i} \right)^{2}
		\end{gathered}
	\end{equation}
	where $ \boldsymbol\eta\in \mathbb{R}^{n+1} $ is non-negative and $ \sum_{s=0}^{n} \eta_{s+1} > 0 $. The weight vector $ \boldsymbol\eta $ can be scaled such that the objective in the problem~\eqref{op:datamodels} can be treated as an expectation, and thus the objective can be approximated through sampling a sufficient number of subsets, upon which an estimate of $(\thetav^*, b^*)$ can be obtained. We show below that $ \boldsymbol\theta^{*} $ to a family of such least square regressions can be cast as some probabilistic values if it is the pairwise differences $ \theta_{j}^{*} - \theta_{k}^{*} $ (for every $ j, k \in [n] $) that matter. 
	
	\begin{restatable}{theorem}{connection} \label{thm:connection}
		Let $ (b^{*}, \boldsymbol\theta^{*}) $ be the uniquely optimal solution to the problem~\eqref{op:datamodels} where $ \eta_{s} = p_{s-1} + p_{s} $ for $ 2\leq s\leq n $. Then, there is
		\begin{equation}
			\begin{gathered}
				\theta_{j}^{*} - \theta_{k}^{*} = \phi_{j} - \phi_{k}\ \text{ for every }\ j, k \in [n] .
			\end{gathered}
		\end{equation}
	\end{restatable}
	
	In other words, $\thetav^* = \phiv + c$ for some constant $c\mathbf{1}$; $ \mathbf{1} \in \mathbb{R}^{n} $ is the all-one vector. 
	When using datamodels to detect similar training examples to a given target, what matters is the relative order of components in $ \boldsymbol\theta^{*} $. Meanwhile, \citet{ilyas2022datamodels} showed that the corresponding performance depends on the choice of the weight vector $ \etav $. Therefore, our  OFA-A estimator serves as a sufficient proxy for a range of $ \{ \boldsymbol\theta^{*} \} $ and would facilitate the fine-tuning of $ \etav $ when using datamodels to detect similar training examples.

	\paragraph{When $ \boldsymbol\theta^{*} $ Can Be Recovered From $ \boldsymbol\phi $}
	Interestingly, for specific choices of $ \mathbf{p} \in \mathbb{R}^{n} $ and $ \boldsymbol\eta \in \mathbb{R}^{n+1} $, it holds that $ \boldsymbol\theta = \boldsymbol\phi $. Theorem~\ref{thm:connection} can be seen as an extension to the previous result stated in the below.
	
	\begin{proposition}[{\cite[Proposition 4]{marichal2011weighted}}] \label{prop:weighted and least}
		Suppose $ 0 < a < 1 $ is given, if $ p_{j} = a^{j-1}(1-a)^{n-j} $ for $ 1\leq j \leq n $ and $ \eta_{k} = a^{k-2}(1-a)^{n-k} $ for $ 1\leq k \leq n+1 $, which leads to $ \eta_{s} = p_{s-1} + p_{s} $ for $ 2\leq s\leq n$, there is 
		\begin{equation}
			\begin{gathered}
				\boldsymbol\theta^{*} = \boldsymbol\phi .
			\end{gathered}
		\end{equation}
	\end{proposition}
	It is worth pointing out that $ \boldsymbol\phi $ in Proposition~\ref{prop:weighted and least} is exactly the weighted Banzhaf value parameterized by $ a $, i.e., WB-$ a $. Even more, under the same setting, we can even solve  a group of datamodels with $ \ell_{1} $ or $ \ell_{2} $ regularization simultaneously.

	\begin{corollary} \label{cor:datamodels}
		Under the setting of Proposition~\ref{prop:weighted and least}, let $ \boldsymbol\theta^{*} $ be the unique optimal solution to
		\begin{equation} \label{op:regularized datamodel}
			\begin{gathered}
				\argmin_{\boldsymbol\theta\in\mathbb{R}^{n}, b\in\mathbb{R}} \left( \sum_{S\subseteq[n]} \eta_{s+1}\left( U(S) - b - \sum_{i\in S}\theta_{i} \right)^{2} \right) + \frac{\lambda}{a(1-a)} \mathcal{R}(\boldsymbol\theta) 
			\end{gathered}
		\end{equation}
		where $ \lambda > 0 $,
		the following are true about the relation between $ \boldsymbol\theta^{*} $ and $ \boldsymbol\phi $:
		\begin{enumerate}
			\item If $ \mathcal{R}(\boldsymbol\theta) = \| \boldsymbol\theta \|_{2}^{2} $, then
			\begin{equation}
				\begin{gathered}
					\boldsymbol\theta^{*} = \left( 1+\frac{\lambda}{a(1-a)} \right)^{-1}\boldsymbol\phi .
				\end{gathered}
			\end{equation}
		
			\item If $ \mathcal{R}(\boldsymbol\theta) = \| \boldsymbol\theta \|_{1} $, then
			\begin{equation}
				\begin{gathered}
					\boldsymbol\theta^{*} = \sign(\boldsymbol\phi)\max\left( 0, |\boldsymbol\phi| - \frac{\lambda}{2a(1-a)} \right) .
				\end{gathered}
			\end{equation}
		\end{enumerate}
		All operations are element-wise.
	\end{corollary}
	This corollary is immediate by combining Proposition~\ref{prop:weighted and least}, and Theorem 2.2 by \citet{saunshi2022understanding}. We comment that replacing $ x_{i} $ by $ 2x_{i} - 1 $, i.e., mapping $ 0 $ and $ 1 $ into $ -1 $ and $ 1 $, respectively, in $ \phi_{\{ i \}}(x) $ used by \citet{saunshi2022understanding} yields $ v_{\{ i \}}(x) $ used by \citet{marichal2011weighted}.	
	A remarkable implication  of the combination of Corollary~\ref{cor:datamodels} and our proposed OFA-A estimator is that we can solve a group of regularized datamodels covered by the problem~\eqref{op:regularized datamodel} \textit{simultaneously}! For example, the coefficient $ \lambda $ can be finetuned by running Algorithm~\ref{alg:ofa} just once.

	\section{Experiments} \label{sec:experiments}	
	In this section, we are to verify i) the simultaneous efficiency of our OFA-A estimator and ii) the faster convergence rate of our OFA-S estimator compared with the considered baselines and our OFA-A estimator. Particularly, if $ D(\mathbf{m}, \mathbf{q}) $ is effective in determining the convergence rate of Algorithm~\ref{alg:ofa}, our OFA-S estimator is expected to be faster than our OFA-A estimator.  All the experiments are conducted using CPUs.
	
	We use two types of utility functions for this end: i) following the experiment settings of \citep{li2024faster}, $ U(S) $ is set to be the cross-entropy of LeNets trained on $ S $ on the classification datasets FMNIST, MNIST and iris; to obtain the exact values, the number of training data $ n $ is set to be $ 24 $; ii) $ U $ is defined to be the sum of unanimity (SOU) games, i.e., $ U(S) = \sum_{j=1}^{d} \alpha_{j} \mathds{1}_{S_{j}\subseteq S} $ where each $ \emptyset\subsetneq S_{j} \subsetneq [n] $ is randomly sampled, for which each semi-value can be computed by $ \phi_{i} = \sum_{j=1}^{d} \alpha_{j} \int_{[0,1]}w^{s_{j}-1}\mathrm{d}\mu(w) $; specifically, we set $ n \in \{ 64, 128, 256 \} $ with $ d = n^{2} $, which implies that the implemented SOU games require $ n^{2} $ utility evaluations to compute semi-values exactly.
	The random seed inside each utility function is fixed as $ 2024 $, and thus each $ U $ is deterministic.
	
	For the simplicity of presenting our empirical results, we use the area under the convergence curve (AUCC) to assess the convergence quality of estimators, and thus the smaller the better. For $ n=24 $, the value of each player is approximated using $ 20,000 $ utility evaluations, and we compute the AUCCs as $ \frac{1}{100}\sum_{j=1}^{100}\frac{\| \hat{\boldsymbol\phi}^{(200j)} - \boldsymbol\phi \|_{2}}{\| \boldsymbol\phi \|_{2}} $ where $ \hat{\boldsymbol\phi}^{(200j)} $ refers to the estimate using $ 200j $ utility evaluations for each player. For $ n \in \{ 64, 128, 256 \} $, the value of each player is approximated using $ 2,000 $ utility evaluations, and the corresponding AUCCs are calculated as $ \frac{1}{100}\sum_{j=1}^{100}\frac{\| \hat{\boldsymbol\phi}^{(20j)} - \boldsymbol\phi \|_{2}}{\| \boldsymbol\phi \|_{2}} $. All the AUCCs are reported with standard deviation using $ 30 $ different random seeds from $ \{ 0,1,2,\dots, 29 \} $.

	\begin{figure}[t]
		\centering
		\begin{tabular}{ccc}
			\multicolumn{3}{c}{\includegraphics[width=0.8\linewidth]{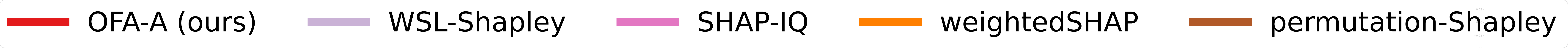}} \\
			\includegraphics[width=0.3\linewidth]{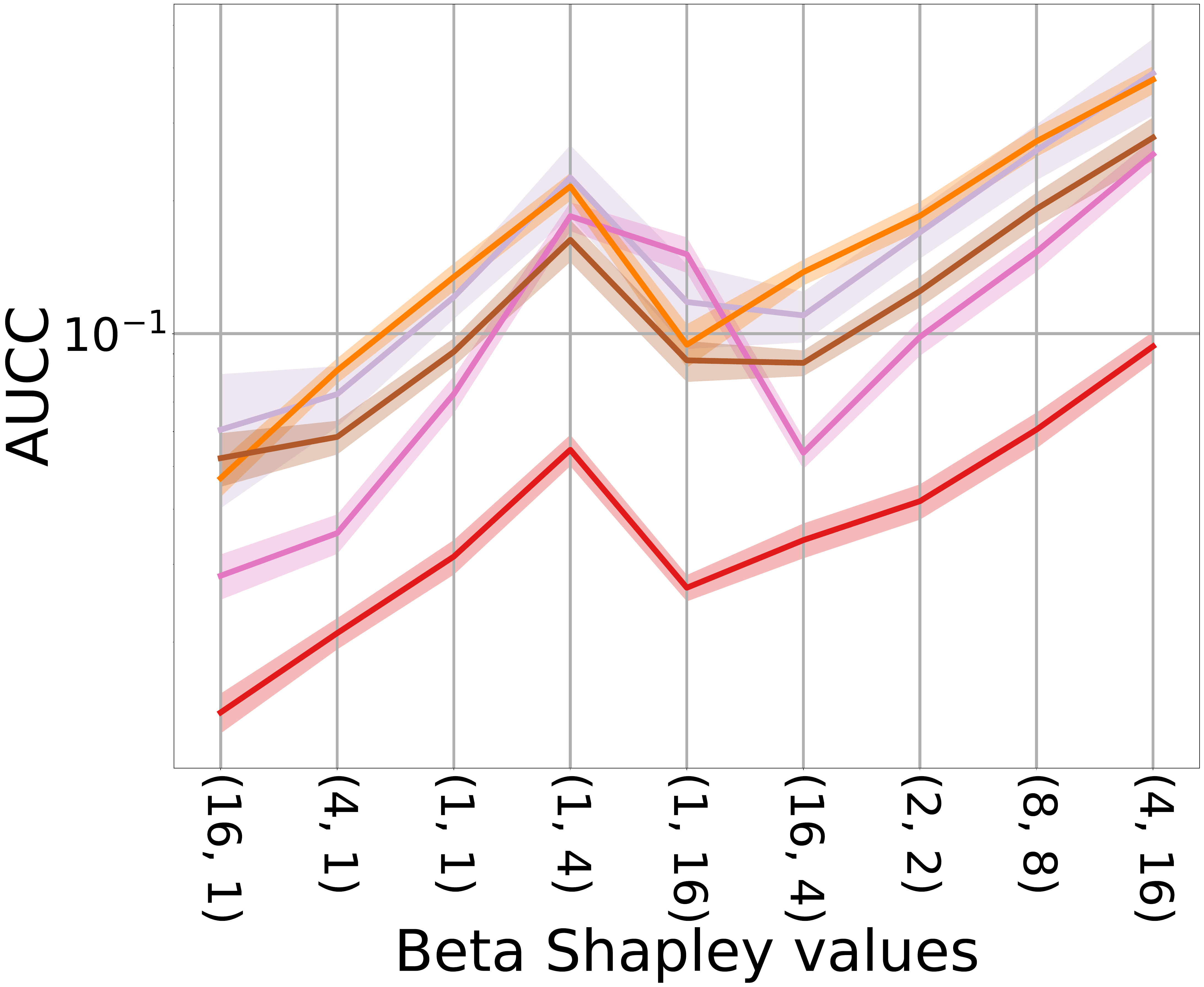} & \includegraphics[width=0.3\linewidth]{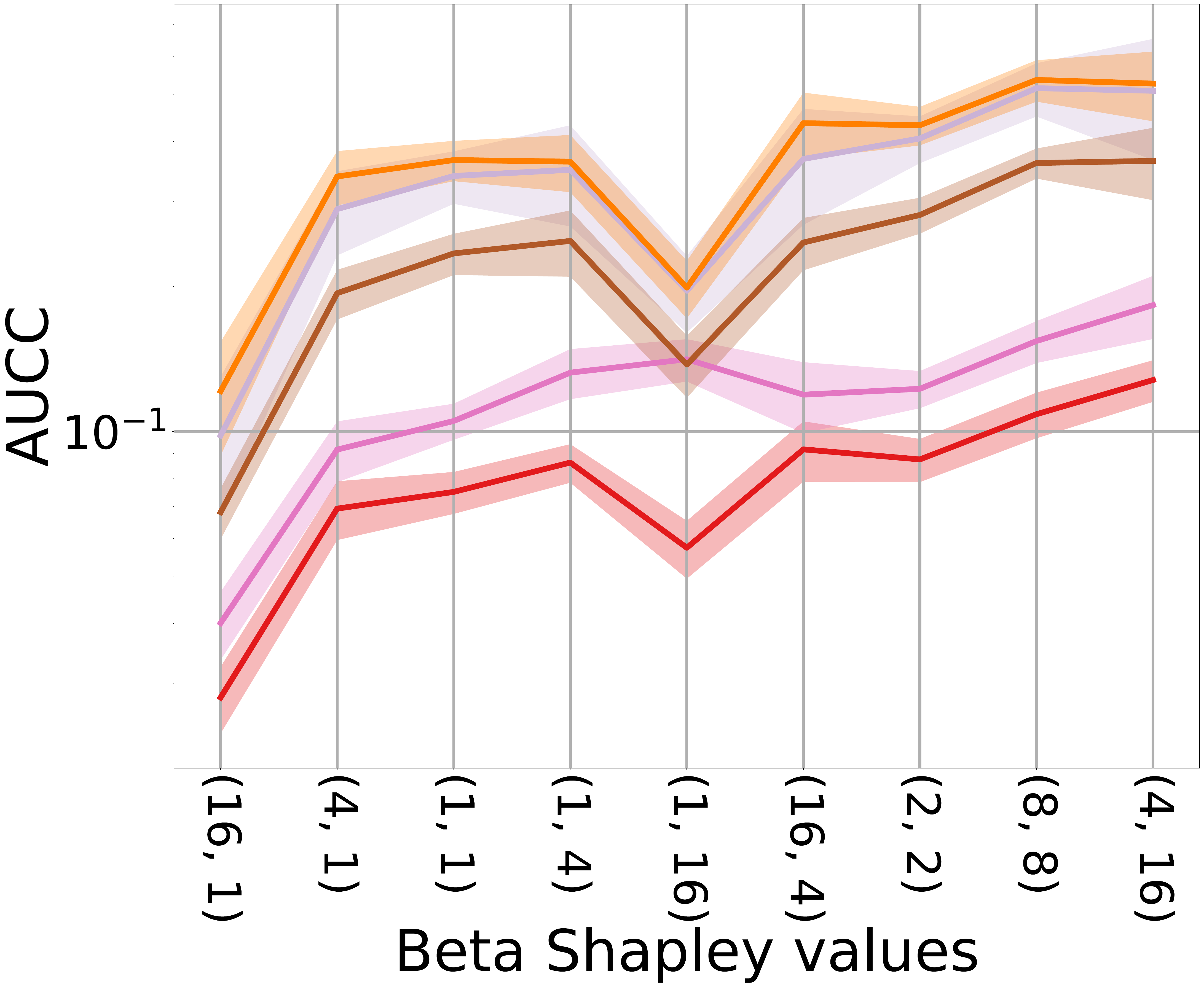} &
			\includegraphics[width=0.3\linewidth]{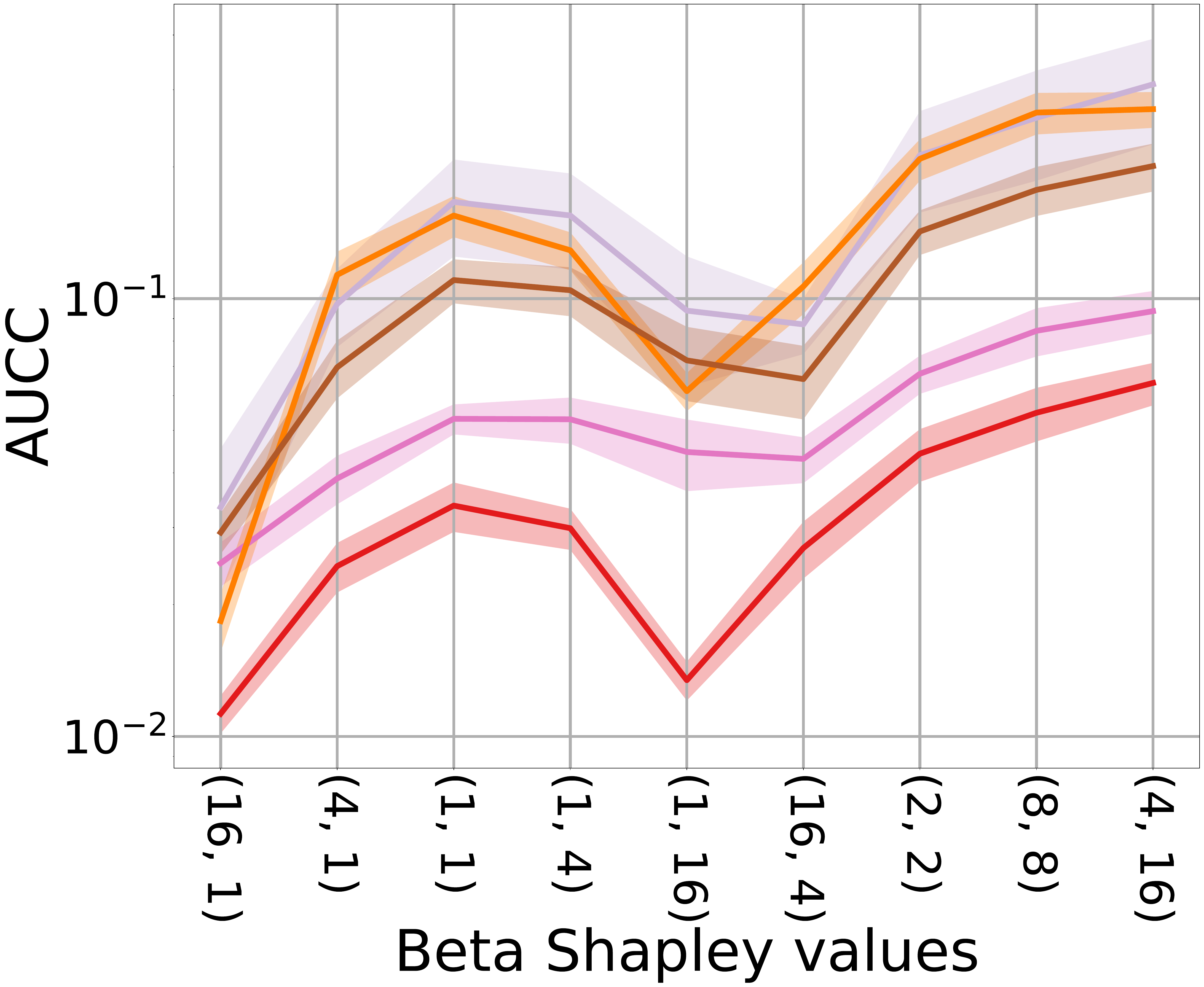} \\
			\includegraphics[width=0.3\linewidth]{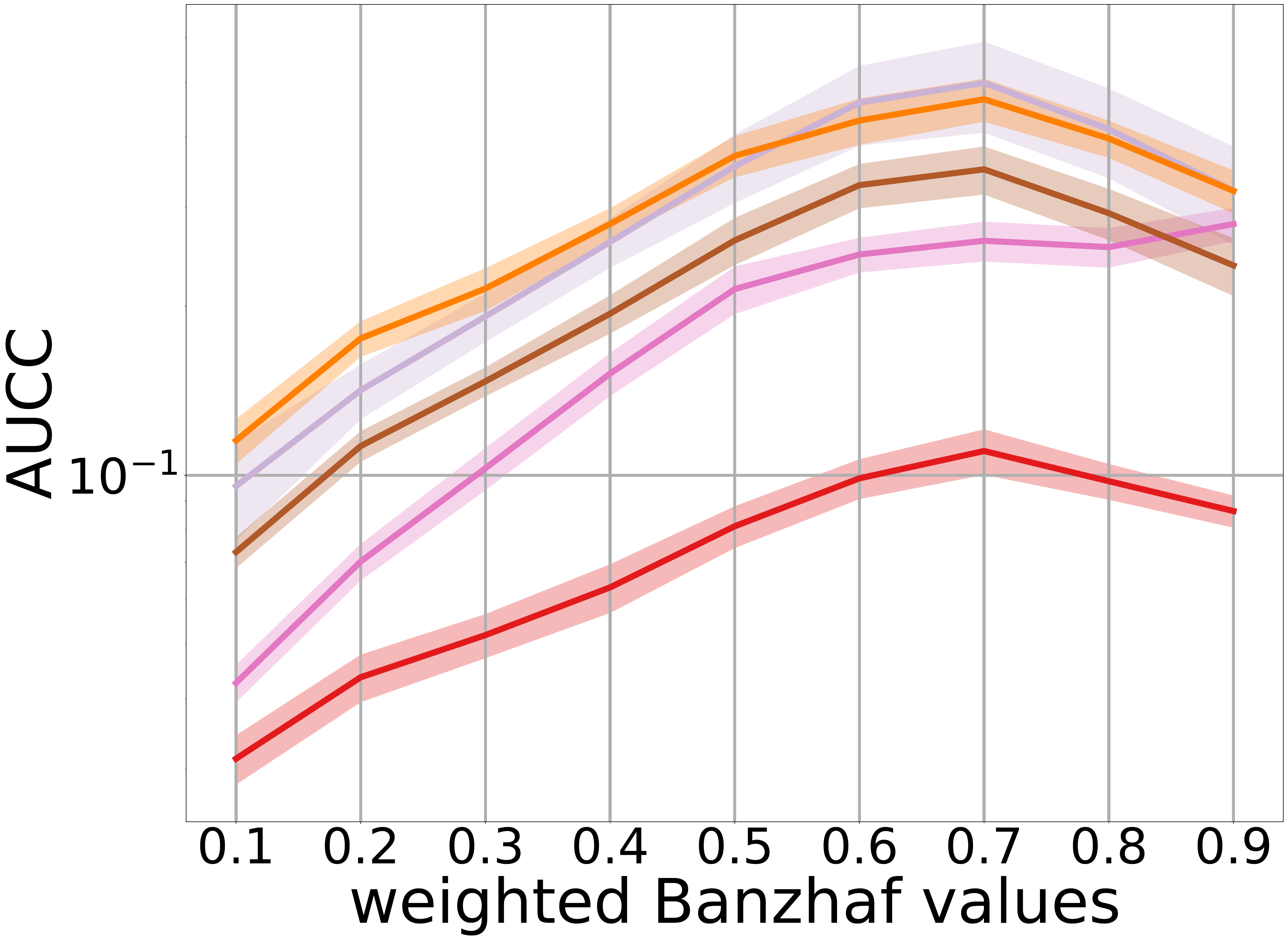} & \includegraphics[width=0.3\linewidth]{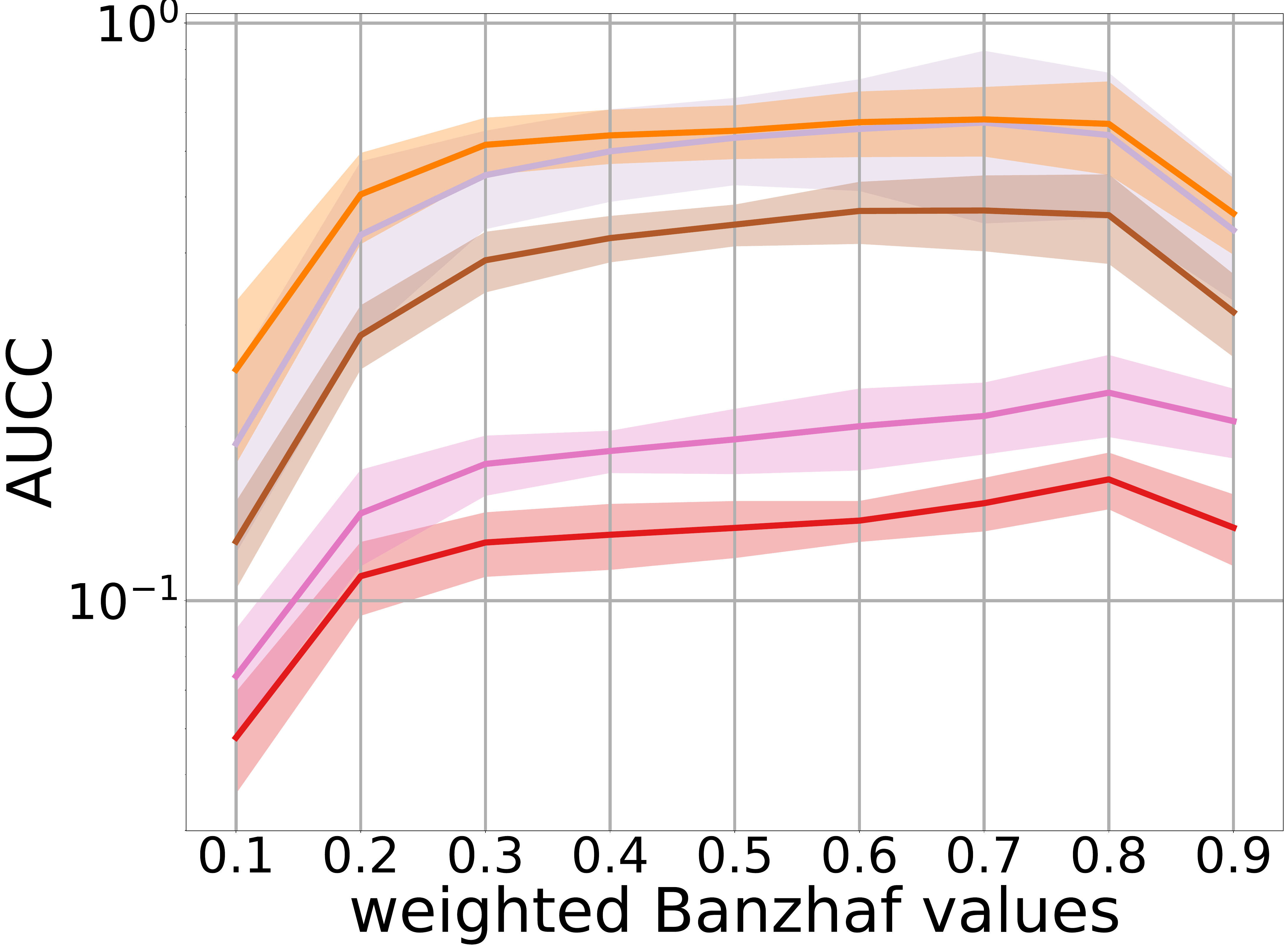} &
			\includegraphics[width=0.3\linewidth]{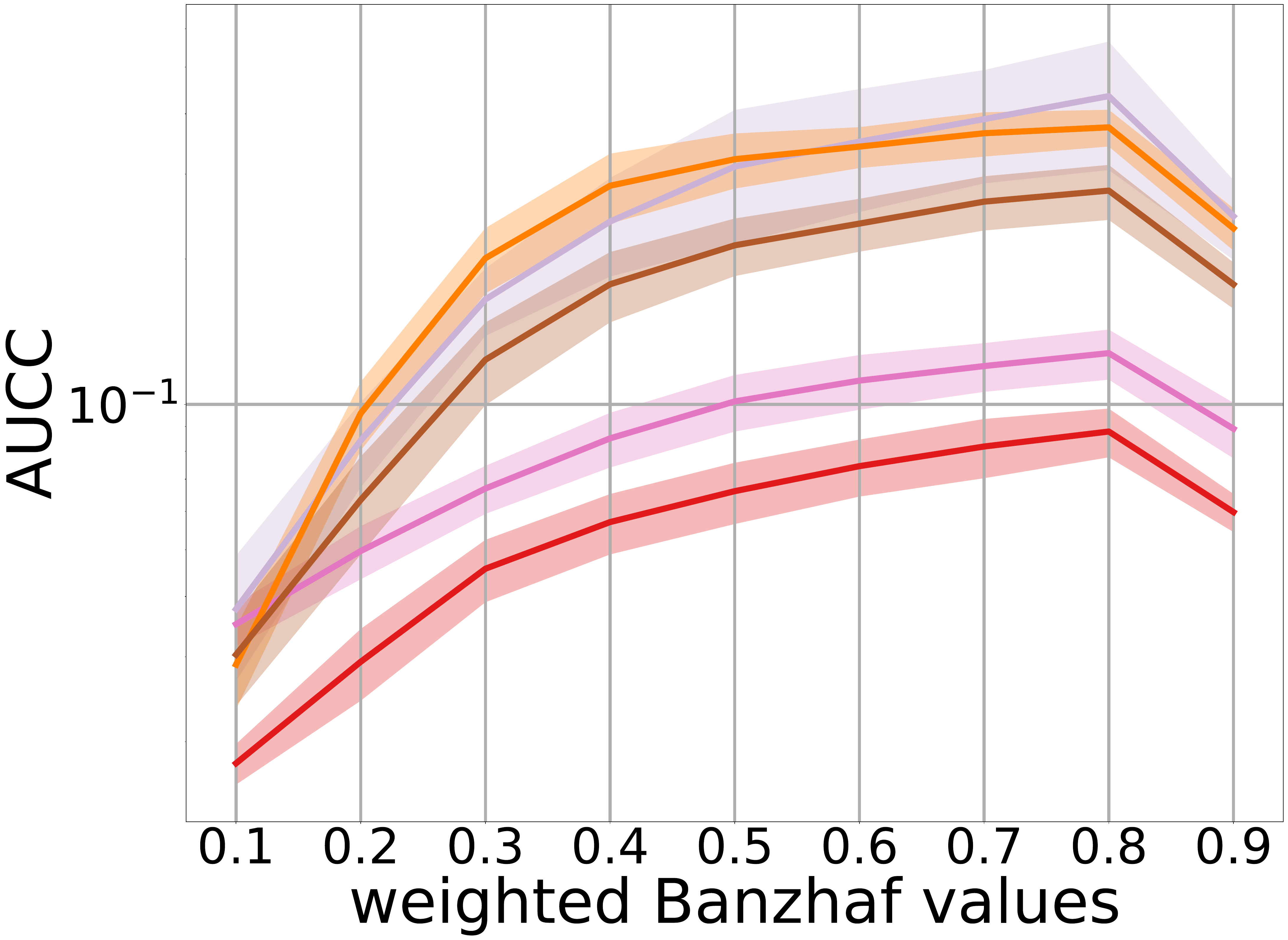} \\
			iris ($ n=24 $) & MNIST ($ n=24 $) & FMNIST ($ n=24 $)\\
			\includegraphics[width=0.3\linewidth]{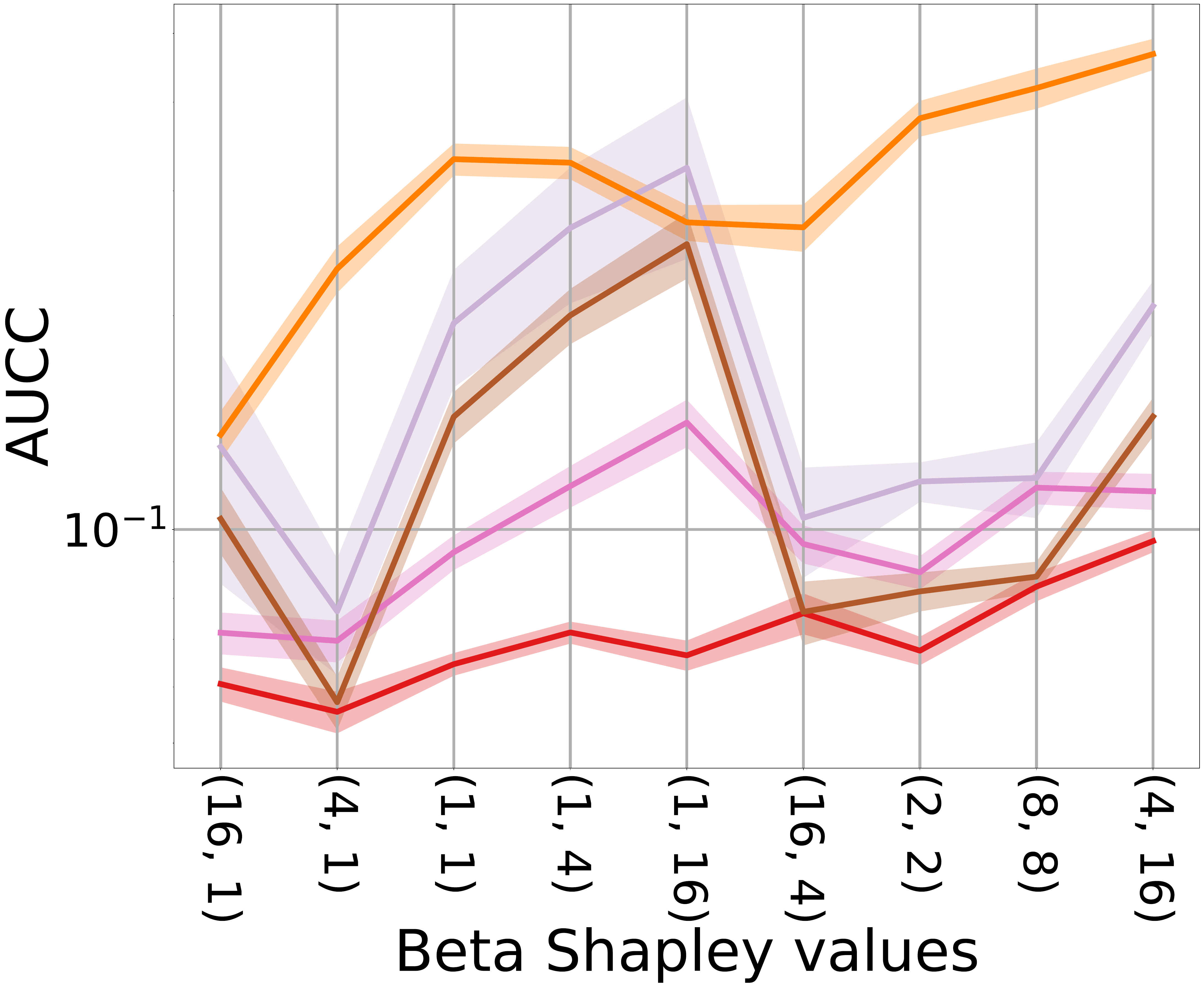} & \includegraphics[width=0.3\linewidth]{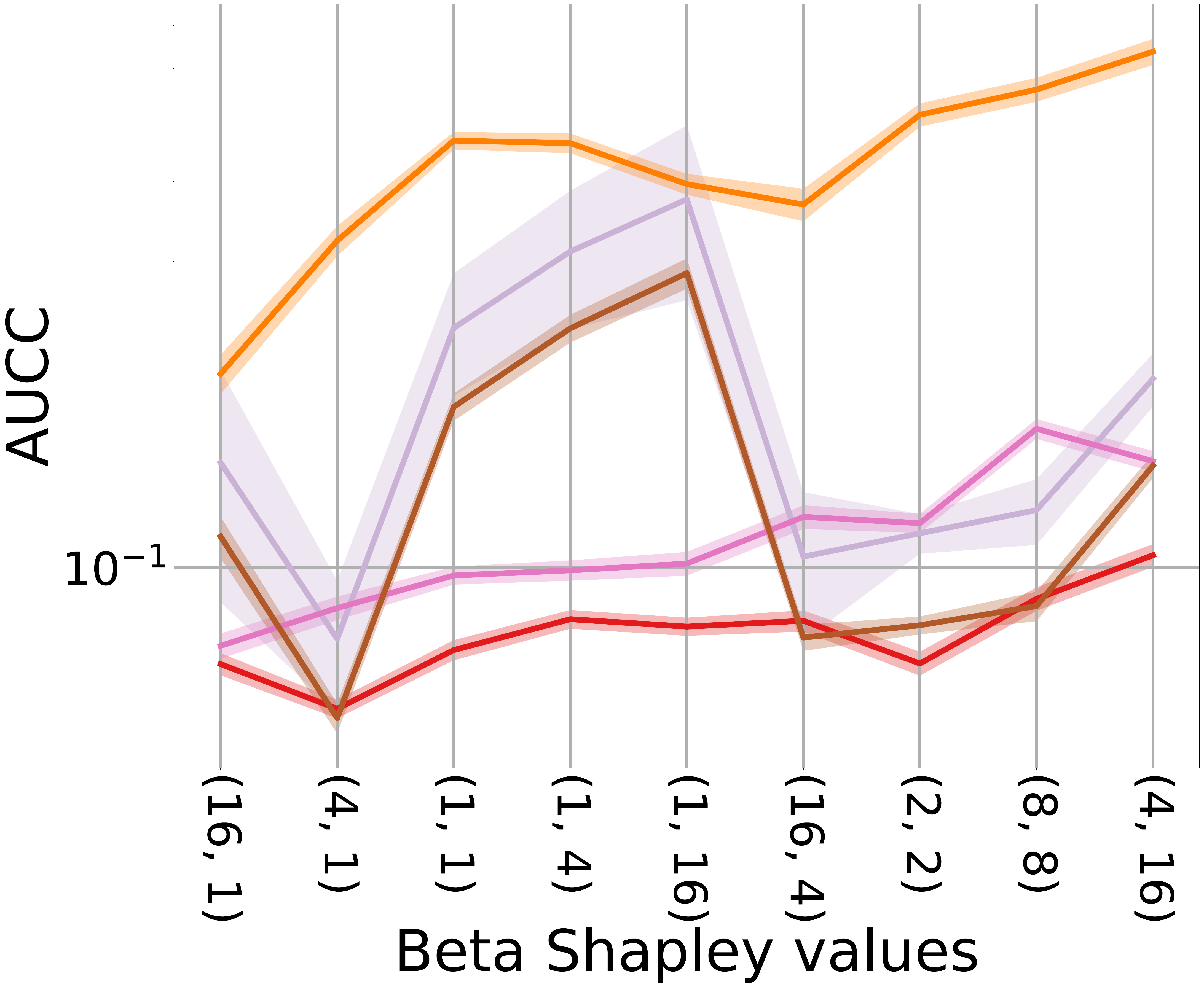} &
			\includegraphics[width=0.3\linewidth]{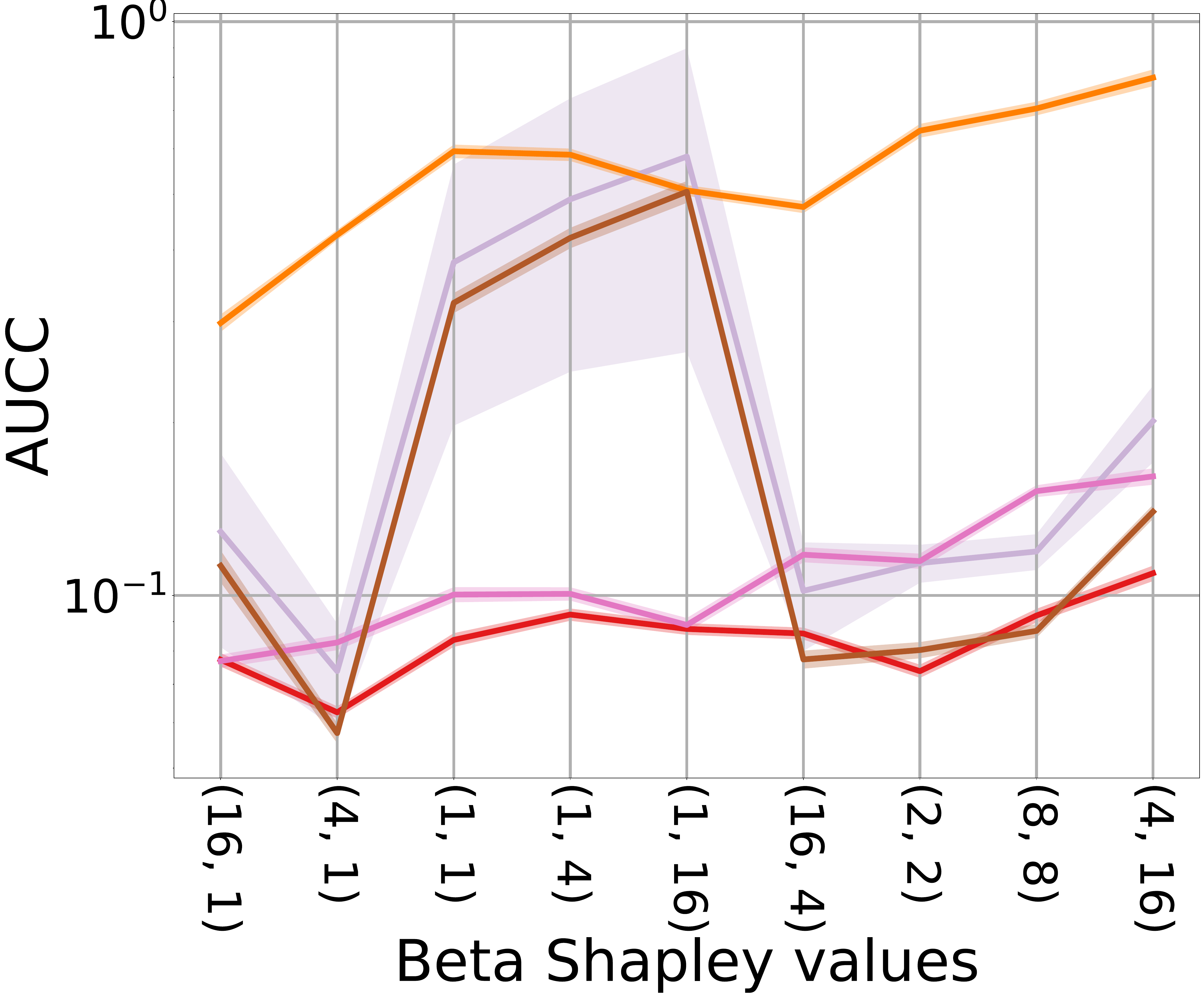} \\
			\includegraphics[width=0.3\linewidth]{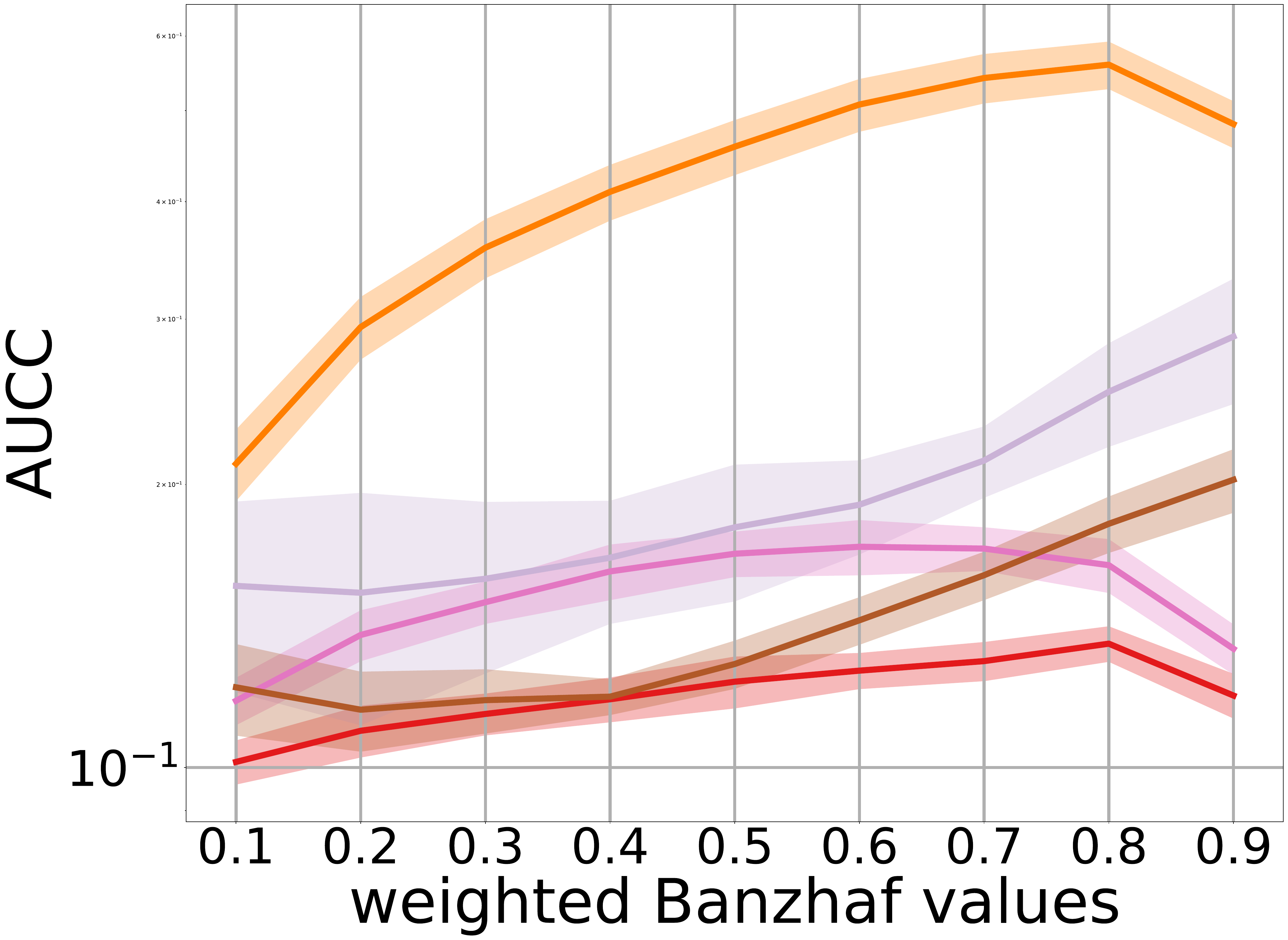} & \includegraphics[width=0.3\linewidth]{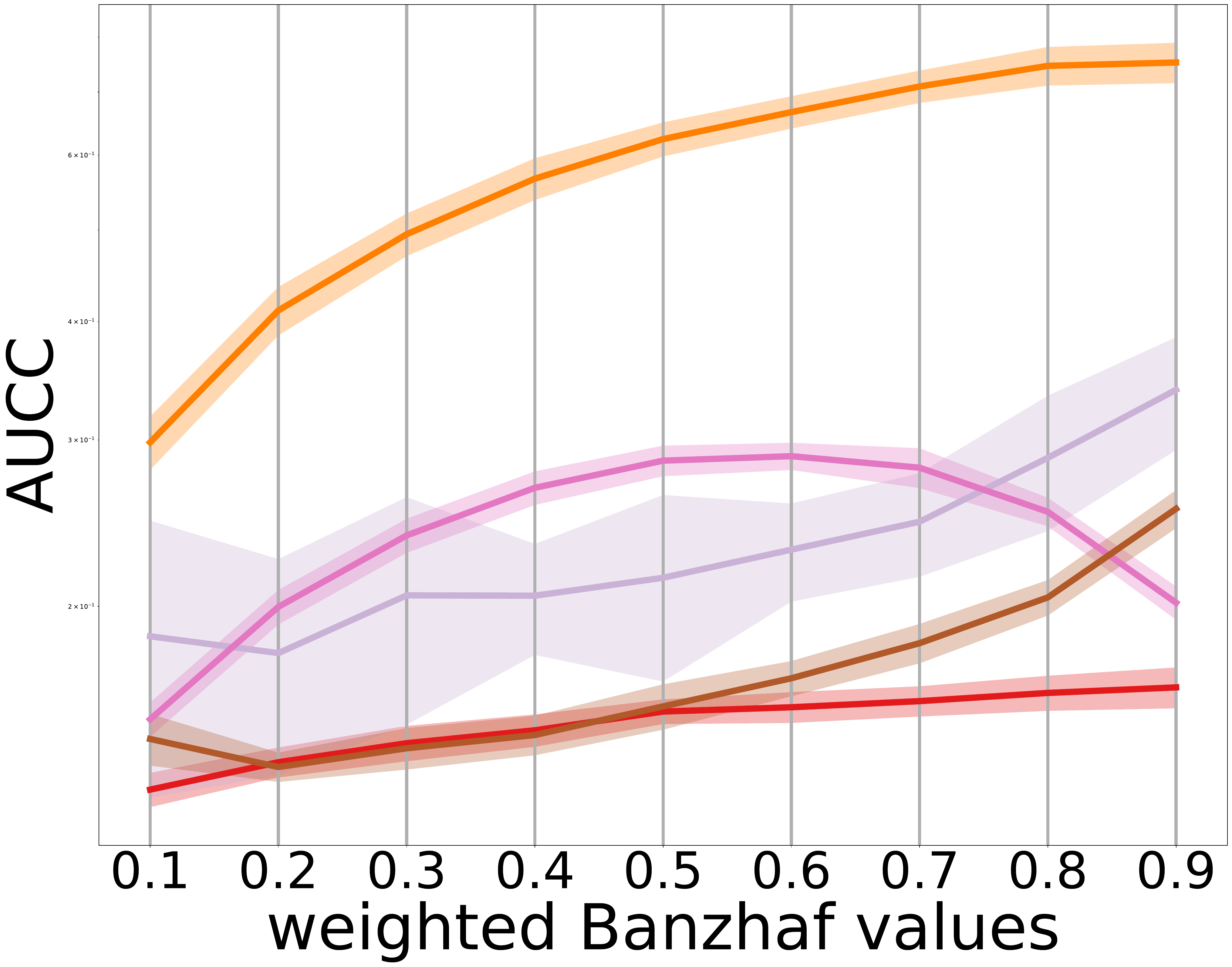} &
			\includegraphics[width=0.3\linewidth]{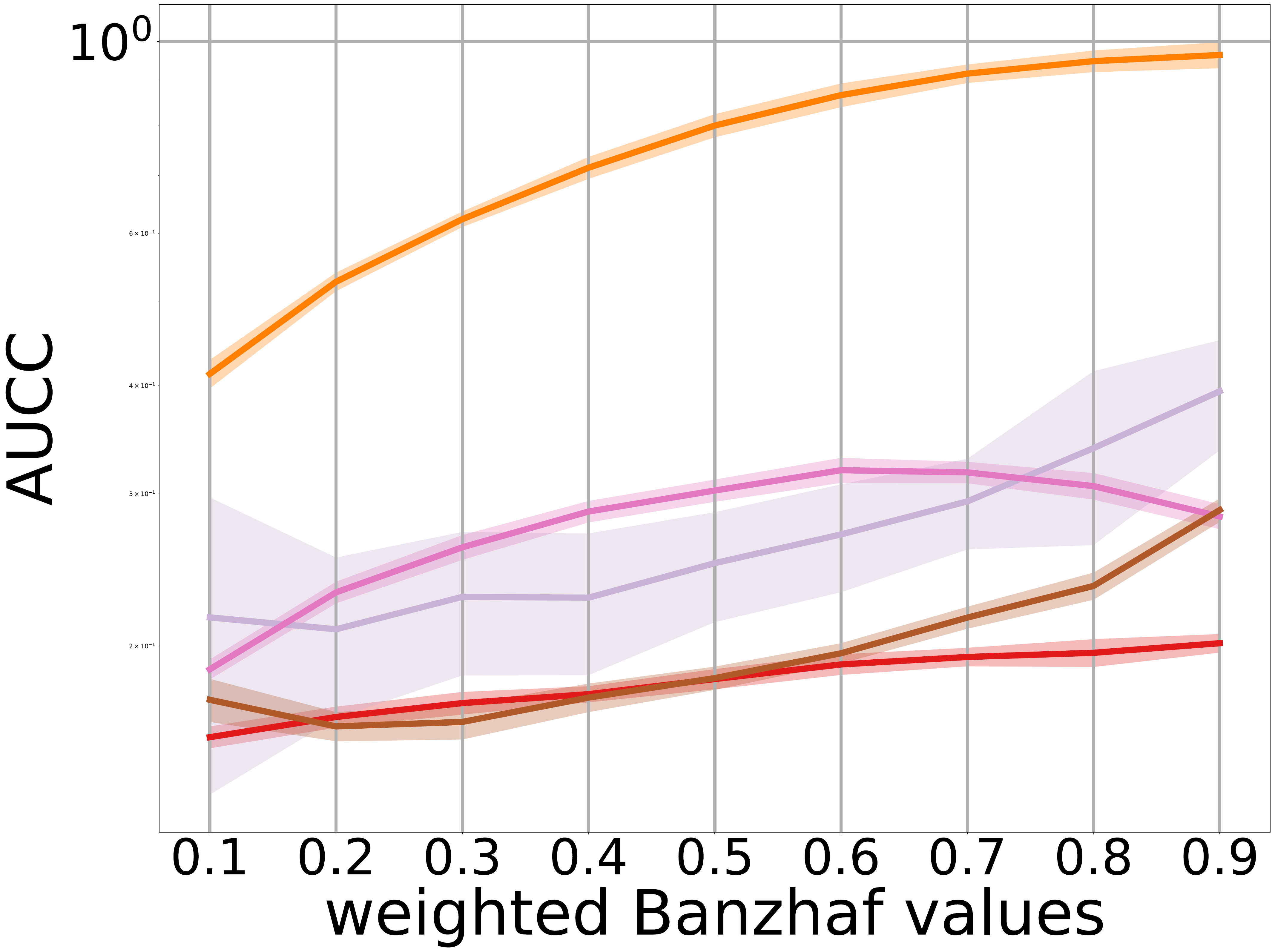}\\
			SOU ($ n=64 $) & SOU ($ n=128 $) & SOU ($ n=256 $)
		\end{tabular}
		\vspace{-.5em}
		\caption{Comparison of one-for-all estimators using six utility functions. All the AUCCs are reported with standard deviation using $ 30 $ random seeds. Smaller AUCC indicates faster convergence rate.}
		\vspace{-1em}
		\label{fig:ofa}
	\end{figure}

	\paragraph{Verification of Our OFA-A Estimator}
	For our OFA-A estimator where we substitute $ \mathbf{q}^{\text{OFA-A}} $, which is defined in Proposition~\ref{prop:ofa-a}, into Algorithm~\ref{alg:ofa}, we choose the baselines according to Figure~\ref{fig:motivation}. The selected baselines include WSL-Shapley \citep{kwon2022beta}, SHAP-IQ \citep{fumagalli2024shap}, weightedSHAP \citep{kwon2022weightedshap} and permutation-Shapley \citep{castro2009polynomial}. The corresponding results are reported in Figure~\ref{fig:ofa}. Overall, our OFA-A estimator performs the best on all the employed $ 18 $ probabilistic values, which verify the simultaneous efficiency of our OFA-A estimator.

	\begin{figure}[t]
		\centering
		\begin{tabular}{ccc}
			\multicolumn{3}{c}{\includegraphics[width=0.97\linewidth]{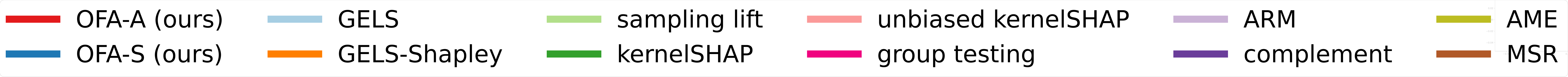}} \\
			\includegraphics[width=0.3\linewidth]{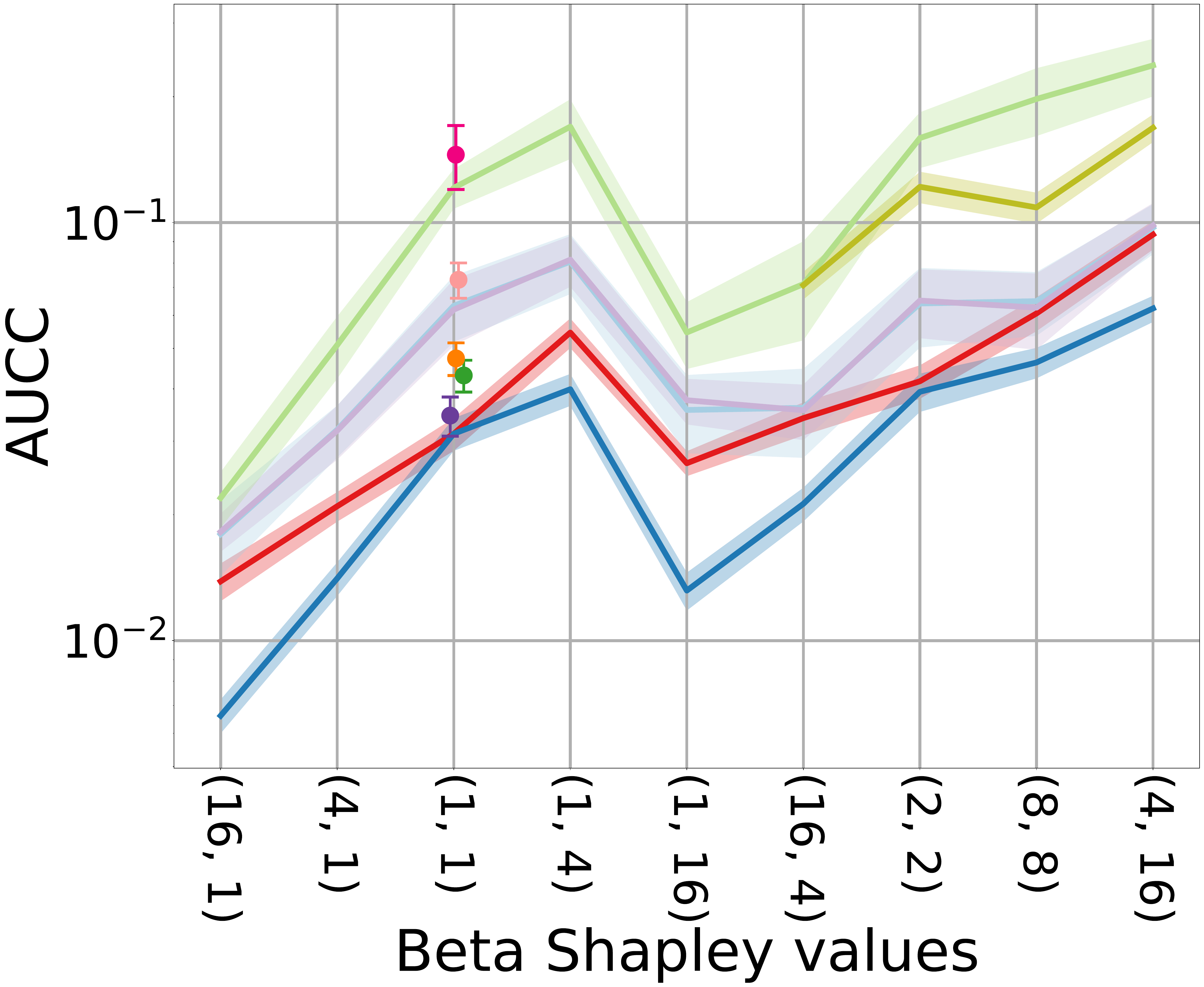} & \includegraphics[width=0.3\linewidth]{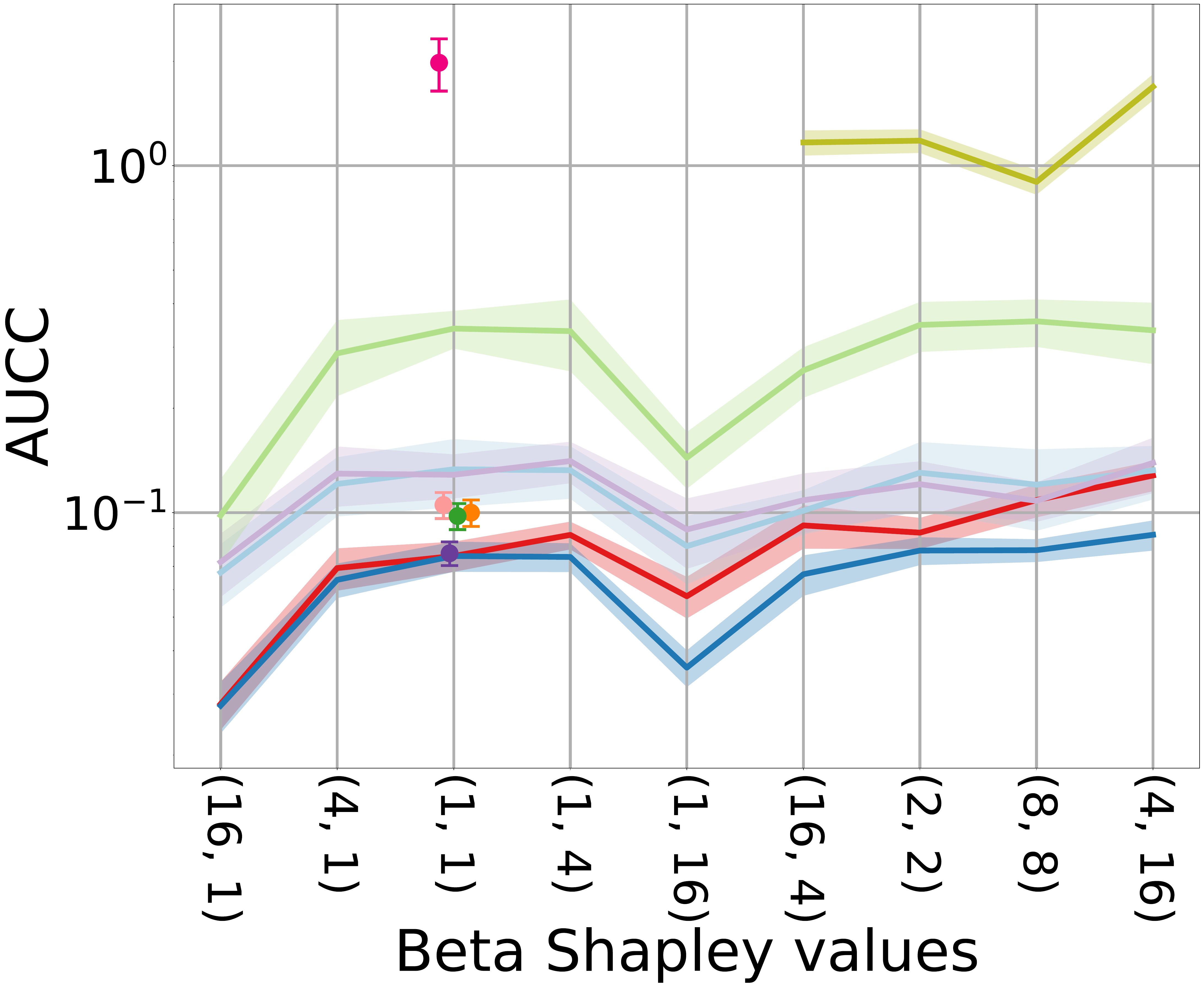} &
			\includegraphics[width=0.3\linewidth]{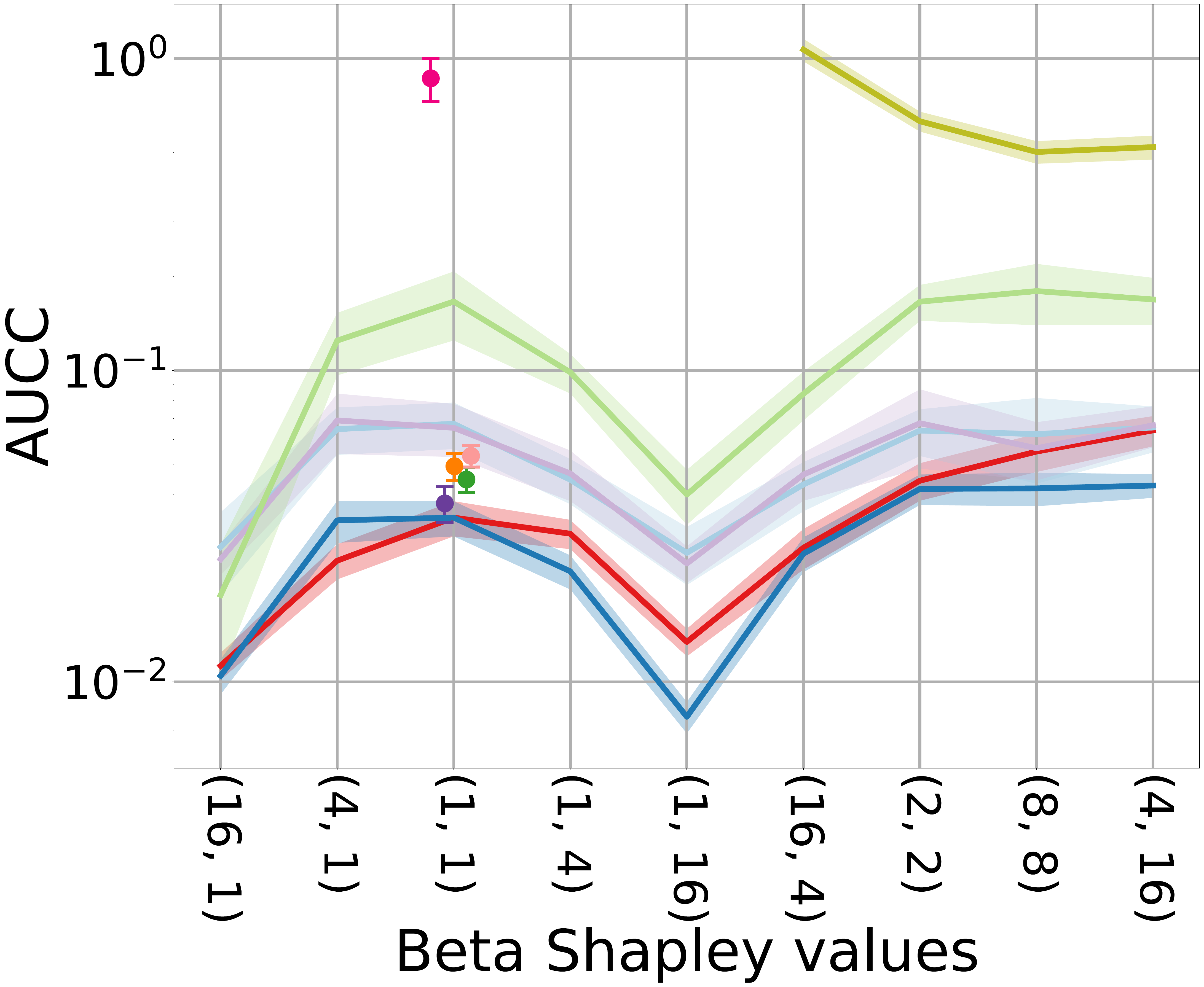} \\
			\includegraphics[width=0.3\linewidth]{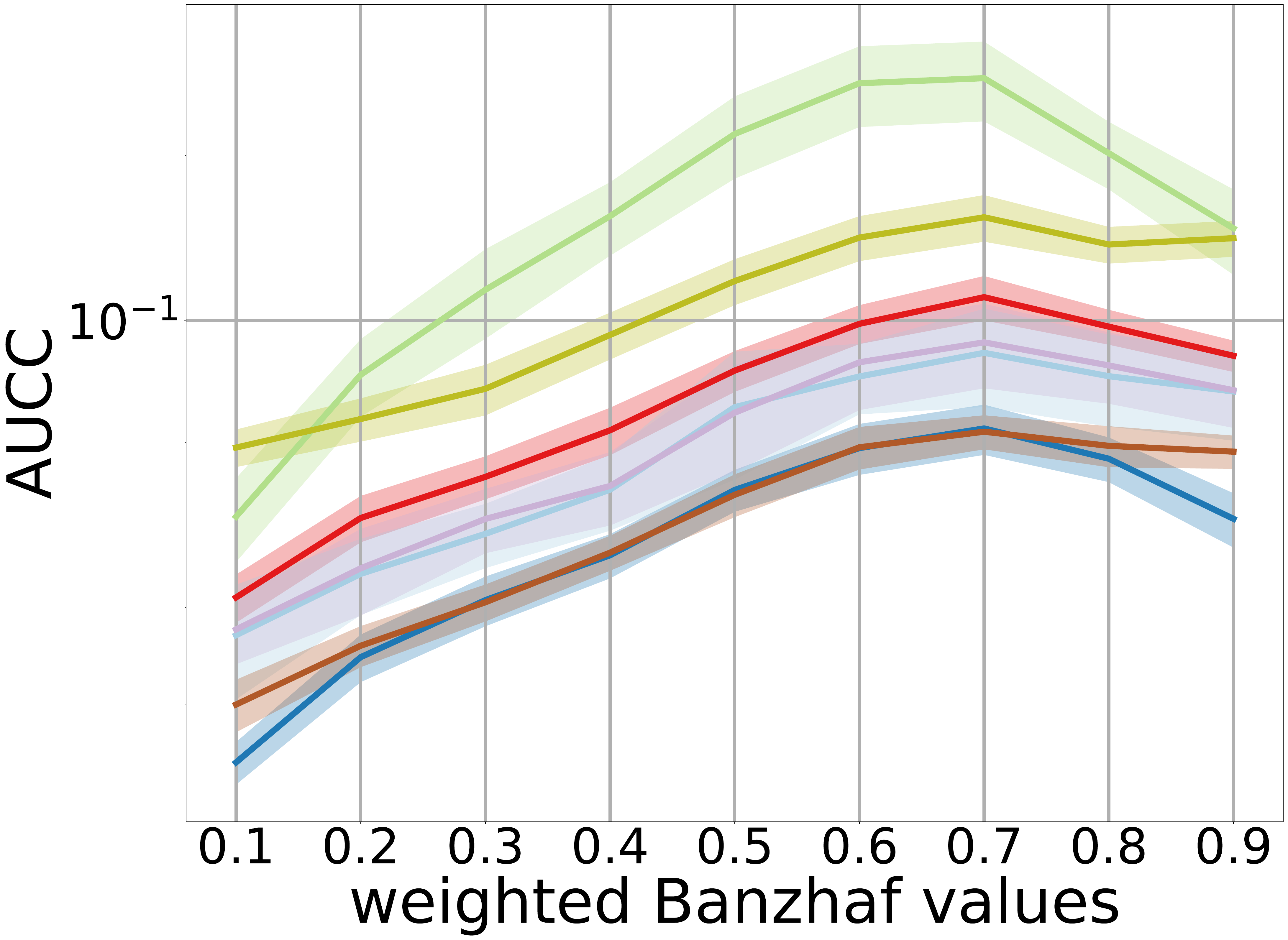} & \includegraphics[width=0.3\linewidth]{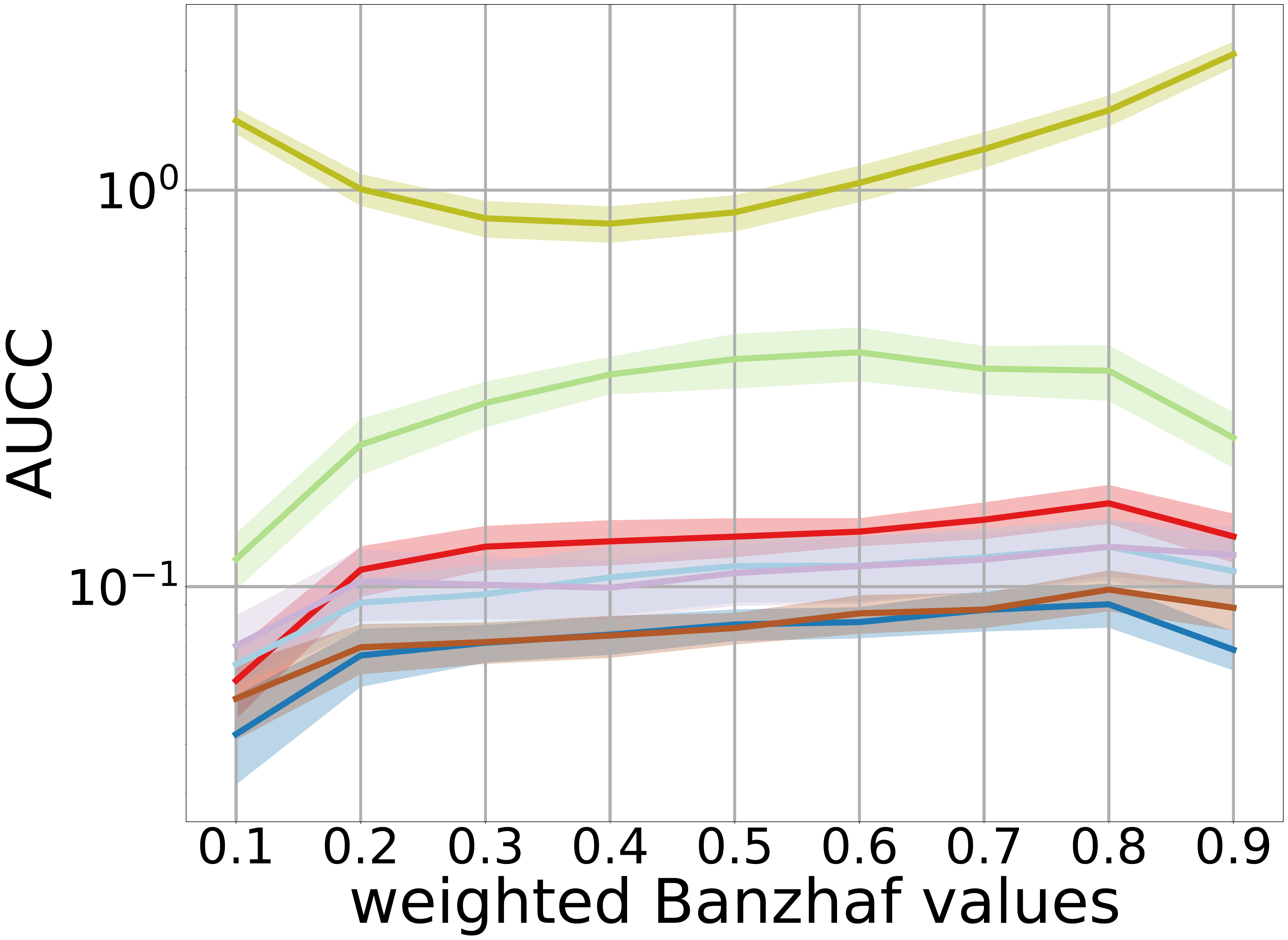} &
			\includegraphics[width=0.3\linewidth]{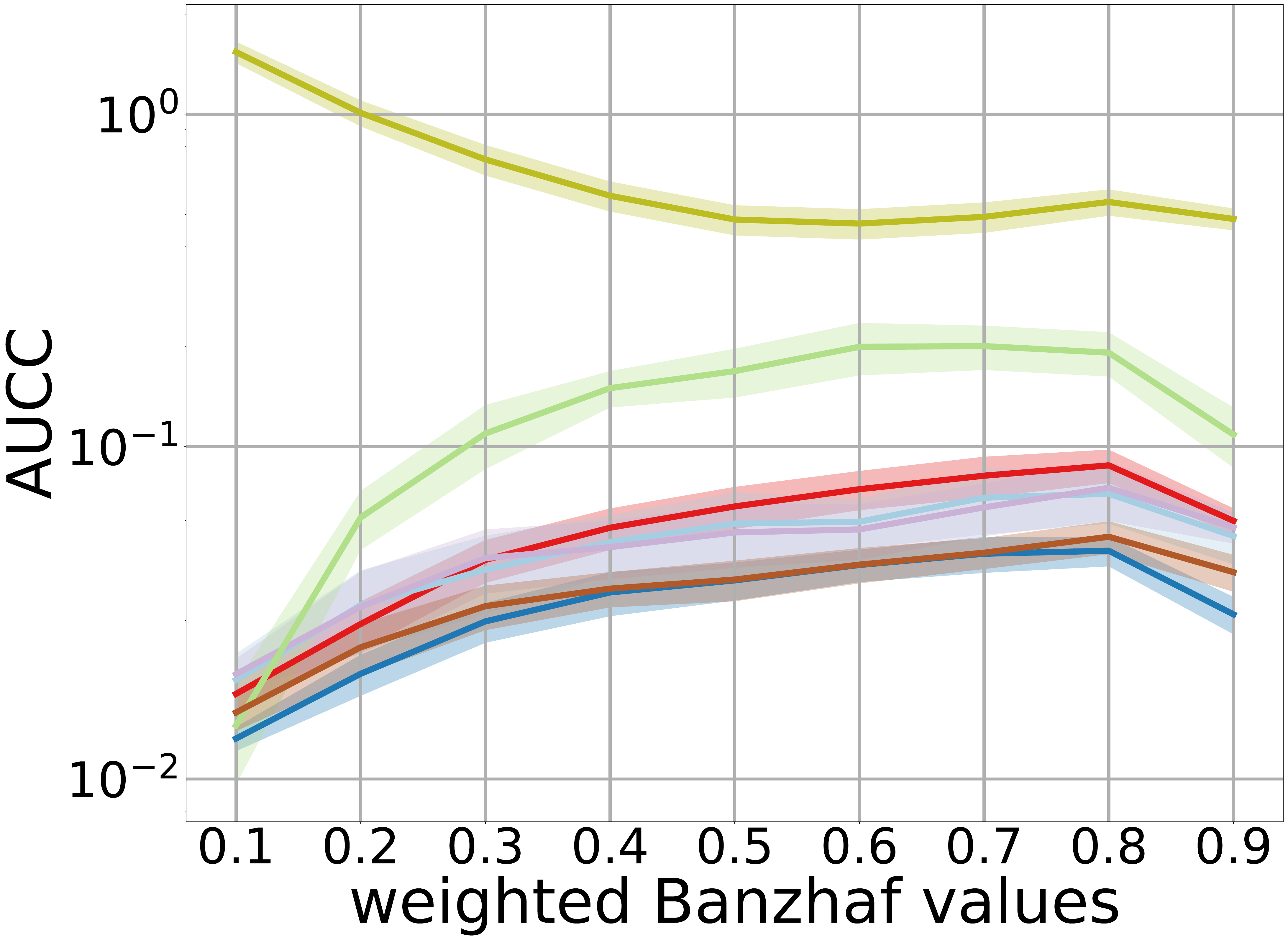} \\
			iris ($ n=24 $) & MNIST ($ n=24 $) & FMNIST ($ n=24 $)\\
			\includegraphics[width=0.3\linewidth]{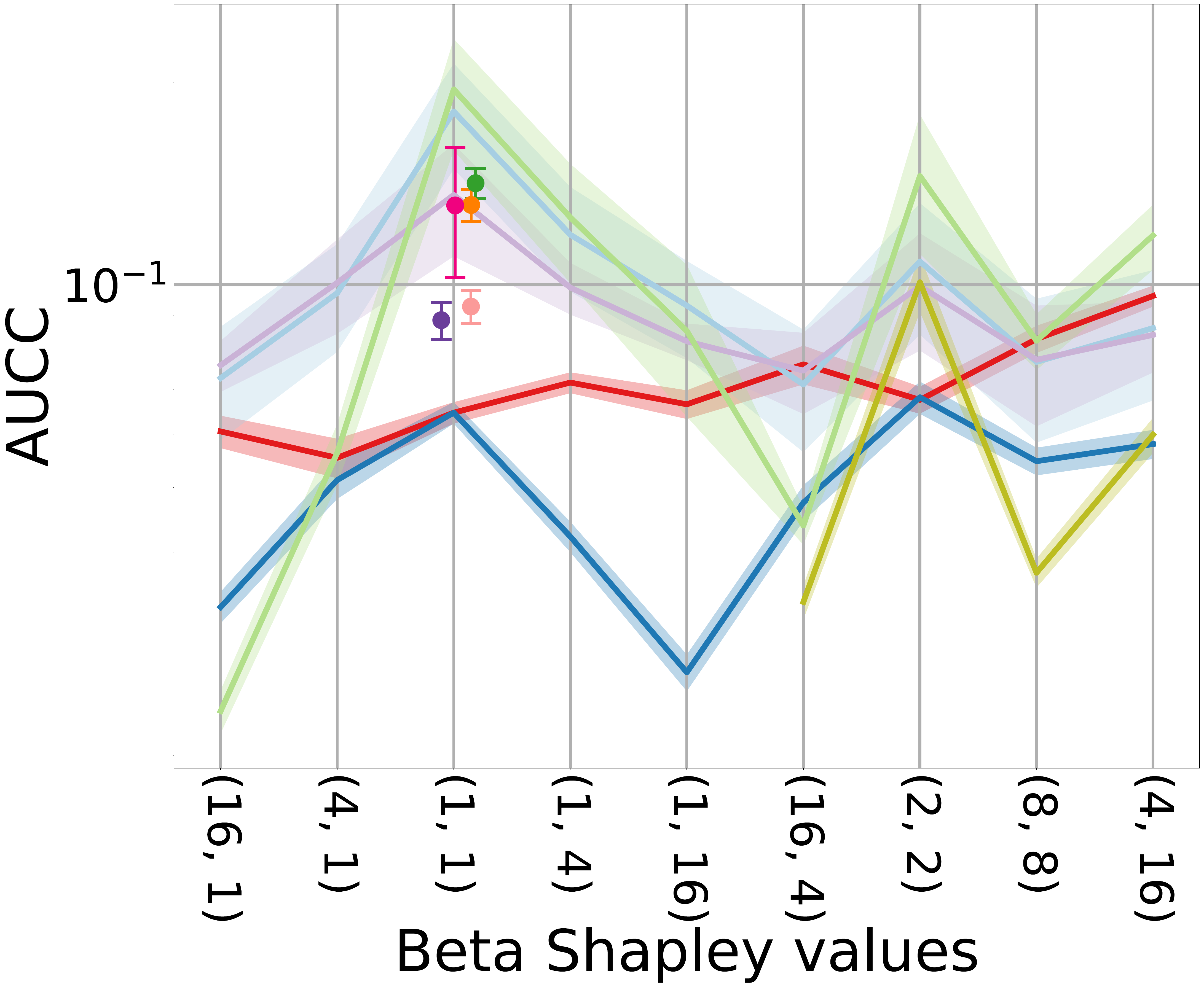} & \includegraphics[width=0.3\linewidth]{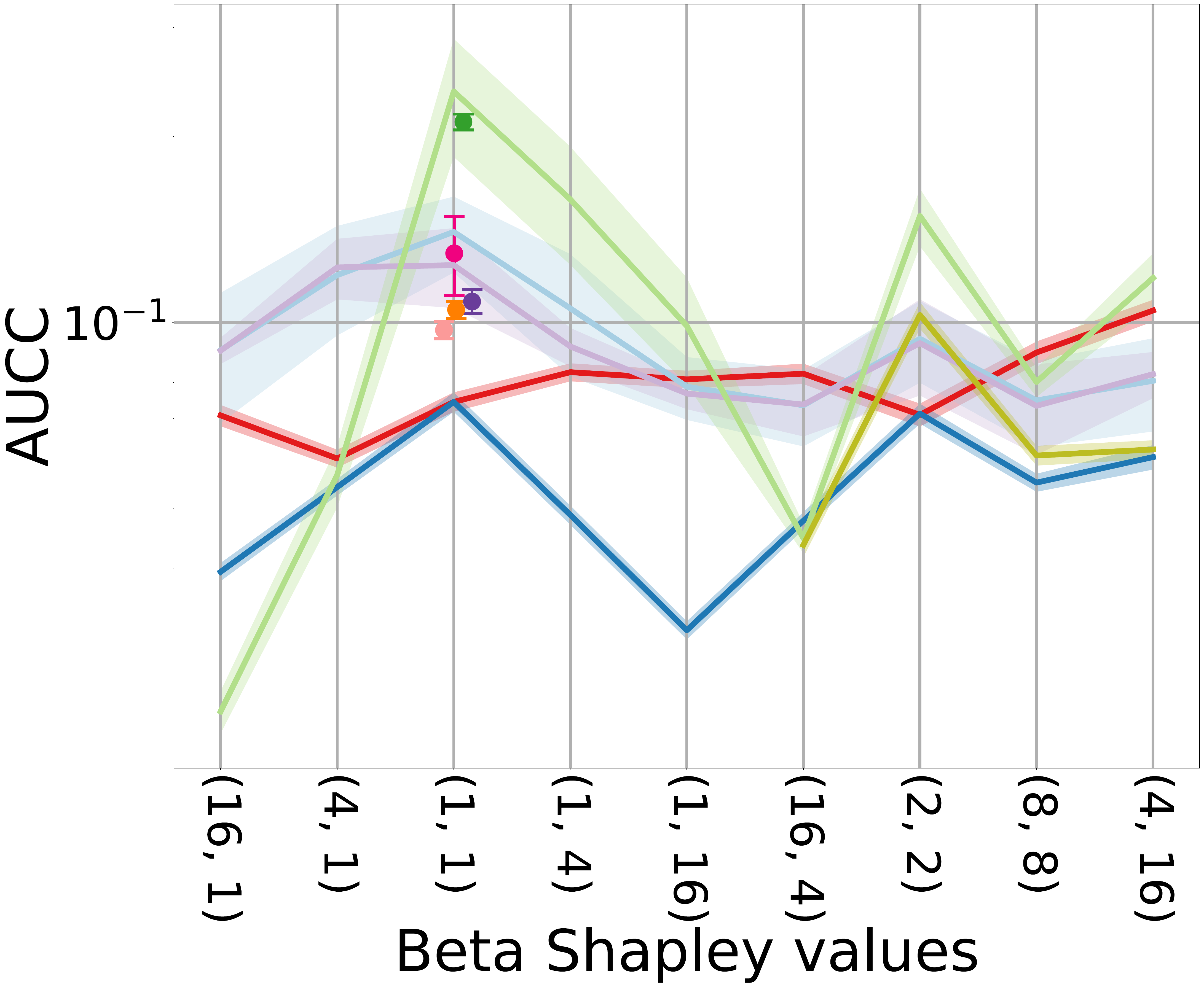} &
			\includegraphics[width=0.3\linewidth]{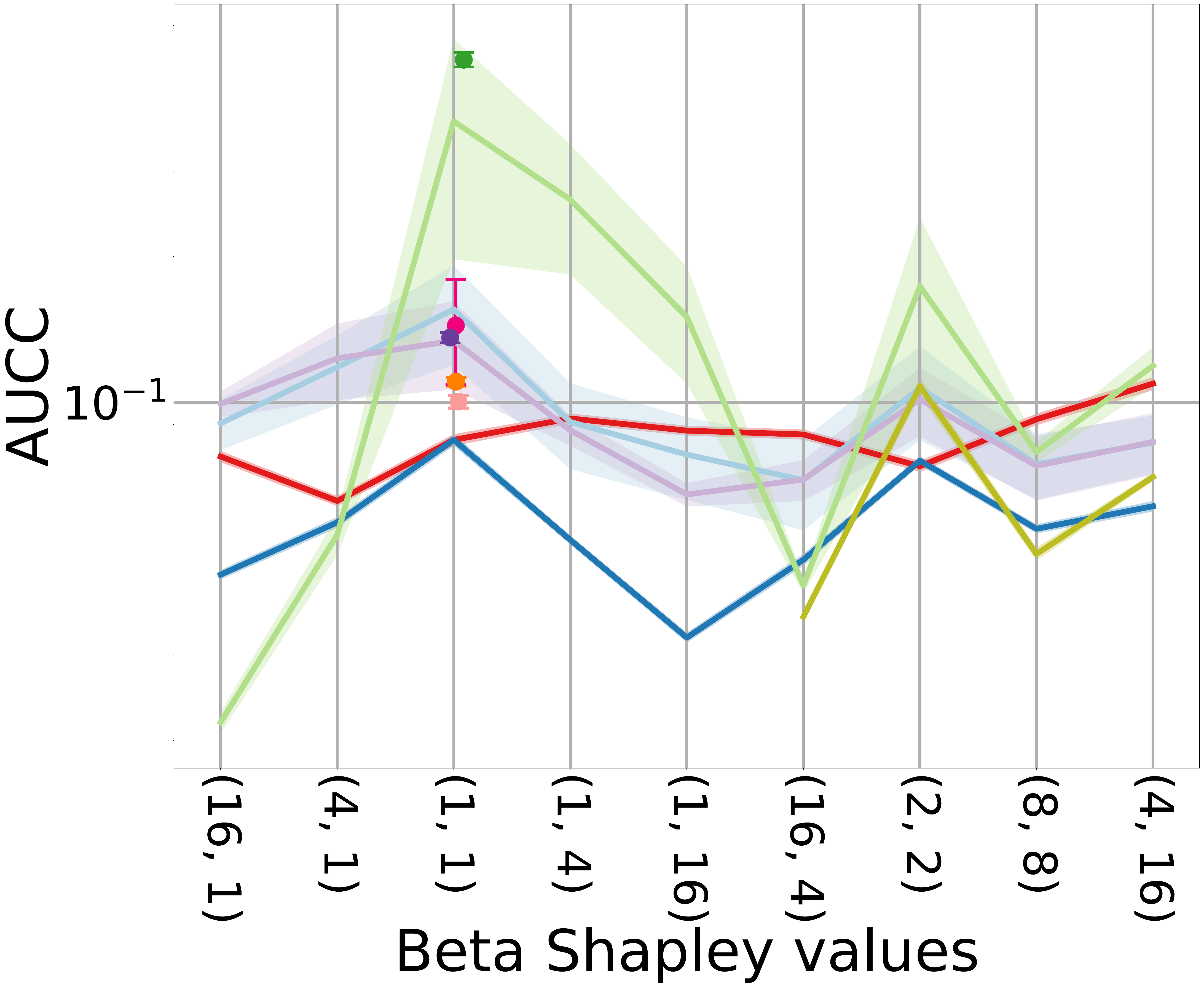} \\
			\includegraphics[width=0.3\linewidth]{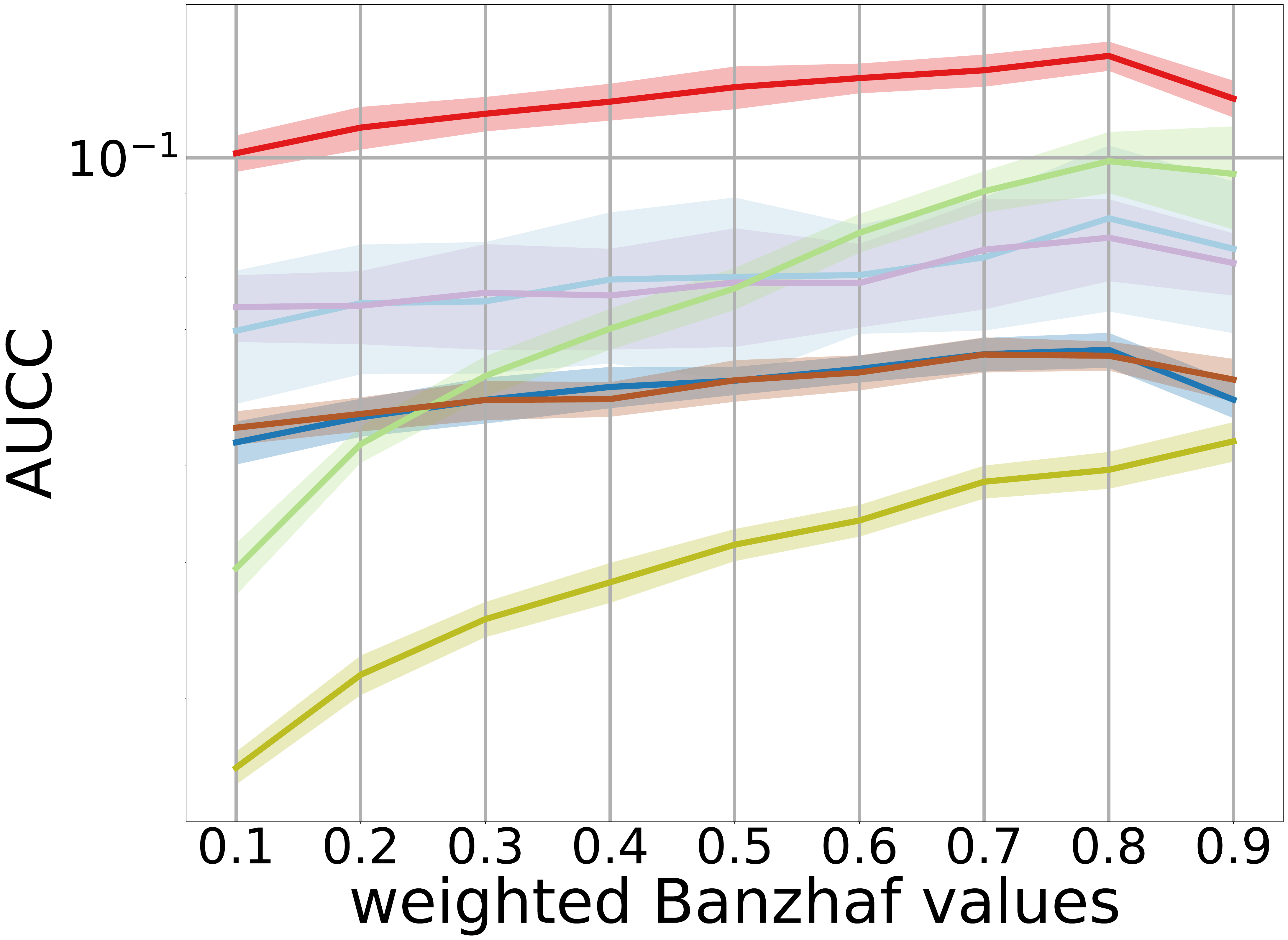} & \includegraphics[width=0.3\linewidth]{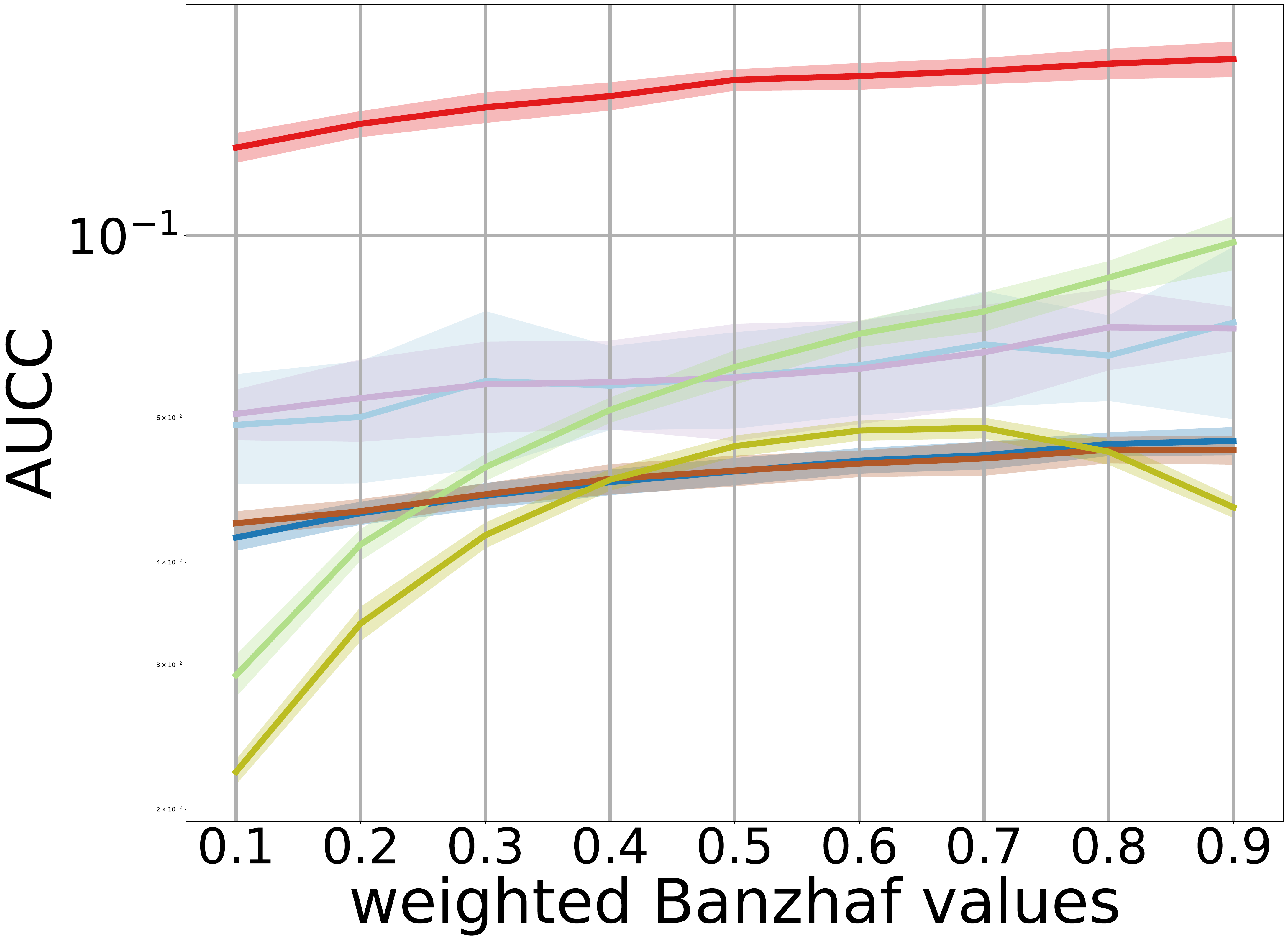} &
			\includegraphics[width=0.3\linewidth]{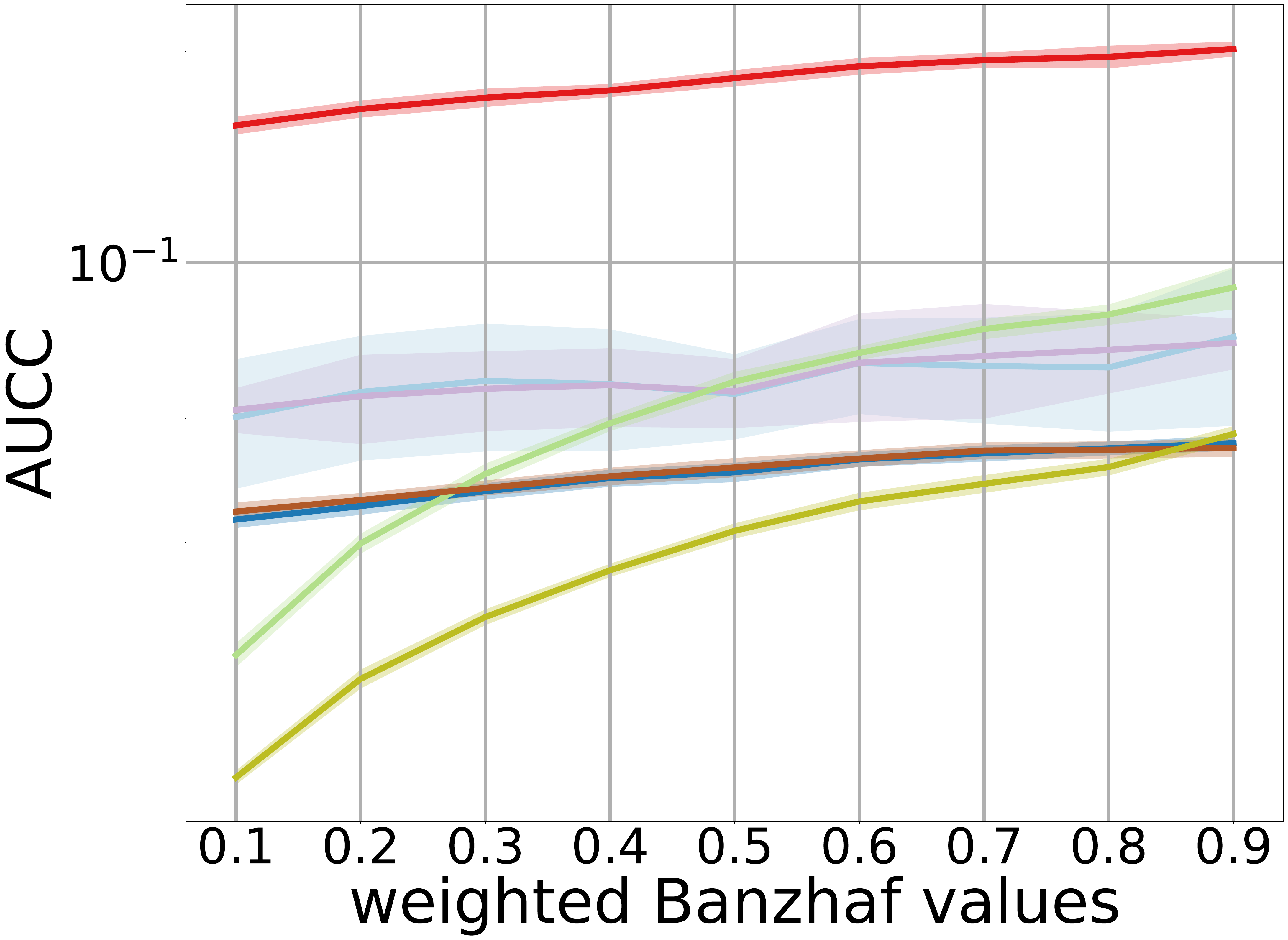}\\
			SOU ($ n=64 $) & SOU ($ n=128 $) & SOU ($ n=256 $)
		\end{tabular}
		\vspace{-.5em}
		\caption{Comparison of twelve estimators using six utility functions. All the AUCCs are reported with standard deviation using $ 30 $ random seeds. Smaller AUCC indicates faster convergence rate.}
		\vspace{-1em}
		\label{fig:generic}
	\end{figure}

	\paragraph{Verification of Our OFA-S Estimator}
	Next, we verify the faster convergence rate of our OFA-S estimator, using $ \mathbf{q}^{\text{OFA-S}} $ as defined in Eq.~\eqref{eq:optimal q}. The baselines we employ in this experiment include: kernelSHAP \citep{lundberg2017unified}, unbiased kernelSHAP \citep{covert2021improving}, GELS and GELS-Shapley \citep{li2024faster}, ARM \citep{kolpaczki2024approximating,li2024faster}, complement \citep{zhang2023efficient}, group testing \citep{jia2019towards,wang2023note}, AME \citep{lin2022measuring}, MSR \citep{wang2023data} and sampling lift \citep{moehle2021portfolio}. Note that not all the baselines are designed for all the probabilistic values we employ. For example, the complement estimator only works for Beta$ (1,1) $, \ie, the Shapley value. The corresponding results are presented in Figure~\ref{fig:generic}. 
	
	First, our OFA-S estimator is indeed faster than our OFA-A estimator, which aligns exactly with our theory; in other words, it implies that our proposed $ D(\mathbf{m}, \mathbf{p}) $ indeed determines the convergence rate of our Algorithm~\ref{alg:ofa}. Second, our OFA-S estimator always performs the best except on the SOU games which require only $ n^{2} $ utility evaluation to get the exact values; by contrast, the utility function defined using the classification datasets require $ 2^{n} $ utility evaluations instead. Third, our proposed estimator is consistently the fastest on the commonly-used Beta$ (1,1) $, i.e., the Shapley value; note that $ \mathbf{q}^{\text{OFA-A}} = \mathbf{q}^{\text{OFA-S}} $ for the Shapley value; therefore, our proposed estimator achieves the currently best convergence rate both empirically and theoretically.

	\section{Conclusion}
	In this work, we propose a framework, termed OFA, that i) adheres to the principle of maximum sample reuse and ii) contains no amplifying scalars for the goal of optimizing all probabilistic values simultaneously and efficiently. Particularly, our OFA framework is parameterized by a sampling vector $ \mathbf{q} \in \mathbb{R}^{n-3} $. To gain insights, we theoretically develop a key formula $ D(\mathbf{m}, \mathbf{q}) $ concerning this framework that effectively determines the corresponding convergence rate. By optimizing $ \mathbf{q} $ in $ D(\mathbf{m}, \mathbf{q}) $ for all probabilistic values on average, we obtain our one-for-all estimator that can theoretically approximate all probabilistic values simultaneously with the currently best convergence rate $ O(n\log n) $ on average. Meanwhile, we propose a faster generic estimator by optimizing $ \mathbf{q} $ for each specific probabilistic value, and we demonstrate that our generic estimate enjoys the best convergence rate for all previously-studied probabilistic values. 
	All of our theoretical findings are verified in our experiments.
	Finally, we establish a connection between probabilistic values and the least square regressions used in datamodels, showing that our OFA-A estimator is capable of solving a family of (regularized) datamodels simultaneously. 
	
	
	\clearpage
	
	\section*{Acknowledgements}
	We thank the reviewers and the area chair for thoughtful comments that have improved our final presentation. 
	YY gratefully acknowledges NSERC and CIFAR for funding support.

	\printbibliography[title={References}]
	\newpage
	\appendix

	\section{Proof of Theorem~\ref{thm:convergence}} \label{app:theorem 1}
	\convergenceFormula*
	\begin{proof}
		Following Algorithm~\ref{alg:ofa}, let $ \{ S_{t} \}_{t=1}^{T} $ be $ T $ independent random subsets. Define
		\begin{equation}
			\begin{gathered}
				T_{i,s}^{+} = \sum_{t=1}^{T} \idt{i\in S_{t}, |S_{t}|=s} \text{ and } T_{i,s}^{-} = \sum_{t=1}^{T} \idt{i \not\in S_{t}, |S_{t}|=s} 
			\end{gathered}
		\end{equation}
		where $ s = 2,3,\dots,n-2 $.
		Then, we have
		\begin{equation}
			\begin{gathered}
				\hat{\phi}_{i,s}^{+} = \frac{1}{T_{i,s}^{+}}\sum_{i=1}^{T} \idt{i\in S_{t}, |S_{t}|=s}\cdot U(S_{t}) \text{ and } \hat{\phi}_{i,s}^{-} = \frac{1}{T_{i,s}^{-}}\sum_{i=1}^{T} \idt{i \not\in S_{t}, |S_{t}|=s}\cdot U(S_{t}) .
			\end{gathered}
		\end{equation}
		Define $ r^{+}_{i,s} = \frac{T_{i,s}^{+}}{T} $ and $ r^{-}_{i,s} = \frac{T_{i,s}^{-}}{T} $. In particular, both $ \idt{i\in S_{t}, |S_{t}|=s} $ and $ \idt{i\not\in S_{t}, |S_{t}|=s} $ are Bernoulli random variables with
		\begin{equation}
			\begin{gathered}
				\mathbb{E}[r_{i,s}^{+}] = q_{s-1}\binom{n-1}{s-1}\binom{n}{s}^{-1} = \frac{q_{s-1}\cdot s}{n} \text{ and } \mathbb{E}[r_{i,s}^{-}] = q_{s-1} \binom{n-1}{s}\binom{n}{s}^{-1} = \frac{q_{s-1}\cdot (n-s)}{n} .
			\end{gathered}
		\end{equation}
		Additionally, $ \bm{R} $ and $ \boldsymbol\tau $ are defined to be vectors in $ \mathbb{R}^{2n-6} $ such that $ R_{2k-1} = r^{+}_{i,k+1} $, $ R_{2k} = r^{-}_{i,k+1} $, $ \tau_{2k-1} = \frac{q_{k}\cdot (k+1)}{n} $ and $ \tau_{2k} = \frac{q_{k}\cdot(n-k-1)}{n} $ for $ k \in [n-3] $. 
		Note that $ \bm{R} $ is a random vector.	
		By Hoeffding's inequality,
		\begin{equation}
			\begin{gathered}
				P(|R_{j} - \tau_{j}| \geq \omega) \leq 2\exp\left( -2T\omega^{2} \right)
			\end{gathered}
		\end{equation}
		where $ \omega > 0 $, and thus
		\begin{equation}
			\begin{gathered}
				P(\| \bm{R} - \boldsymbol\tau \|_{\infty} \geq \omega) \leq P(\bigcup_{j \in [2n-6]} |R_{j} - \tau_{j}| \geq \omega) \leq (4n-12)\exp\left( -2T\omega^{2} \right) .
			\end{gathered}
		\end{equation}
		Denote the event $ \{ \sum_{s=2}^{n-2}[ m_{s}(\hat{\phi}_{i,s}^{+}-\phi_{i,s}^{+}) + m_{s+1}(\phi_{i,s}^{-} - \hat{\phi}_{i,s}^{-})] \geq \epsilon \} $ by $ E_{i} $ where $ \epsilon > 0 $. 
		Let $ \mathcal{C} $ be the set that contains all possible configurations $ \mathbf{C} \in \{ 0,1 \}^{(2n-6)\times T} $ such that $ \frac{\mathbf{C}\mathbf{1}_{T}}{T} = \bm{R} $ and $ \mathbf{1}_{T}^{\top}\mathbf{C} = \mathbf{1}_{2n-6}^{\top} $, i.e., $ C_{j,k} = 0 $ indicates that the $ k $-th subset is sampled from $ \{ R \subseteq [n] \mid r = (j+3)/2 \text{ and } i \in R \} $ if $ j $ is odd and $ \{ R \subseteq [n] \mid r = (j+2)/2 \text{ and } i\not\in R \} $ otherwise. 
		Then,
		\begin{equation}
			\begin{aligned}
				P(E_{i}) = \sum_{\mathbf{C} \in \mathcal{C}} P(E_{i} \cap \mathbf{C}) = \sum_{\mathbf{C} \in \mathcal{C}} P(E_{i} \mid \mathbf{C})\cdot P(\mathbf{C}) .
			\end{aligned}
		\end{equation}
		Observe that $ \mathcal{C} $ can be divided into two separate groups $ \mathcal{C}_{<\omega} $ and $ \mathcal{C}_{\geq\omega} $ such that
		\begin{equation}
			\begin{gathered}
				\sum_{\mathbf{C}_{<\omega} \in \mathcal{C}_{<\omega}}P(\mathbf{C}_{<\omega}) = P(\| \bm{R} - \boldsymbol\tau \|_{\infty} < \omega) \text{ and } \sum_{\mathbf{C}_{\geq\omega}\in \mathcal{C}_{\geq\omega}}P(\mathbf{C}_{\geq\omega}) = P(\| \bm{R} - \boldsymbol\tau \|_{\infty} \geq \omega) .
			\end{gathered}
		\end{equation}
		Therefore,
		\begin{equation} \label{eq:bound Ei}
			\begin{aligned}
				P(E_{i}) &= \sum_{\mathbf{C}_{<\omega} \in \mathcal{C}_{<\omega}} P(E_{i} \mid \mathbf{C}_{<\omega})\cdot P(\mathbf{C}_{<\omega}) + \sum_{\mathbf{C}_{\geq\omega} \in \mathcal{C}_{\geq\omega}} P(E_{i} \mid \mathbf{C}_{\geq\omega})\cdot P(\mathbf{C}_{\geq\omega}) \\
				\leq& \sum_{\mathbf{C}_{<\omega} \in \mathcal{C}_{<\omega}} P(E_{i} \mid \mathbf{C}_{<\omega})\cdot P(\mathbf{C}_{<\omega}) + (4n-12)\exp\left( -2T\omega^{2} \right) .
			\end{aligned}
		\end{equation}
		
		For simplicity, we write $ P_{\mathbf{C}_{<\omega}}(E_{i}) $ instead of $ P(E_{i} \mid \mathbf{C}_{<\omega}) $. Additionally, we assume $ \omega < \frac{\gamma(\mathbf{q})}{2} $ so that neither $ T_{i,s}^{+} $ nor $ T_{i,s}^{-} $ is zero when conditioned on any $ \mathbf{C}_{<\omega} $.
		By the Chernoff bound, for any $ \lambda > 0 $, there is
		\begin{equation}
			\begin{aligned}
				P_{\mathbf{C}_{<\omega}}(E_{i}) &\leq \mathbb{E}_{\mathbf{C}_{<\omega}}\left[ \exp\left( \lambda\sum_{s=2}^{n-2}\left( m_{s}(\hat{\phi}_{i,s}^{+}-\phi_{i,s}^{+}) + m_{s+1}(\phi_{i,s}^{-} - \hat{\phi}_{i,s}^{-}) \right) \right) \right] \cdot e^{-\lambda \epsilon}\\
				&= e^{-\lambda \epsilon} \prod_{s=2}^{n-2}\mathbb{E}_{\mathbf{C}_{<\omega}}\left[\exp\left( \lambda m_{s}(\hat{\phi}_{i,s}^{+} - \phi_{i,s}^{+}) \right)\right] \prod_{s=2}^{n-2}\mathbb{E}_{\mathbf{C}_{<\omega}}\left[ \exp\left( \lambda m_{s+1}(\phi_{i,s}^{-} - \hat{\phi}_{i,s}^{-}) \right) \right]
			\end{aligned}
		\end{equation}
		where the equality is due to the independence that stems from the independence of random subsets and that the configuration is fixed.
		Moreover,
		\begin{equation}
			\begin{aligned}
				\mathbb{E}_{\mathbf{C}_{<\omega}}\left[\exp\left( \lambda m_{s}(\hat{\phi}_{i,s}^{+} - \phi_{i,s}^{+}) \right)\right] &= \mathbb{E}_{\mathbf{C}_{<\omega}}\left[ \exp\left( \lambda m_{s} \frac{1}{T_{i,s}^{+}}\sum_{j=1}^{T_{i,s}^{+}}( U(S_{i,s,j}^{+}) - \phi_{i,s}^{+}) \right) \right]\\
				&= \prod_{j=1}^{T_{i,s}^{+}} \mathbb{E}_{\mathbf{C}_{<\omega}}\left[ \exp\left(  \frac{\lambda m_{s}}{T_{i,s}^{+}}( U(S_{i,s,j}^{+}) - \phi_{i,s}^{+}) \right) \right]
			\end{aligned}
		\end{equation}
		where  $ \{ S_{i,s,j}^{+} \}_{1\leq j\leq T_{i,s}^{+}} $ is obtained by ordering $ \{ S_{t} \mid |S_{t}| = s \text{ and } i \in S_{t} \} $.
		In a similar fashion, we have
		\begin{equation}
			\begin{gathered}
				\mathbb{E}_{\mathbf{C}_{<\omega}}\left[ \exp\left( \lambda m_{s+1}(\phi_{i,s}^{-} - \hat{\phi}_{i,s}^{-}) \right) \right] = \prod_{j=1}^{T_{i,s}^{-}}\mathbb{E}_{\mathbf{C}_{<\omega}}\left[ \exp\left( \frac{\lambda m_{s+1}}{T_{i,s}^{-}}(\phi_{i,s}^{-} - U(S_{i,s,j}^{-})) \right) \right]
			\end{gathered}
		\end{equation}
		By Hoeffding's lemma, 
		\begin{equation}
			\begin{gathered}
				\mathbb{E}_{\mathbf{C}_{<\omega}}\left[ \exp\left(  \frac{\lambda m_{s}}{T_{i,s}^{+}}(U(S_{i,s,j}^{+}) - \phi_{i,s}^{+}) \right) \right] \leq \exp\left( \frac{\lambda^{2}m_{s}^{2}u^{2}}{2T_{i,s}^{+}\cdot T_{i,s}^{+}} \right) ,
				\\
				\mathbb{E}_{\mathbf{C}_{<\omega}}\left[ \exp\left(  \frac{\lambda m_{s+1}}{T_{i,s}^{-}}(\phi_{i,s}^{-} - U(S_{i,s,j}^{-})) \right) \right] \leq \exp\left( \frac{\lambda^{2}m_{s+1}^{2}u^{2}}{2T_{i,s}^{-}\cdot T_{i,s}^{-}} \right) ,
			\end{gathered}
		\end{equation}
		which leads to 
		\begin{equation}
			\begin{gathered}
				\prod_{j=1}^{T_{i,s}^{+}} \mathbb{E}_{\mathbf{C}_{<\omega}}\left[ \exp\left(  \frac{\lambda m_{s}}{T_{i,s}^{+}}( U(S_{i,s,j}^{+}) - \phi_{i,s}^{+}) \right) \right] \leq \exp\left( \frac{\lambda^{2}m_{s}^{2} u^{2}}{2T_{i,s}^{+}} \right) ,\\
				\prod_{j=1}^{T_{i,s}^{-}}\mathbb{E}_{\mathbf{C}_{<\omega}}\left[ \exp\left( \frac{\lambda m_{s+1}}{T_{i,s}^{-}}(\phi_{i,s}^{-} - U(S_{i,s,j}^{-})) \right) \right] \leq \exp\left( \frac{\lambda^{2}m_{s+1}^{2} u^{2}}{2T_{i,s}^{-}} \right) .
			\end{gathered}
		\end{equation}
		Therefore, 
		\begin{equation}
			\begin{gathered}
				P_{\mathbf{C}_{<\omega}}(E_{i})
				\leq \exp\left( \frac{\lambda^{2}u^{2}}{2T}\hat{D} - \lambda \epsilon \right)
			\end{gathered}
		\end{equation}
		where $ \hat{D} = \sum_{s=2}^{n-2}\left( \frac{T}{T_{i,s}^{+}}m_{s}^{2} + \frac{T}{T_{i,s}^{-}}m_{s+1}^{2} \right) $. 
		Next, we aim to show that $ |\hat{D} - D(\mathbf{m}, \mathbf{q})| \leq D(\mathbf{m}, \mathbf{q}) $. Observe that
		\begin{equation}
			\begin{gathered}
				|\hat{D} - D(\mathbf{m}, \mathbf{q})| \leq \sum_{s=2}^{n-2}\left(\left|\frac{1}{r_{2s-3}} - \frac{1}{\tau_{2s-3}}\right|m_{s}^{2} - \left|\frac{1}{r_{2s-2}} - \frac{1}{\tau_{2s-2}}\right|m_{s+1}^{2} \right) ,
			\end{gathered}
		\end{equation}
		and since $ |r_{j} - \tau_{j}| < \omega $,
		\begin{equation}
			\begin{gathered}
				\left|\frac{1}{r_{j}} - \frac{1}{\tau_{j}}\right| \leq \frac{\omega}{(\tau_{j} - \omega)\tau_{j}} = \frac{1}{\tau_{j} - \omega} - \frac{1}{\tau_{j}} .
			\end{gathered}
		\end{equation}
		Since $ \gamma(\mathbf{q}) \leq \tau_{j} $ and $ \omega \leq \frac{\gamma(\mathbf{q})}{2} $,
		\begin{equation}
			\begin{gathered}
				\frac{1}{\tau_{j} - \omega} = \frac{\tau_{j}}{\tau_{j} - \omega}\cdot \frac{1}{\tau_{j}} = \frac{1}{1 - \frac{\omega}{\tau_{j}}}\cdot \frac{1}{\tau_{j}} \leq \frac{2}{\tau_{j}} .
			\end{gathered}
		\end{equation}
		As a result, we have $ |\hat{D} - D(\mathbf{m}, \mathbf{q})| \leq D(\mathbf{m}, \mathbf{q}) $, and thus
		\begin{equation} \label{eq:bound conditioned}
			\begin{gathered}
				P_{\mathbf{C}_{<\omega}}(E_{i}) \leq \exp\left( \frac{\lambda^{2}u^{2}}{T}D(\mathbf{m},\mathbf{q}) - \lambda\epsilon \right) .
			\end{gathered}
		\end{equation}
		Combining Eqs.~\eqref{eq:bound Ei} and~\eqref{eq:bound conditioned} yields
		\begin{equation}
			\begin{gathered}
				P(E_{i}) \leq \exp\left( \frac{\lambda^{2}u^{2}}{T}D(\mathbf{m}, \mathbf{q}) - \lambda \epsilon \right) + (4n-12)\exp(-2T\omega^{2}) .
			\end{gathered}
		\end{equation}
		
		Choosing $ \lambda > 0 $ that minimizes the upper bound yields
		\begin{equation}
			\begin{gathered}
				P(E_{i}) \leq \exp\left( -\frac{T\epsilon^{2}}{4u^{2}D(\mathbf{m}, \mathbf{q})} \right) + (4n-12)\exp(-2T\omega^{2}) .
			\end{gathered}
		\end{equation}
		Solving the equation $ -\frac{T\epsilon^{2}}{4u^{2}D(\mathbf{m}, \mathbf{q})} = -2T\omega^{2} $  yields $ \omega = \sqrt{\frac{\epsilon^{2}}{8D(\mathbf{m}, \mathbf{q})u^{2}}} $, which gives
		\begin{equation}
			\begin{gathered}
				-2T\omega^{2} = -\frac{T\epsilon^{2}}{4D(\mathbf{m}, \mathbf{q})u^{2}} .
			\end{gathered}
		\end{equation}
		Particularly, to meet the assumption $ \omega \leq \frac{\gamma(\mathbf{q})}{2} $, we have to have $ \epsilon \leq \sqrt{2D(\mathbf{m}, \mathbf{q})\gamma(\mathbf{q})^{2}u^{2}} $. To conclude, provided that $ \epsilon \leq \sqrt{2D(\mathbf{m}, \mathbf{q})\gamma(\mathbf{q})^{2}u^{2}} $, we have
		\begin{equation}
			\begin{gathered}
				P(\sum_{s=2}^{n-2} \left( m_{s}(\hat{\phi}_{i,s}^{+}-\phi_{i,s}^{+}) + m_{s+1}(\phi_{i,s}^{-} - \hat{\phi}_{i,s}^{-}) \right) \geq \epsilon) \leq 4n\exp(-\frac{T\epsilon^{2}}{4D(\mathbf{m}, \mathbf{q})u^{2}}) .
			\end{gathered}
		\end{equation}
		Similarly, there is
		\begin{equation}
			\begin{gathered}
				P(\sum_{s=2}^{n-2} \left( m_{s}(\phi_{i,s}^{+}-\hat{\phi}_{i,s}^{+}) + m_{s+1}(\hat{\phi}_{i,s}^{-} - \phi_{i,s}^{-}) \right) \geq \epsilon) \leq 4n\exp(-\frac{T\epsilon^{2}}{4D(\mathbf{m}, \mathbf{q})u^{2}}) ,
			\end{gathered}
		\end{equation}
		and thus
		\begin{equation}
			\begin{gathered}
				P(|\hat{\phi}_{i} - \phi_{i}| \geq \epsilon) \leq 8n\exp(-\frac{T\epsilon^{2}}{4D(\mathbf{m}, \mathbf{q})u^{2}}) .
			\end{gathered}
		\end{equation}

		Eventually, we have
		\begin{equation}
			\begin{gathered}
				P(\| \hat{\boldsymbol\phi} - \boldsymbol\phi \|_{2} \geq \epsilon) \leq P(\bigcup_{i\in [n]} |\hat{\phi}_{i} - \phi_{i}|\geq \frac{\epsilon}{\sqrt{n}}) \leq 8n^{2}\exp(-\frac{T\epsilon^{2}}{4nD(\mathbf{m}, \mathbf{q})u^{2}}) .
			\end{gathered}
		\end{equation}
		Solving $ \delta \geq 8n^{2}\exp(-\frac{T\epsilon^{2}}{4nD(\mathbf{m}, \mathbf{q})u^{2}}) $ yields $ T \geq \frac{4nD(\mathbf{m}, \mathbf{q})u^{2}}{\epsilon^{2}}\log\frac{8n^{2}}{\delta} $. Note the assumption $ \epsilon \leq \sqrt{2D(\mathbf{m}, \mathbf{q})\gamma(\mathbf{q})^{2}u^{2}} $ can be removed if the configuration is fixed with $ T^{+}_{i,s} \approx \frac{s\cdot q_{s-1}}{n} T $ and $ T_{i,s}^{-} \approx \frac{(n-s)q_{s-1}}{n}T $.
	\end{proof}

	\section{Proofs of Propositions} \label{app:propositions}
	\ofaEstimator*
	\begin{proof}
		Let $ \Lambda = \{ \mathbf{x} \in \mathbb{R}^{n-1} \mid 0 \leq \sum_{j=1}^{n-1} x_{j} \leq L_{n} \} $ where $ L_{n} = n^{\frac{1}{2(n-1)}} $, and
		a smooth homeomorphism $ f : \Lambda \to \Delta $ is defined by letting
		\begin{equation}
			\begin{gathered}
				f(\mathbf{x}) = \frac{1}{L_{n}}(x_{1}, x_{2}, \cdots, x_{n-1}, L_{n}-\sum_{j=1}^{n-1}x_{j})^{\top} .
			\end{gathered}
		\end{equation}
		In other words, both $ f $ and $ f^{-1} $ are $ C^{\infty} $.
		Since the volume of $ \Delta $ is $ \frac{n^{\frac{1}{2}}}{(n-1)!} $, there is
		\begin{gather*}
			\frac{(n-1)!}{n^{\frac{1}{2}}}\int_{\mathbf{x} \in \Lambda} D(f(\mathbf{x}), \mathbf{q})  \sqrt{\det\left( Df(\mathbf{x})^{\top} Df(\mathbf{x}) \right)} \mathrm{d}\mathbf{x}
			= \int_{\mathbf{m} \in \Delta} D(\mathbf{m}, \mathbf{q}) \mathrm{d}\nu(\mathbf{m}) .
		\end{gather*}
		Note that $ \sqrt{\det\left( Df(\mathbf{x})^{\top} Df(\mathbf{x}) \right)} = 1 $ for every $ \mathbf{x} \in \Lambda $.
		With $ \overline{\Lambda} = \{ \mathbf{y} \in \mathbb{R}^{n-1} \mid 0 \leq \sum_{j=1}^{n-1}y_{j} \leq 1 \} $, we have
		\begin{equation}
			\begin{gathered}
				\frac{(n-1)!}{ n^{\frac{1}{2}}}\int_{\mathbf{x} \in \Lambda} D(f(\mathbf{x}), \mathbf{q}) \mathrm{d}\mathbf{x} = (n-1)! \int_{\mathbf{y} \in \overline{\Lambda}} D(f(L_{n}\mathbf{y}), \mathbf{q}) \mathrm{d}\mathbf{y} .
			\end{gathered}
		\end{equation}
		For simplicity, assume that $ n=4 $, notice that
		\begin{equation}
			\begin{gathered}
				\int_{y\in\overline{\Lambda}} y_{n-1}^{2} \mathrm{d}\mathbf{y} = \int_{0}^{1}\mathrm{d}y_{1}\int_{0}^{1-y_{1}}\mathrm{d}y_{2}\int_{0}^{1-y_{1}-y_{2}} y_{3}^{2} \mathrm{d}y_{3} = \frac{1}{3\cdot 4\cdot 5} = \frac{1}{\prod_{k=1}^{n-1}(2+k)} .
			\end{gathered}
		\end{equation}
		Therefore,
		\begin{gather*}
			\int_{\mathbf{y} \in \overline{\Lambda}} D(f(L_{n}\mathbf{y}), \mathbf{q}) \mathrm{d}\mathbf{y}
			= \sum_{s=2}^{n-2}\frac{n}{q_{s-1}}\int_{y\in\overline{\Lambda}} \left( \frac{y_{s}^{2}}{s} + \frac{y_{s+1}^{2}}{n-s} \right) \mathrm{d}\mathbf{y}\\
			= \frac{1}{\prod_{k=1}^{n-1}(2+k)} \sum_{s=2}^{n-2} \frac{n}{q_{s-1}}\left( \frac{1}{s} + \frac{1}{n-s} \right) ,
		\end{gather*}
		which leads to
		\begin{equation}
			\begin{gathered}
				\overline{D}(\mathbf{q}) = \frac{(n-1)!}{\prod_{k=1}^{n-1}(2+k)} \sum_{s=2}^{n-2} \frac{n}{q_{s-1}}\left( \frac{1}{s} + \frac{1}{n-s} \right) .
			\end{gathered}
		\end{equation}
		Since $ \overline{D}(\mathbf{q}) $ is convex in $ \mathbf{q} $, $ \mathbf{q}^{\text{OFA-A}} $ can be directly obtained using the KKT conditions, which is
		\begin{equation}
			\begin{gathered}
				q_{s-1}^{\text{OFA-A}} = \frac{\sqrt{\frac{n}{s} + \frac{n}{n-s}}}{\sum_{s=2}^{n-2}\sqrt{\frac{n}{s}+\frac{n}{n-s}}}  .
			\end{gathered}
		\end{equation}
		Therefore, we have
		\begin{equation}
			\begin{gathered}
				\overline{D}(\mathbf{q}^{\text{OFA-A}}) = \frac{(n-1)!}{\prod_{k=1}^{n-1}(2+k)} \left( \sum_{s=2}^{n-2} \sqrt{\frac{n}{s} + \frac{n}{n-s}} \right)^{2} .
			\end{gathered}
		\end{equation}
		Since $  \lim_{n\to\infty}\frac{(n-1)! (n-1)^{2}}{\prod_{k=1}^{n-1}(2+k)} = 2\Gamma(3) $, when $ n $ is sufficiently large, there is
		\begin{equation}
			\begin{gathered}
				\overline{D}(\mathbf{q}^{\text{OFA-A}}) \approx \frac{1}{n^{2}} \left( \sum_{s=2}^{n-2} \sqrt{\frac{n}{s} + \frac{n}{n-s}} \right)^{2} = \left( \frac{1}{n}\sum_{s=2}^{n-2} \sqrt{\frac{1}{\frac{s}{n}(1-\frac{s}{n})}} \right)^{2} < \left( \int_{0}^{1}\frac{1}{x(1-x)} \mathrm{d}x \right)^{2} = \pi^{2}.
			\end{gathered}
		\end{equation}
	\end{proof}
	
	\ofaConvergence*
	\begin{proof}
		Let $ \boldsymbol\phi $ be a semi-value such that $ p_{s} = \int_{0}^{1}w^{s-1}(1-w)^{n-s}\mathrm{d}\mu(w) = \int_{0}^{1}w^{s-1}(1-w)^{n-s}p_{\mu}(w)\mathrm{d}w $ such that $ p_{\mu}(w) \leq B $ for every $ w \in [0, 1] $. Particularly, we have
		\begin{equation}
			\begin{gathered}
				m_{s} = \binom{n-1}{s-1}p_{s} \leq B\cdot\binom{n-1}{s-1}\int_{0}^{1}w^{s-1}(1-w)^{n-s}\mathrm{d}w = B\cdot\binom{n-1}{s-1}\frac{(s-1)!(n-s)!}{n!} = \frac{B}{n} . 
			\end{gathered}
		\end{equation}
		Therefore,
		\begin{equation}
			\begin{gathered}
				D(\mathbf{m}, \mathbf{q}^{\text{OFA-A}}) \leq \frac{B^{2}}{n} \sum_{s=2}^{n-2}\frac{1}{q^{\text{OFA-A}}_{s-1}}\left( \frac{1}{s} + \frac{1}{n-s} \right) = B^{2} \left( \sum_{s=2}^{n-2} \frac{1}{\sqrt{s(n-s)}} \right)^{2} < B^{2}\pi^{2} .
			\end{gathered}
		\end{equation}
	\end{proof}
	
	
	\ofaWeighted*
	\begin{proof}	
		With $ q^{\text{OFA-A}}_{s-1} \propto \frac{1}{\sqrt{s(n-s)}} $, we have
		\begin{equation}
			\begin{gathered}
				D(\mathbf{m}, \mathbf{q}^{\text{OFA-A}}) = C\cdot n\cdot \sum_{s=2}^{n-2}\left( \sqrt{\frac{n-s}{s}}m_{s}^{2} + \sqrt{\frac{s}{n-s}}m_{s+1}^{2} \right)
				\text{ where } C = \sum_{s=2}^{n-2} \frac{1}{\sqrt{s(n-s)}} < \pi
			\end{gathered}
		\end{equation}
		Then,
		\begin{equation}
			\begin{gathered}
				D(\mathbf{m}, \mathbf{q}^{\text{OFA-A}}) \\
				= C\cdot \sum_{s=2}^{n-2} n\cdot \left( \sqrt{\frac{n-s}{s}}\binom{n-1}{s-1}^{2}\left( w^{s-1}(1-w)^{n-s} \right)^{2} + \sqrt{\frac{s}{n-s}}\binom{n-1}{s}^{2}\left( w^{s}(1-w)^{n-s-1} \right)^{2} \right) .
			\end{gathered}
		\end{equation}
		Specifically,
		\begin{equation}
			\begin{gathered}
				\sqrt{\frac{n-s}{s}}\binom{n-1}{s-1}^{2} = \sqrt{\frac{s}{n-s}}\frac{(n-1)!^{2}}{(s-1)!s!(n-s-1)!(n-s)!} \\
				\text{and } \sqrt{\frac{s}{n-s}}\binom{n-1}{s}^{2} = \sqrt{\frac{n-s}{s}}\frac{(n-1)!^{2}}{(s-1)!s!(n-s-1)!(n-s)!} ,
			\end{gathered}
		\end{equation}
		and thus
		\begin{equation}
			\begin{gathered}
				n\cdot \left( \sqrt{\frac{n-s}{s}}\binom{n-1}{s-1}^{2}\left( w^{s-1}(1-w)^{n-s} \right)^{2} + \sqrt{\frac{s}{n-s}}\binom{n-1}{s}^{2}\left( w^{s}(1-w)^{n-s-1} \right)^{2} \right)\\
				= n\cdot \left( w^{s-1}(1-w)^{n-s-1} \right)^{2}\frac{(n-1)!^{2}}{(s-1)!s!(n-s-1)!(n-s)!}\left( \sqrt{\frac{s}{n-s}}(1-w)^{2} + \sqrt{\frac{n-s}{s}}w^{2} \right)
			\end{gathered}
		\end{equation}
		Since 
		\begin{equation}
			\begin{gathered}
				\sqrt{\frac{s}{n-s}}(1-w)^{2} + \sqrt{\frac{n-s}{s}}w^{2} \leq \sqrt{\frac{s}{n-s}} + \sqrt{\frac{n-s}{s}} = \frac{n}{\sqrt{s(n-s)}} ,
			\end{gathered}
		\end{equation}
		there is
		\begin{equation}
			\begin{gathered}
				n\cdot \left( \sqrt{\frac{n-s}{s}}\binom{n-1}{s-1}^{2}\left( w^{s-1}(1-w)^{n-s} \right)^{2} + \sqrt{\frac{s}{n-s}}\binom{n-1}{s}^{2}\left( w^{s}(1-w)^{n-s-1} \right)^{2} \right)\\
				\leq \sqrt{s(n-s)}\frac{\left( \binom{n}{s} w^{s} (1-w)^{n-s} \right)^{2}}{w^{2}(1-w)^{2}} \leq n\cdot \frac{\left( \binom{n}{s} w^{s} (1-w)^{n-s} \right)^{2}}{w^{2}(1-w)^{2}}.
			\end{gathered}
		\end{equation}
		
		Using the identity $ \sum_{j=0}^{m} \binom{m}{j}^{2}(x+y)^{2j}(x-y)^{2(m-j)} = \sum_{j=0}^{m}\binom{2j}{j}\binom{2(m-j)}{m-j}x^{2j}y^{2(m-j)} $, there is
		\begin{equation}
			\begin{gathered}
				\sum_{s=2}^{n-2}\left( \binom{n}{s} w^{s} (1-w)^{n-s} \right)^{2} = \sum_{s=0}^{n} \binom{2s}{s}\binom{2(n-s)}{n-s}\frac{1}{2^{2s}}\left( \frac{2w-1}{2} \right)^{2(n-s)}\\
				= \binom{2n}{n}\left( \frac{2w-1}{2} \right)^{2n} + \sum_{s=1}^{n-1} \binom{2s}{s}\binom{2(n-s)}{n-s}\frac{1}{2^{2s}}\left( \frac{2w-1}{2} \right)^{2(n-s)} + \binom{2n}{n}\frac{1}{2^{2n}}.
			\end{gathered}
		\end{equation}
		For every $ k\geq 1 $, $ \binom{2k}{k} \approx \frac{2^{2k}}{\sqrt{k}} $ using the Stirling's approximation, and thus
		\begin{equation}
			\begin{gathered}
				\binom{2n}{n}\left( \frac{2w-1}{2} \right)^{2n} \approx \frac{z^{n}}{\sqrt{n}}, \quad \binom{2n}{n}\frac{1}{2^{2n}} \approx \frac{1}{\sqrt{n}}\\
				\sum_{s=1}^{n-1} \binom{2s}{s}\binom{2(n-s)}{n-s}\frac{1}{2^{2s}}\left( \frac{2w-1}{2} \right)^{2(n-s)} \approx \sum_{s=1}^{n-1}\frac{1}{\sqrt{s(n-s)}}z^{n-s} \leq \frac{\sum_{j=1}^{n-1}z^{j}}{\sqrt{n-1}},
			\end{gathered}
		\end{equation}
		where $ z = (2w-1)^{2} < 1 $. Therefore, we obtain $ \sum_{s=2}^{n-2}\left( \binom{n}{s} w^{s} (1-w)^{n-s} \right)^{2} \leq O(n^{-\frac{1}{2}}) $, which eventually leads to
		\begin{equation}
			\begin{gathered}
				D(\mathbf{m}, \mathbf{q}^{\text{OFA-A}}) \leq  \frac{n}{w^{2}(1-w)^{2}} \sum_{s=2}^{n-2}\left( \binom{n}{s} w^{s} (1-w)^{n-s} \right)^{2} \leq O(n^{\frac{1}{2}}) .
			\end{gathered}
		\end{equation}

		%
		%
	\end{proof}
	
	\genericConvergence*
	\begin{proof}
		If $ \mu(\{ 0 \}) \not= 0 $ ($ \mu(\{ 1 \}) \not= 0 $, respectively), its induced marginal contributions all reside in $ \phi_{i,1}^{+} $ and $ \phi_{i,0}^{-} $ ($ \phi_{i,n}^{+} $ and $ \phi_{i,n-1}^{-} $, respectively), which is computed exactly using Algorithm~\ref{alg:ofa}. Therefore, W.L.O.G., we assume that $ \mu((0,1)) = 1 $.
		
		Suffice it to show that if $ \int_{0}^{1}\frac{1}{w(1-w)} \mathrm{d}\mu(w) < \infty $, there is
		\begin{equation}
			\begin{gathered}
				\sum_{s=2}^{n-2}\sqrt{\frac{n}{s}m_{s}^{2} + \frac{n}{n-s}m_{s+1}^{2}} \in O(1) .
			\end{gathered}
		\end{equation}
		Specifically,
		\begin{gather*}
			\mathbf{D}(\mathbf{m}, \mathbf{q}^{\text{OFA-S}}) = \sum_{s=2}^{n-2}\sqrt{\frac{n}{s}m_{s}^{2} + \frac{n}{n-s}m_{s+1}^{2}} \leq \sum_{s=2}^{n-2}\left( \sqrt{\frac{n}{s}}m_{s} + \sqrt{\frac{n}{n-s}}m_{s+1} \right) \\
			= \int_{0}^{1} \sum_{s=2}^{n-2} \left( \sqrt{\frac{s}{n}}\binom{n}{s}w^{s-1}(1-w)^{n-s} + \sqrt{\frac{n-s}{n}}\binom{n}{s}w^{s}(1-w)^{n-s-1} \right) \mathrm{d}\mu(w) .
		\end{gather*}
		Since $ \sqrt{\frac{s}{n}}(1-w) + \sqrt{\frac{n-s}{n}}w \leq 2 $, we have
		\begin{equation}
			\begin{gathered}
				\int_{0}^{1} \sum_{s=2}^{n-2} \left( \sqrt{\frac{s}{n}}\binom{n}{s}w^{s-1}(1-w)^{n-s} + \sqrt{\frac{n-s}{n}}\binom{n}{s}w^{s}(1-w)^{n-s-1} \right) \mathrm{d}\mu(w) \\
				\leq \int_{0}^{1} \sum_{s=2}^{n-2} \frac{2\binom{n}{s}w^{s}(1-w)^{n-s}}{w(1-w)} \mathrm{d}\mu(w) \leq 2 \int_{0}^{1}\frac{1}{w(1-w)} \mathrm{d}\mu(w) \in O(1).
			\end{gathered}
		\end{equation}
	\end{proof}

	\section{Proof of Theorem~\ref{thm:connection}} \label{app:theorem 2}
	To prove this theorem, we first state useful definitions and lemmas.
	\begin{definition}[Semi Inner Product]
		Let $ \mathcal{V} $ is a real linear space. A semi inner product $ \langle \cdot, \cdot \rangle $ on $ \mathcal{V} $ satisfies, for every $ x,y,z \in \mathcal{V} $ and every $ \alpha \in \mathbb{R} $, i) $ \langle x,y \rangle = \langle y, x \rangle $, ii) $ \langle \alpha x, y \rangle = \alpha\langle x, y \rangle $, iii) $ \langle x+y, z \rangle = \langle x, z \rangle + \langle y, z \rangle $, and iv) $ \langle x, x \rangle \geq 0 $. In addition, we write $ \| x \| = \sqrt{\langle x,x \rangle} $ for every $ x \in \mathcal{V} $.
	\end{definition}
	
	\begin{lemma} \label{lem:optimal criterion}
		Let a semi inner product on a linear space $ \mathcal{V} $ be given, and $ \mathcal{A} \subseteq \mathcal{V} $ is some affine space. For the following optimization problem
		\begin{equation}
			\begin{gathered}
				\argmin_{x \in \mathcal{A}} \| x - p \|^{2}
			\end{gathered}
		\end{equation}
		where $ p \in \mathcal{V} $, $ x^{*} $ is optimal if and only if 
		\begin{equation}  
			\begin{gathered}
				\langle x^{*} - p, y - x^{*} \rangle = 0,\ \forall y \in \mathcal{A} . \label{eq:optimal criterion for projection}
			\end{gathered}
		\end{equation}
	\end{lemma}
	\begin{proof}
		Suppose $ x^{*} $ verifies Eq.~\eqref{eq:optimal criterion for projection}, for every $ y \in \mathcal{A} $,
		\begin{equation}
			\begin{gathered}
				\| y - p \|^{2} = \| x^{*} - p \|^{2} + \| y - x^{*} \|^{2} + 2\langle x^{*}-p, y-x^{*} \rangle \geq \| x^{*} - p \|^{2} .
			\end{gathered}
		\end{equation}
		
		Next, suppose $ x^{*} $ is optimal, and for the sake of contradiction, assume that there is some $ y \in \mathcal{A} $ such that $ \langle x^{*}-p, y-x^{*} \rangle \not= 0  $. Write $ z = y - x^{*} $, for $ t \in \mathbb{R} $
		\begin{equation}
			\begin{gathered}
				\| x^{*}+tz - p \|^{2} = \| x^{*}-p \|^{2} + t^{2}\| z \|^{2} + 2t\langle x^{*}-p, z \rangle .
			\end{gathered}
		\end{equation}
		Since $ \langle x^{*}-p, z \rangle \not= 0 $, there exists some $ t_{o} \in \mathbb{R} $ such that $ t^{2}\| z \|^{2} + 2t\langle x^{*}-p, z \rangle < 0 $, and thus $ \| x^{*}+t_{o}z - p \|^{2} < \| x^{*} - p \|^{2} $, a contradiction.
	\end{proof}
	
	\begin{definition}[Projection Induced by a Semi Inner Product]
		Given a semi inner product on a linear space $ \mathcal{V} $, the set of all optimal solutions to the problem
		\begin{equation}
			\begin{gathered}
				\argmin_{x \in \mathcal{A}} \| x - p \|^{2} ,
			\end{gathered}
		\end{equation}
		where $ \mathcal{A} \subseteq \mathcal{V} $ is an affine space and $ p \in \mathcal{V} $, is denoted by $ \mathrm{Proj}_{\mathcal{A}}(\{ p \}) $. To account for the possibility that there are multiple optimal solutions, we extend the definition by letting $ \mathrm{Proj}_{\mathcal{A}}(S) = \bigcup_{p \in S}\mathrm{Proj}_{\mathcal{A}}(\{ p \}) $.
	\end{definition}
	
	\begin{lemma} \label{lem:successive projection}
		Let $ \mathcal{V} $ be a linear space with a semi inner product. Suppose there are two affine spaces $ \mathcal{B} \subseteq \mathcal{A} $, for every $ p \in \mathcal{V} $, there is
		\begin{equation}
			\begin{gathered}
				\mathrm{Proj}_{\mathcal{B}}(\mathrm{Proj}_{\mathcal{A}}(\{ p \})) \subseteq \mathrm{Proj}_{\mathcal{B}}(\{ p \}) .
			\end{gathered}
		\end{equation}  
	\end{lemma}

	\begin{proof}
		We rephrase Lemma~\ref{lem:optimal criterion} to ease the proof. For each affine space $ \mathcal{A} \subseteq \mathcal{V} $, define $ \mathcal{L}_{\mathcal{A}} = \mathcal{A} - q $ for some $ q \in \mathcal{A} $. Note that the resulting $ \mathcal{L}_{\mathcal{A}} $ is independent of the choice of $ q \in \mathcal{A} $ and it is a subspace in $ \mathcal{V} $. Therefore, Eq.~\eqref{eq:optimal criterion for projection} is equivalent to
		\begin{equation}
			\begin{gathered}
				\langle x^{*} - p, z \rangle = 0,\ \forall z \in \mathcal{L}_{\mathcal{A}} .
			\end{gathered}
		\end{equation}
		
		Suppose $ x \in \mathrm{Proj}_{\mathcal{B}}(\mathrm{Proj}_{\mathcal{A}}(\{ p \})) $, by Lemma~\ref{lem:optimal criterion}, there exists $ y \in \mathrm{Proj}_{\mathcal{A}}(\{ p \}) $ such that
		\begin{equation}
			\begin{gathered}
				\langle y-p, a \rangle = 0,\ \forall a \in \mathcal{L}_{\mathcal{A}}
				\ \text{ and }\ \langle x-y, b-x \rangle = 0,\ \forall b \in \mathcal{B} .
			\end{gathered}
		\end{equation}
		Therefore, 
		\begin{equation}
			\begin{gathered}
				\langle x-p, b-x \rangle = \langle x-y, b-x \rangle + \langle y-p, b-x \rangle = 0 + 0 .
			\end{gathered}
		\end{equation}
		$ \langle y-p, b-x \rangle = 0 $ is due to that $ b - x \in \mathcal{L}_{\mathcal{B}} \subseteq \mathcal{L}_{A} $.
	\end{proof}
	
	\begin{lemma}[{\cite[Theorem 12]{ruiz1998family}}] \label{lem:constrained}
		Let $ \mathbf{v}^{*} $ be the uniquely optimal solution to
		\begin{equation} \label{op:constrained}
			\begin{gathered}
				\argmin_{\mathbf{v} \in \mathbb{R}^{n}} \sum_{S\subseteq [n]} \eta_{s+1} \left( U(S) - U(\emptyset) - \sum_{i\in S} v_{i} \right)^{2}
				\ \text{ s.t. }\ \sum_{i\in [n]}v_{i} = U([n]) - U(\emptyset)
			\end{gathered}
		\end{equation}
		where $ \eta_{s} = p_{s-1} + p_{s} $ for $ 2\leq s\leq n $. Then, there is
		\begin{equation}
			\begin{gathered}
				v_{i}^{*} - v_{j}^{*} = \phi_{i} - \phi_{j} \ \text{ for every } i,j\in [n] .
			\end{gathered}
		\end{equation}
	\end{lemma}
	
	Recall that the problem~\eqref{op:datamodels} is
	\begin{equation} 
		\begin{gathered}
			\argmin_{\boldsymbol\theta\in\mathbb{R}^{n}, b \in \mathbb{R}} \sum_{S \subseteq [n]} \eta_{s+1}\left( U(S) - b - \sum_{i\in S}\theta_{i} \right)^{2} ,
		\end{gathered}
	\end{equation}
	and our goal is to prove that
	\begin{equation}
		\begin{gathered}
			\theta_{i}^{*} - \theta_{j}^{*} = v_{i}^{*} - v_{j}^{*} \ \text{ for every }\ i,j \in [n] ,
		\end{gathered}
	\end{equation}	
	which together with Lemma~\ref{lem:constrained} is sufficient to complete our proof. 
	\connection*
	\begin{proof}
		The first part of our proof was inspired by \citep[Lemma 2.9]{hammer1992approximations}. 
		Let $ \mathcal{G} = \{ U:2^{[n]} \to \mathbb{R} \} $, $ \mathcal{AG} = \{ U \in \mathcal{G} \mid U(S) = a_{0} + \sum_{i\in S} a_{i} \ \text{ for every }\ S \subseteq [n] \} $ and $ \mathcal{A}_{U} = \{ g \in \mathcal{AG} \mid U([n]) = g([n]) \ \text{ and }\ U(\emptyset)=g(\emptyset) \} $. Note that $ \mathcal{G} $ is a linear space and the other two are affine spaces with $ \mathcal{A}_{U} \subseteq \mathcal{AG} $. For clarity, each game in $ \mathcal{AG} $ is written as $ [a_{0}, \mathbf{a}] $ where $ \mathbf{a} \in \mathbb{R}^{n} $. 
		
		A semi inner product on $ \mathcal{G} $ can be defined by letting $ \langle g_{1}, g_{2} \rangle = \sum_{S\subseteq [n]} \eta_{s+1}\cdot g_{1}(S)g_{2}(S) $ for every $ g_{1}, g_{2} \in \mathcal{G} $. 
		Then, 
		$ [b^{*}, \boldsymbol\theta^{*}] $ is the projection of $ U $ onto $ \mathcal{AG} $, whereas $ [U(\emptyset), \mathbf{v}^{*}] $ is the projection  of $ U $ onto $ \mathcal{A}_{U} $ where $ \mathbf{v}^{*} $ is the uniquely optimal solution to the problem~\eqref{op:constrained}.

		By Lemma~\ref{lem:successive projection}, there is $ \mathrm{Proj}_{\mathcal{A}_{U}}(\mathrm{Proj}_{\mathcal{AG}}(\{ U \})) \subseteq \mathrm{Proj}_{\mathcal{A}_{U}}(\{ U \}) $. Moreover, the uniqueness in problem~\eqref{op:constrained} implies that $ \mathrm{Proj}_{\mathcal{A}_{U}}(\{ U \}) = \{ [U(\emptyset), \mathbf{v}^{*}] \} $, and thus
		\begin{equation}
			\begin{gathered}
				\mathrm{Proj}_{\mathcal{A}_{U}}(\mathrm{Proj}_{\mathcal{AG}}(\{ U \})) = \mathrm{Proj}_{\mathcal{A}_{U}}(\{ U \}) =\{ [U(\emptyset), \mathbf{v}^{*}] \} .
			\end{gathered}
		\end{equation}
		Since $ [b^{*}, \boldsymbol\theta^{*}] \in \mathrm{Proj}_{\mathcal{AG}}(\{ U \}) $, the equality $ \mathrm{Proj}_{\mathcal{A}_{U}}(\{ [b^{*}, \boldsymbol\theta^{*}] \}) = \{ [U(\emptyset), \mathbf{v}^{*}] \} $ means that
		$ [U_{\emptyset}, \mathbf{v}^{*}] $ is the uniquely optimal solution to the problem
		\begin{equation} \label{op:successive projection}
			\begin{gathered}
				\argmin_{[U(\emptyset), \mathbf{v}] \in \mathcal{A}_{U}} \sum_{S\subseteq [n]} \eta_{s+1}\left( [U(\emptyset), \mathbf{v}](S) - [b^{*},\boldsymbol\theta^{*}](S) \right)^{2} .
			\end{gathered}
		\end{equation}
		Pick $ i,j \in [n] $ such that $ i \not= j $, and define an additive game $ \mathbf{e}^{i} \in \mathcal{AG} $ by letting $ \mathbf{e}^{i}(S) = 1 $ if $ i \in S $ and $ 0 $ otherwise, $ \mathbf{e}^{j} $ is defined similarly. Consider the problem
		\begin{equation} \label{op:single variable}
			\begin{gathered}
				\argmin_{t \in \mathbb{R}} \sum_{S\subseteq [n]} \eta_{s+1}\left( [U(\emptyset), \mathbf{v}^{*}](S) - [b^{*}, \boldsymbol\theta^{*}](S) + t(\mathbf{e}^{i}(S)-\mathbf{e}^{j}(S)) \right)^{2} .
			\end{gathered}
		\end{equation}
		Note that $ [U(\emptyset), \mathbf{v}^{*}] + t(\mathbf{e}^{i}-\mathbf{e}^{j}) \in \mathcal{A}_{U} $ for every $ t \in \mathbb{R} $, and the uniqueness to the problem~\eqref{op:successive projection} suggests that $ t^{*} = 0 $ is the uniquely optimal solution to the problem~\eqref{op:single variable}.
		Removing all constant terms in the problem~\eqref{op:single variable} yields an equivalent problem
		\begin{equation}
			\begin{aligned}
				\argmin_{t \in \mathbb{R}} & \sum_{S\subseteq [n] \colon i \in S, j\not\in S} \eta_{s+1}\left( [U(\emptyset), \mathbf{v}^{*}](S) - [b^{*}, \boldsymbol\theta^{*}](S) + t \right)^{2}\\ &+ \sum_{S\subseteq [n]\colon i\not\in S, j\in S} \eta_{s+1}\left( [U(\emptyset), \mathbf{v}^{*}](S) - [b^{*}, \boldsymbol\theta^{*}](S) - t \right)^{2} .
			\end{aligned}
		\end{equation}
		Write $ g = [U(\emptyset), \mathbf{v}^{*}] - [b^{*}, \boldsymbol\theta^{*}] $, since this problem is convex, letting the derivative equal $ 0 $ leads to
		\begin{equation}
			\begin{gathered}
				t^{*} = \frac{\sum_{S\subseteq[n]\colon i\not\in S, j\in S}\eta_{s+1}\cdot g(S) - \sum_{S\subseteq[n]\colon i\in S, j\not\in S}\eta_{s+1}\cdot g(S)}{2\sum_{S\colon i\in S, j\not\in S} \eta_{s+1}} = 0 .
			\end{gathered}
		\end{equation}
		Write $ g = [g_{0}, \mathbf{g}] $ where $ g_{0} = U(\emptyset) - b^{*} $ and $ \mathbf{g} = \mathbf{v}^{*} - \boldsymbol\theta^{*} $, there is
		\begin{equation}
			\begin{gathered}
				\sum_{S\subseteq[n]\colon i\in S, j\not\in S}\eta_{s+1}\cdot g(S) = \alpha (g_{0}+g_{i}) + \beta \sum_{1\leq k\leq n \colon k \not= i,j} g_{k} \\
				\text{where } \alpha = \sum_{s=1}^{n-1} \binom{n-2}{s-1} \eta_{s+1} \text{ and } \beta = \sum_{s=2}^{n-1} \binom{n-3}{s-2} \eta_{s+1} .
			\end{gathered}
		\end{equation}
		Similarly, we have $ \sum_{S\subseteq[n]\colon i\not\in S, j\in S}\eta_{s+1}\cdot g(S) = \alpha(g_{0}+g_{j}) + \beta \sum_{1\leq k\leq n\colon k\not= i, j} g_{k} $, and therefore
		\begin{equation}
			\begin{gathered}
				\alpha(g_{j} - g_{i}) = 0 .
			\end{gathered}
		\end{equation}
		Since $ \alpha > 0 $, we eventually get $ g_{i} = g_{j} $. In other words, $ v_{i}^{*} - \theta^{*}_{i} = v_{j}^{*} - \theta_{j}^{*} $.  
		Because $ i $ and $ j $ are chosen arbitrarily, our proof is completed.
		
		To be self-contained, we also prove that the problem~\eqref{op:datamodels} has only one optimal solution provided that $ \eta_{s} = p_{s-1} + p_{s} $ for $ 2\leq s \leq n $.
		W.L.O.G., assume $ \sum_{S\subseteq [n]} \eta_{s+1} = 1 $. By letting the derivative of the problem~\eqref{op:datamodels} equal $ 0 $, we have $ \mathbf{Ax} = \mathbf{b} $
		\begin{equation}
			\begin{gathered}
				\mathbf{A} = \begin{pmatrix}
					1 & \kappa & \kappa & \cdots & \kappa\\
					\kappa & \kappa & \tau & \cdots & \tau\\
					\kappa & \tau & \kappa & \ddots & \vdots\\
					\vdots & \vdots & \ddots & \ddots & \tau\\
					\kappa & \tau & \cdots & \tau & \kappa
				\end{pmatrix}, \\
				\kappa = \sum_{s=1}^{n}\binom{n-1}{s-1} \eta_{s+1}, \quad \tau = 	\sum_{s=2}^{n} \binom{n-2}{s-2} \eta_{s+1}, \quad
				b_{1} = \sum_{S\subseteq [n]} \eta_{s+1}U(S), \\
				b_{j+1} = \sum_{S\subseteq[n]\colon j \in S}\eta_{s+1}U(S)\ \text{ for every }\ j \in [n], \quad x_{1} = b \ \text{ and }\ x_{j+1} = \theta_{j} \ \text{ for every }\  j \in [n] .
			\end{gathered}
		\end{equation}	
		Left multiplying $ \mathbf{A} $ with some row operation matrix $ \mathbf{R} $ gives
		\begin{equation}
			\begin{gathered}
				\mathbf{RA} = \begin{pmatrix}
					1 & \kappa & \kappa & \cdots & \kappa\\
					0 & \kappa-\kappa^{2} & \tau-\kappa^{2} & \cdots & \tau-\kappa^{2}\\
					0 & \tau-\kappa^{2} & \kappa-\kappa^{2} & \ddots & \vdots\\
					\vdots & \vdots & \ddots & \ddots & \ddots\\
					0 & \tau-\kappa^{2} & \cdots & \tau-\kappa^{2} & \kappa-\kappa^{2}
				\end{pmatrix} .
			\end{gathered}
		\end{equation}
		It is sufficient to prove that the bottom-right $ n\times n $ submatrix  of $ \mathbf{RA} $ is invertible. 
		Suffice it to show $ \kappa - \tau \not=0 $ and $ \kappa + (n-1)\tau - n\kappa^{2} \not= 0 $. Using $ \binom{n}{s} = \binom{n-1}{s} + \binom{n-1}{s-1} $, we have
		\begin{equation}
			\begin{gathered}
				\kappa - \tau = \sum_{s=1}^{n-1} \binom{n-2}{s-1} \eta_{s+1} > 0 .
			\end{gathered}
		\end{equation}	
		Using $ n\binom{n-1}{s-1} = s\binom{n}{s} $, we have
		\begin{equation}
			\begin{gathered}
				\kappa+(n-1)\tau = \sum_{s=1}^{n} s\binom{n-1}{s-1} \eta_{s+1} = \frac{1}{n}\sum_{s=1}^{n}s^{2}\binom{n}{s} \eta_{s+1},\\
				n\cdot \kappa^{2} = n\cdot\left( \sum_{s=1}^{n}\binom{n-1}{s-1} \eta_{s+1} \right)^{2} = \frac{1}{n}\left( \sum_{s=1}^{n} s \binom{n}{s} \eta_{s+1} \right)^{2} .
			\end{gathered}
		\end{equation}
		Let $ \gamma = 1 - \eta_{1} $ and $ \zeta_{s} = \eta_{s+1} / \gamma $ for every $ s \in [n] $, there is
		\begin{equation}
			\begin{gathered}
				n\cdot \kappa^{2} = \frac{\gamma^{2}}{n}\left( \sum_{s=1}^{n} s\binom{n}{s}\zeta_{s} \right)^{2} = \frac{\gamma^{2}}{n} \mathbb{E}[s]^{2} \leq \frac{\gamma^{2}}{n}\mathbb{E}[s^{2}] = \gamma(\kappa+(n-1)\tau) \leq \kappa+(n-1)\tau.
			\end{gathered}
		\end{equation}
		If $ \eta_{1} > 0 $, the last inequality is strict as $ \gamma < 1 $. Otherwise, the first inequality is strict as $ \mathrm{Var}[s] = \mathbb{E}[s^{2}] - \mathbb{E}[s]^{2} > 0 $.
	\end{proof}
	
	\section{Overview of Estimators} \label{app:overview}
	Recall that each probabilistic value is defined to be, for every $ i \in [n] $,
	\begin{equation} \label{eq:probabilistic value}
		\begin{gathered}
			\phi_{i} = \phi_{i}(U) = \sum_{S\subseteq[n]\backslash i} p_{s+1}(U(S\cup i) - U(S))
		\end{gathered}
	\end{equation}
	where $ \mathbf{p} \in \mathbb{R}^{n} $ is a non-negative vector with $ \sum_{s=1}^{n}\binom{n-1}{s-1}p_{s} = 1 $. If $ p_{s} = \int_{0}^{1}w^{s-1}(1-w)^{n-s}\mathrm{d}\mu(w) $ for some probability measure $ \mu $ on the closed interval $ [0,1] $, the induced $ \boldsymbol\phi $ is referred to as a semi-value.
	
	\paragraph{The Sampling Lift Estimator \citep{moehle2021portfolio}}
	The sampling lift estimator is based on
	\begin{equation}
		\begin{gathered}
			\phi_{i} = \mathbb{E}_{S \subseteq [n]\backslash i}[U(S\cup i) - U(S)] \ \text{ where }\ P(S) = p_{s+1} .
		\end{gathered}
	\end{equation}
	The sampling procedure is: i) sample a subset size $ s \in [n] $ with $ P(s) = \binom{n-1}{s-1}p_{s} $, and then ii) sample a subset $ S $ uniformly from $ \{ R \subseteq [n]\backslash i \mid r=s-1 \} $. For semi-values such that $ p_{s} = \int_{0}^{1}w^{s-1}(1-w)^{n-s}\mathrm{d}\mu(w) $ where $ \mu $ is a probability measure on the closed interval $ [0,1] $, there is an alternative: i) sample a $ w \in [0, 1] $ according to $ \mu $, and then sample a subset $ S \subseteq [n]\backslash i $ by incorporating each player in $ [n]\backslash i $ with probability $ w $. With a sequence of sampled subsets $ \{ S_{j} \}_{j=1}^{T} $, the $ i $-th estimate is $ \hat{\phi}_{i} = \frac{1}{T}\sum_{j=1}^{T}(U(S_{j}\cup i) - U(S_{j})) $.
	
	\paragraph{The Weighted Sampling Lift Estimator \citep{kwon2022beta}}
	The formula it is built upon is
	\begin{equation}
		\begin{gathered}
			\phi_{i} = \mathbb{E}_{S\subseteq[n]\backslash i}^{\text{Shap}}\left[\frac{p_{s+1}}{p_{s+1}^{\text{Shap}}}(U(S\cup i) - U(S))\right] \ \text{ where }\ P(S) = p_{s+1}^{\text{Shap}} .
		\end{gathered}
	\end{equation} 
	Note that substituting $ \mathbf{p}^{\text{Shap}} $ in Eq.~\eqref{eq:probabilistic value} leads to the Shapley value. The sampling procedure is: i) sample a $ w $ uniformly from $ [0, 1] $, and then ii) sample a subset $ S \subseteq [n]\backslash i $ by incorporating each player in $ [n]\backslash i $ with probability $ w $. Then, the $ i $-th estimate is $ \hat{\phi}_{i} = \frac{1}{T}\sum_{j=1}^{T}\frac{p_{s_{j}+1}}{p_{s_{j}+1}^{\text{Shap}}}\left( U(S_{j}\cup i) - U(S_{j}) \right) $.
	
	\paragraph{The KernelSHAP Estimator \citep{lundberg2017unified}}
	This estimator is specific to the Shapley value. It employs the fact that the Shapley value $ \boldsymbol\phi_{i}^{\text{Shap}} $ is the uniquely optimal solution to
	\begin{equation} \label{op:ls for shap}
		\begin{gathered}
			\argmin_{\boldsymbol\phi \in \mathbb{R}^{n}} \sum_{\emptyset\subsetneq S \subsetneq [n]} \binom{n-2}{s-1}^{-1}\left( U(S) - U(\emptyset) - \sum_{i\in S}\phi_{i} \right)^{2} \ \text{ s.t. }\ \sum_{i\in [n]}\phi_{i} = U([n]) - U(\emptyset) .
		\end{gathered}
	\end{equation}
	Note that the weights can be scaled so that the objective is an expectation. A sequence of subsets $ \{ S_{j} \}_{j=1}^{T} $ where $ \emptyset \subsetneq S_{j} \subsetneq [n] $ is sampled according to $ P(S) \propto   \binom{n-2}{s-1}^{-1} $. Then, we have an approximate problem as
	\begin{equation}
		\begin{gathered}
			\argmin_{\boldsymbol\phi \in \mathbb{R}^{n}} \frac{1}{T}\sum_{j=1}^{T} \left( U(S_{j}) - U(\emptyset) - \sum_{i\in S_{j}}\phi_{i} \right)^{2} \ \text{ s.t. }\ \sum_{i\in [n]}\phi_{i} = U([n]) - U(\emptyset) ,
		\end{gathered}
	\end{equation}
	the uniquely optimal solution of which is treated as the estimates, i.e.,
	\begin{equation}
		\begin{gathered}
			\hat{\boldsymbol\phi}^{\text{Shap}} = \hat{\mathbf{A}}^{-1}\left( \hat{\mathbf{b}} - \mathbf{1}_{n}\frac{\mathbf{1}_{n}^{\top}\hat{\mathbf{A}}^{-1}\hat{\mathbf{b}}-U([n])+U(\emptyset)}{\mathbf{1}_{n}^{\top}\hat{\mathbf{A}}^{-1}\mathbf{1}_{n}} \right)\\
			\text{where }\ \hat{\mathbf{A}} = \frac{1}{T}\sum_{j=1}^{T}\mathbf{1}_{S_{j}}\mathbf{1}_{S_{j}}^{\top} \ \text{ and }\ \hat{\mathbf{b}}= \frac{1}{T}\sum_{j=1}^{T}(U(S_{j}) - U(\emptyset))\cdot\mathbf{1}_{S_{j}} . 
		\end{gathered}
	\end{equation}
	Specifically, $ \mathbf{1}_{S_{j}} \in \{ 0, 1 \}^{n} $ such that its $ i $-th entry is $ 1 $ if and only if $ i \in S_{j} $.
	
	\paragraph{The Unbiased KernelSHAP Estimator \citep{covert2021improving}}
	The uniquely optimal solution $ \boldsymbol\phi^{\text{Shap}} $ to the problem~\eqref{op:ls for shap} is
	\begin{equation}
		\begin{gathered}
			\boldsymbol\phi^{\text{Shap}} = \mathbf{A}^{-1}\left(\mathbf{b} - \mathbf{1}_{n}\frac{\mathbf{1}_{n}^{\top}\mathbf{A}^{-1}\mathbf{b}-U([n])+U(\emptyset)}{\mathbf{1}_{n}^{\top}\mathbf{A}^{-1}\mathbf{1}_{n}} \right) \\
			\text{where }\ \mathbf{A} = \mathbb{E}[\mathbf{1}_{S}\mathbf{1}_{S}^{\top}] \ \text{ and } \mathbf{b} = \mathbb{E}[(U(S)-U(\emptyset))\cdot \mathbf{1}_{n}] .
		\end{gathered}
	\end{equation}
	This estimator employs the fact that $ \mathbf{A}_{ij} = \frac{1}{2} $ if $ i=j $ and $ \frac{1}{n(n-1)}\frac{\sum_{s=2}^{n-1}\frac{s-1}{n-s}}{\sum_{s=1}^{n-1}\frac{1}{s(n-s)}} $ otherwise.
	In other words, the estimates of this estimator is
	\begin{equation} \label{eq: est of unbiased kernelSHAP}
		\begin{gathered}
			\hat{\boldsymbol\phi}^{\text{Shap}} = \mathbf{A}^{-1}\left(\hat{\mathbf{b}} - \mathbf{1}_{n}\frac{\mathbf{1}_{n}^{\top}\mathbf{A}^{-1}\hat{\mathbf{b}}-U([n])+U(\emptyset)}{\mathbf{1}_{n}^{\top}\mathbf{A}^{-1}\mathbf{1}_{n}} \right) \ \text{ where }\ \hat{\mathbf{b}}= \frac{1}{T}\sum_{j=1}^{T}(U(S_{j}) - U(\emptyset))\cdot\mathbf{1}_{S_{j}} .
		\end{gathered}
	\end{equation}
	Particularly $ \{ S_{j} \}_{j=1}^{T} $ where $ \emptyset\subsetneq S_{j} \subsetneq [n] $ are sampled using $ P(S) \propto \binom{n-2}{s-1}^{-1} $.
	
	Recently, \citet{fumagalli2024shap} proved that Eq.~\eqref{eq: est of unbiased kernelSHAP} can be simplified as
	\begin{equation}
		\begin{gathered}
			\hat{\phi}_{i}^{\text{Shap}} = \frac{U([n])-U(\emptyset)}{n} + \frac{2\sum_{s=1}^{n-1}\frac{1}{s}}{T}\sum_{j=1}^{T}U(S_{j}) \left( \mathds{1}_{i\in S_{j}} - \frac{s_{j}}{n} . \right)
		\end{gathered}
	\end{equation}

	\paragraph{The ARM Estimator \citep{kolpaczki2024approximating}}
	This estimator is designed according to
	\begin{equation}
		\begin{gathered}
			\phi_{i} = \mathbb{E}_{S\sim P^{+} \mid i\in S}[U(S)] - \mathbb{E}_{S\sim P^{-} \mid i\not\in S}[U(S)]
		\end{gathered}
	\end{equation}
	where $ P^{+}(S) \propto p_{s} $ for every $ \emptyset \subsetneq S \subseteq [n] $ and $ P^{-}(S) \propto p_{s+1} $ for every $ S \subsetneq [n] $ \citep[Proposition 8]{li2024faster}. A sequence of subsets $ \{ S_{j} \}_{j=1}^{T} $ are sampled using $ P^{+} $ and $ P^{-} $ alternatively, i.e., $ \{ S_{2k-1} \}_{k=1}^{\frac{T}{2}} $ are sampled independently according to $ P^{+} $, whereas $ \{ S_{2k} \}_{k=1}^{\frac{T}{2}} $ are sampled independently using $ P^{-} $. Then, the $ i $-th estimate is
	\begin{equation}
		\begin{gathered}
			\hat{\phi}_{i} = \frac{1}{T_{i}^{+}}\sum_{k=1}^{\frac{T}{2}}U(S_{2k-1})\mathds{1}_{i\in S_{2k-1}} - \frac{1}{T_{i}^{-}}\sum_{k=1}^{\frac{T}{2}}U(S_{2k})\mathds{1}_{i\not\in S_{2k}}  
		\end{gathered}
	\end{equation}
	where $ T_{i}^{+} = \sum_{k=1}^{\frac{T}{2}}\mathds{1}_{i\in S_{2k-1}} $ and $ T_{i}^{-} = \sum_{k=1}^{\frac{T}{2}}\mathds{1}_{i\not\in S_{2k}} $.
	
	\paragraph{The AME Estimator \citep{lin2022measuring}}
	This estimator is restricted to a sub-family of semi-values that satisfy $ \int_{0}^{1}\frac{1}{w(1-w)}\mathrm{d}\mu(w) < \infty $. For such a semi-value $ \boldsymbol\phi $, it can be cast as a uniquely optimal solution to
	\begin{equation}
		\begin{gathered}
			\argmin_{\mathbf{v} \in \mathbb{R}^{n}} \mathbb{E}[(Y-\bm{X}^{\top}\mathbf{v})^{2}]
		\end{gathered}
	\end{equation}
	where $ \bm{X} \in \mathbb{R}^{n} $ and $ Y $ are random variables. The sampling procedure is: i) sample a $ w \in (0, 1) $ using $ \mu $, 
	ii) sample a subset $ S $ by incorporating each player with probability $ w $, and then iii) $ Y = U(S) $ and $ \bm{X} = \bm{X}(S) $ such that $ X_{i} = \frac{1}{w\cdot C} $ if $ i \in S $ and $ -\frac{1}{(1-w)C} $ otherwise where $ C = \int_{0}^{1} \frac{1}{w(1-w)}\mathrm{d}\mu(w) $. With a sequence of subsets $ \{ S_{j} \}_{j=1}^{T} $, the uniquely optimal solution to the approximate problem
	\begin{equation}
		\begin{gathered}
			\argmin_{\mathbf{v} \in \mathbb{R}^{n}} \frac{1}{T}\sum_{j=1}^{T}\left( U(S_{j}) -  \bm{X}(S_{j})^{\top}\mathbf{v} \right)^{2}
		\end{gathered}
	\end{equation}
	is taken as the induced estimates, which is $ \hat{\boldsymbol\phi} = (\mathbf{A}^{\top}\mathbf{A})^{-1}\mathbf{A}^{\top}\mathbf{b} $
	where the $ j $-th row of $ \mathbf{A} $ is $ \bm{X}(S_{j})^{\top} $ and $ b_{j} = U(S_{j}) $.
	
	\paragraph{The MSR Estimator \citep{wang2023data}}
	The methodology of this estimator is limited to weighted Banzhaf values parameterized with $ 0<a<1 $ \citep[Appendix C.2]{wang2023data}. Precisely, $ p_{s} = a^{s-1}(1-a)^{n-s} $. Each subset is sampled by incorporating each player with probability $ a $, and then the $ i $-th estimate is
	\begin{equation}
		\begin{gathered}
			\hat{\phi}_{i} = \frac{1}{T_{i}^{+}}\sum_{j=1}^{T}U(S_{j})\mathds{1}_{i\in S_{j}} - \frac{1}{T_{i}^{-}}\sum_{j=1}^{T}U(S_{j})\mathds{1}_{i\not\in S_{j}}
		\end{gathered}
	\end{equation}
	where $ T_{i}^{+} = \sum_{j=1}^{T}\mathds{1}_{i\in S_{j}} $ and $ T_{i}^{-} = \sum_{j=1}^{T}\mathds{1}_{i\not\in S_{j}} $.
	
	\paragraph{The GELS Estimator \citep{li2024faster}}
	This estimator is established using the fact that $ \phi_{i} = v^{*}_{i} - v^{*}_{n+1} $ where $ \mathbf{v}^{*} \in \mathbb{R}^{n+1} $ is the uniquely optimal solution to
	\begin{equation}
		\begin{gathered}
			\argmin_{\mathbf{v} \in \mathbb{R}^{n+1}}\sum_{\emptyset\subsetneq S \subsetneq [n+1]} p_{s}\left( U(S\cap [n]) - \sum_{i\in S}v_{i} \right)^{2} .
		\end{gathered}
	\end{equation}
	The subsets $ \{ S_{j} \}_{j=1}^{T} $ where $ \emptyset \subsetneq S_{j} \subsetneq [n+1] $ are sampled using $ P(S) \propto p_{s} $, and then the $ i $-th estimate is
	\begin{equation}
		\begin{gathered}
			\hat{\phi}_{i} = \left( \sum_{s=1}^{n}\binom{n}{s-1}p_{s} \right)(\hat{v}_{i} - \hat{v}_{n+1}) 
		\end{gathered}
	\end{equation}
	where $ \hat{v}_{k} = \frac{1}{T_{k}}\sum_{j=1}^{T}U(S_{j}\cap [n])\mathds{1}_{k \in S_{j}} $ and $ T_{k} = \sum_{j=1}^{T}\mathds{1}_{k\in S_{j}} $.
	
	\paragraph{The Complement Estimator \citep{zhang2023efficient}}
	The complement estimator is specific to the Shapley value using the fact that
	\begin{equation}
		\begin{gathered}
			\phi_{i}^{\text{Shap}} = \frac{1}{n}\sum_{S \subseteq [n]\backslash i} \binom{n-1}{s}^{-1} \left( U(S\cup i) - U([n]\backslash(S\cup i)) \right) .
		\end{gathered}
	\end{equation}
	The sequence of subsets $ \{ S_{j} \}_{j=1}^{T} $ is sampled using i) sample a subset size $ s \in [n] $ uniformly, and then sample a subset $ S $ uniformly from $ \{ R \subseteq [n] \mid r = s \} $. Then, the $ i $-th estimate is
	\begin{equation}
		\begin{gathered}
			\hat{\phi}_{i}^{\text{Shap}} = \frac{1}{n}\sum_{s=1}^{n}\hat{\phi}_{i,s}  \ \text{ where }\ \hat{\phi}_{i,s} = \frac{1}{T_{i,s}}\sum_{j=1}^{n}\left( v_{j}\idt{i\in S_{j}, s_{j}=s} - v_{j}\idt{i\not\in S_{j}, n-s_{j}=s} \right)\\
			v_{j} = U(S_{j}) - U([n]\backslash S_{j}) \ \text{ and }\ T_{i,s} = \sum_{j=1}^{T}\left( \idt{i\in S_{j}, s_{j}=s} + \idt{i\not\in S_{j}, n-s_{j}=s} \right) .
		\end{gathered}
	\end{equation}
	
	\paragraph{The Group Testing Estimator \citep{jia2019towards}}
	We introduce the improved version presented by \citet{wang2023note}. Note that this estimator is specific to the Shapley value.
	A sequence of subsets $ \{ S_{j} \}_{j=1}^{T} $ are independently sampled according to: i) sample a subset size $ s \in [n] $ using $ P(s) \propto \frac{1}{s(n+1-s)} $, and then ii) sample a subset $ S $ uniformly from $ \{ R \subseteq [n+1] \mid r=s \} $. Then, the $ i $-th estimate is
	\begin{equation}
		\begin{gathered}
			\hat{\phi}_{i}^{\text{Shap}} = \frac{2\sum_{s=1}^{n}\frac{1}{s}}{T}\sum_{j=1}^{T} U(S_{j}\cap [n]) \left(\idt{i \in S_{j}, n+1\not\in S_{j}} - \idt{i \not\in S_{j}, n+1\in S_{j}} \right) .
		\end{gathered}
	\end{equation}
	
	\paragraph{The Permutation Estimator \citep{castro2009polynomial}}
	This estimator is specific to the Shapley value, using the formula
	\begin{equation}
		\begin{gathered}
			\phi_{i}^{\text{Shap}} = \frac{1}{n!}\sum_{\pi \in \Pi} (U(\mathcal{P}^{i}(\pi)\cup i) - U(\mathcal{P}^{i}(\pi)))
		\end{gathered}
	\end{equation}
	where $ \Pi $ contains all permutations of $ [n] $ and $ \mathcal{P}^{i}(\pi) $ is the subset that contains all players preceding $ i $ in $ \pi $ . Thus, it samples a sequence of permutations $ \{ \pi_{j} \}_{j=1}^{T} $ from $ \Pi $ uniformly with replacement, and then the $ i $-th estimate is $ \hat{\phi}_{i}^{\text{Shap}} = \frac{1}{T}\sum_{j=1}^{T} \left( U(\mathcal{P}^{i}(\pi_{j})\cup i) - U(\mathcal{P}^{i}(\pi_{j})) \right) $.

	\paragraph{The WeightedSHAP Estimator \citep{kwon2022weightedshap}}
	As mentioned in the main paper, it is based on
	\begin{equation}
		\begin{gathered}
			\phi_{i} = \sum_{s=1}^{n}m_{s}\cdot\mathbb{E}_{\substack{R\subseteq[n]\backslash i\\ r=s-1}}[U(R\cup i)-U(R)]
		\end{gathered}
	\end{equation}
	where $ m_{s} = \binom{n-1}{s-1}p_{s} $.
	
	For each player $ i \in [n] $, it samples a sequence of permutations $ \{ \pi_{j} \}_{j=1}^{T} $ of $ [n]\backslash i $. Then, the corresponding estimate is $ \hat{\phi}_{i} = \sum_{s=1}^{n}m_{s}\hat{\phi}_{i,s} $ where $ \hat{\phi}_{i,k} = \frac{1}{T}\sum_{j=1}^{T} \left( U(\mathcal{S}^{k}(\pi_{j}) \cup i) - U(\mathcal{S}^{k}(\pi_{j})) \right) $ and $ \mathcal{S}^{k}(\pi_{j}) $ is the subset that contains the first $ k-1 $ players in $ \pi_{j} $.
	
	\paragraph{The SHAP-IQ Estimator \citep{fumagalli2024shap}}
	Recall that its underlying formula is
	\begin{equation}
		\begin{gathered}
			\phi_{i} = p_{n}\cdot (U([n])-U(\emptyset)) + 2H\cdot\mathbb{E}_{\emptyset\subsetneq S \subsetneq [n]}[((n-s)m_{s}\mathds{1}_{i\in S} - sm_{s+1}\mathds{1}_{i\not\in S})\cdot(U(S)-U(\emptyset))]
		\end{gathered}
	\end{equation}
	where $ m_{s} = \binom{n-1}{s-1}p_{s} $, $ H = \sum_{j=1}^{n-1}\frac{1}{j} $, and $ P(S) \propto \binom{n-2}{s-1}^{-1} $. Therefore, a sequence of subsets $ \{ S_{j} \}_{j=1}^{T} $ where $ \emptyset \subsetneq S_{j} \subsetneq [n] $ is sampled using $ P(S) \propto \binom{n-2}{s-1}^{-1} $, and the $ i $-th estimate is
	\begin{equation}
		\begin{gathered}
			\hat{\phi}_{i} = p_{n}\cdot(U([n])-U(\emptyset)) + \frac{2H}{T}\sum_{j=1}^{T}(U(S_{j})-U(\emptyset))\cdot\left( (n-s)m_{s}\mathds{1}_{i\in S_{j}} - sm_{s+1}\mathds{1}_{i\not\in S_{j}} \right) .
		\end{gathered}
	\end{equation} 
	
	\newpage
	\section*{NeurIPS Paper Checklist}

	\begin{enumerate}
		
		\item {\bf Claims}
		\item[] Question: Do the main claims made in the abstract and introduction accurately reflect the paper's contributions and scope?
		\item[] Answer: \answerYes{} 
		\item[] Justification: Our claimed theories are presented in Section~\ref{sec:main} and are empirically verified in Section~\ref{sec:experiments}. 
		
		\item {\bf Limitations}
		\item[] Question: Does the paper discuss the limitations of the work performed by the authors?
		\item[] Answer: \answerYes{} 
		\item[] Justification: Proposition~\ref{prop:ofaWeighted} demonstrates that our OFA-A estimator does not rival the previously best estimator for weighted Banzhaf values in terms of convergence rate, which is a price to pay for using a fixed sampling scheme for all probabilistic values.  
		
		\item {\bf Theory Assumptions and Proofs}
		\item[] Question: For each theoretical result, does the paper provide the full set of assumptions and a complete (and correct) proof?
		\item[] Answer: \answerYes{} 
		\item[] Justification: We have provided detailed proofs in the Appendices~\ref{app:theorem 1}, \ref{app:propositions} and \ref{app:theorem 2}.
		
		\item {\bf Experimental Result Reproducibility}
		\item[] Question: Does the paper fully disclose all the information needed to reproduce the main experimental results of the paper to the extent that it affects the main claims and/or conclusions of the paper (regardless of whether the code and data are provided or not)?
		\item[] Answer: \answerYes{} 
		\item[] Justification: Our experiment settings are stated in Section~\ref{sec:experiments}, and our method is presented in Algorithm~\ref{alg:ofa}.

		\item {\bf Open access to data and code}
		\item[] Question: Does the paper provide open access to the data and code, with sufficient instructions to faithfully reproduce the main experimental results, as described in supplemental material?
		\item[] Answer: \answerYes{} 
		\item[] Justification: The datasets we used are from open resources, and our code will be released on a github repo.

		\item {\bf Experimental Setting/Details}
		\item[] Question: Does the paper specify all the training and test details (e.g., data splits, hyperparameters, how they were chosen, type of optimizer, etc.) necessary to understand the results?
		\item[] Answer: \answerYes{} 
		\item[] Justification: Our experiment settings are stated in Section~\ref{sec:experiments}.
		
		\item {\bf Experiment Statistical Significance}
		\item[] Question: Does the paper report error bars suitably and correctly defined or other appropriate information about the statistical significance of the experiments?
		\item[] Answer: \answerYes{} 
		\item[] Justification: Our experiment results in Section~\ref{sec:experiments} are all reported with standard deviation using $ 30 $ random seeds.
		
		\item {\bf Experiments Compute Resources}
		\item[] Question: For each experiment, does the paper provide sufficient information on the computer resources (type of compute workers, memory, time of execution) needed to reproduce the experiments?
		\item[] Answer: \answerYes{} 
		\item[] Justification: It is stated in Section~\ref{sec:experiments}.
		
		\item {\bf Code Of Ethics}
		\item[] Question: Does the research conducted in the paper conform, in every respect, with the NeurIPS Code of Ethics \url{https://neurips.cc/public/EthicsGuidelines}?
		\item[] Answer: \answerYes{} 
		\item[] Justification: We have complied with the NeurIPS Code of Ethics.

		\item {\bf Broader Impacts}
		\item[] Question: Does the paper discuss both potential positive societal impacts and negative societal impacts of the work performed?
		\item[] Answer: \answerNA{} 
		\item[] Justification: Our work focuses on the convergence rate of estimators for probabilistic values that do not appear to have any significant societal impact.
		
		\item {\bf Safeguards}
		\item[] Question: Does the paper describe safeguards that have been put in place for responsible release of data or models that have a high risk for misuse (e.g., pretrained language models, image generators, or scraped datasets)?
		\item[] Answer: \answerNA{} 
		\item[] Justification:  Our work focuses on the convergence rate of estimators for probabilistic values that do not appear to pose any risk for misuse.
		
		\item {\bf Licenses for existing assets}
		\item[] Question: Are the creators or original owners of assets (e.g., code, data, models), used in the paper, properly credited and are the license and terms of use explicitly mentioned and properly respected?
		\item[] Answer: \answerYes{} 
		\item[] Justification: It is stated in Section~\ref{sec:experiments}.
		
		\item {\bf New Assets}
		\item[] Question: Are new assets introduced in the paper well documented and is the documentation provided alongside the assets?
		\item[] Answer: \answerNA{} 
		\item[] Justification: We do not introduce any new assets.
		
		\item {\bf Crowdsourcing and Research with Human Subjects}
		\item[] Question: For crowdsourcing experiments and research with human subjects, does the paper include the full text of instructions given to participants and screenshots, if applicable, as well as details about compensation (if any)? 
		\item[] Answer: \answerNA{} 
		\item[] Justification: Our work does not involve human subjects.
		
		\item {\bf Institutional Review Board (IRB) Approvals or Equivalent for Research with Human Subjects}
		\item[] Question: Does the paper describe potential risks incurred by study participants, whether such risks were disclosed to the subjects, and whether Institutional Review Board (IRB) approvals (or an equivalent approval/review based on the requirements of your country or institution) were obtained?
		\item[] Answer: \answerNA{} 
		\item[] Justification: Our work does not involve human subjects.
		
	\end{enumerate}
	
\end{document}